\documentclass{article}

\usepackage{amsmath,amsfonts,amssymb}
\usepackage{amsthm}
\usepackage{graphicx} 

\iftrue 
\fi

\usepackage{hyperref}
\usepackage{url}
\usepackage{amssymb,amsfonts,amsmath,amscd,dsfont,mathrsfs}
\usepackage{float,psfrag,epsfig}
\usepackage{wrapfig}
\usepackage{color}

\usepackage{float}

\usepackage{subcaption}

\iftrue 
\newtheorem{propo}{Proposition}[section]
\newtheorem{corollary}[propo]{Corollary}
\newtheorem{theorem}{Theorem}
\newtheorem{remark}[propo]{Remark}
\newtheorem{lemma}[propo]{Lemma}
\fi

\newcommand{\dataset}{{\cal D}}
\newcommand{\fracpartial}[2]{\frac{\partial #1}{\partial  #2}}

\makeatletter
\newsavebox{\@brx}
\newcommand{\llangle}[1][]{\savebox{\@brx}{\(\m@th{#1\langle}\)}%
  \mathopen{\copy\@brx\kern-0.5\wd\@brx\usebox{\@brx}}}
\newcommand{\rrangle}[1][]{\savebox{\@brx}{\(\m@th{#1\rangle}\)}%
  \mathclose{\copy\@brx\kern-0.5\wd\@brx\usebox{\@brx}}}
\makeatother

\def\<{\langle}
\def\>{\rangle}
\def\ones{\mathds 1}
\def\id{\mathbf I}
\newcommand{\prob}[1]{{ \mathbb{P}\left\{ #1 \right\} }}
\newcommand{\expect}[1]{\mathbb{E}\left[ #1 \right]}

\def\E{\mathbb E}
\def\reals{\mathbb R}
\def\Z{\mathbb Z}
\def\ind{\mathbb{I}}
\def\diag{{\rm diag}}

\newcommand{\abs}[1]{\left| #1 \right|}

\newcommand{\vertiii}[1]{{\left\vert\kern-0.25ex\left\vert\kern-0.25ex\left\vert #1 
    \right\vert\kern-0.25ex\right\vert\kern-0.25ex\right\vert}}

\newcommand{\triplenorm}[1]{{\vert\kern-0.25ex \vert\kern-0.25ex \vert #1 
    \vert\kern-0.25ex \vert\kern-0.25ex \vert}}

\def\top{T}

\newcommand{\bb}{{\alpha}}
\newcommand{\hTheta}{\widehat{\Theta}}
\newcommand{\hL}{\widehat{L}}
\newcommand{\hN}{\widehat{N}}
\newcommand{\tell}{{\tilde \ell}}
\newcommand{\cL}{{\cal L}}
\newcommand{\cG}{{\cal G}}
\newcommand{\cT}{{\cal T}}
\newcommand{\cA}{{\cal A}}
\newcommand{\bB}{{\mathbb B}}
\newcommand{\calL}{{\cal L}}
\newcommand{\calA}{{\cal A}}
\newcommand{\calP}{{\cal P}}
\newcommand{\calB}{{\cal B}}
\newcommand{\calS}{{\cal S}}
\newcommand{\tV}{\widetilde{V}}
\newcommand{\tk}{\tilde{k}}
\newcommand{\tH}{\tilde{H}}
\newcommand{\tW}{\tilde{W}}
\newcommand{\tX}{\tilde{X}}
\newcommand{\tY}{\tilde{Y}}
\newcommand{\tZ}{\tilde{Z}}
\newcommand{\txi}{{\tilde{\xi}}}

\newcommand{\Fnorm}[1]{\vertiii{#1}_{\rm F}}
\newcommand{\fnorm}[1]{\vertiii{#1}_{\rm F}}
\newcommand{\Lnorm}[1]{\vertiii{#1}_{\rm L}}
\newcommand{\Lnucnorm}[1]{\vertiii{#1}_{\rm L\text{-}nuc}}
\newcommand{\hLnucnorm}[1]{\vertiii{#1}_{\rm{\hat{L}}\text{-}nuc}}
\newcommand{\nucnorm}[1]{\vertiii{#1}_{\rm nuc}}
\newcommand{\opnorm}[1]{\left\| #1 \right\|_2}
\newcommand{\lnorm}[2]{\vertiii{#1}_{{#2}}}
\newcommand{\Iprod}[2]{\llangle #1, #2 \rrangle}
\newcommand{\indc}[1]{\mathbb{I}\left( {#1 } \right) }
\newcommand{\trace}[1]{{\rm tr}\left(#1\right)}

\title{Learning from Comparisons and Choices }

\author{
Sahand Negahban\thanks{Statistics Department, Yale, 
\texttt{sahand.negahban@yale.edu}}, \;
Sewoong Oh\thanks{Department of Industrial and Enterprise Systems Engineering, University of Illinois at Urbana-Champaign, 
\texttt{swoh@illinois.edu}}, \; %
Kiran K. Thekumparampil\thanks{Department of Electrical and Computer Engineering, University of Illinois at Urbana-Champaign, 
\texttt{thekump2@illinois.edu}}, and \;
Jiaming Xu\thanks{Krannert School of Management, Purdue University, 
\texttt{xu972@purdue.edu}}
}

\date{}

\begin{document}
\maketitle

\begin{abstract}
When tracking user-specific online activities, 
each user's preference is revealed  
 in the form of choices and comparisons. 
For example, a user's purchase history is a record of her choices, i.e.~which item was chosen among a subset of offerings. 
A user's preferences can be observed 
either  explicitly as in movie ratings or implicitly as in viewing times of news articles. 
Given such individualized ordinal data in the form of comparisons and choices, 
we address the problem of collaboratively learning  
representations of the users and the items. 
The learned features can be used to predict a user's preference of an unseen item 
to be used in recommendation systems. 
This also allows one to compute similarities 
among users and items to be used for categorization and search. 
Motivated by the empirical successes of the MultiNomial Logit (MNL) model 
in marketing and transportation, and also more recent successes in 
word embedding and crowdsourced image embedding,
 we pose this problem as learning the MNL model parameters that best explain the data. 
We propose a convex relaxation for learning the MNL model, 
and show that it is minimax optimal up to a logarithmic factor 
by comparing its performance to a fundamental lower bound. 
This characterizes the minimax sample complexity of the problem, and  
proves that the proposed estimator cannot be improved upon other than by a logarithmic factor. 
Further, the analysis identifies how the accuracy depends on the topology of sampling via the spectrum of the sampling graph. 
This  provides a guideline for 
designing surveys when one can choose which items are to be compared. 
This is accompanied by numerical simulations on synthetic and real data sets, confirming our theoretical predictions.  
\end{abstract}


\section{Introduction} 
\label{sec:intro} 

Given data on how users compared subsets of items,  
we address the fundamental problem of learning a 
representation of users and items. 
Such data can be observed in the form of choices (e.g.~which item was bought) or 
in the form of comparisons (e.g.~which items are rated higher). 
From such ordinal data on the items, 
we want to find low dimensional representations, which we call (latent) features, 
that explain crucial aspects of the users' choices.  
Once learned, 
these features can be used to predict each user's preference over 
items that the user has not seen yet, 
which can  be used in recommendation systems and revenue management. 
These learned features also provide an embedding of the users and items on the same Euclidean space 
that allows us to directly quantify similarities via distances, that can be used to categorize and cluster. 
These embeddings can reveal the underlying structure of data such as images. 
Such an embedding of a discrete set of objects based on ordinal data has 
recently gained tremendous attraction mainly due to  word embeddings based on 
co-occurrence  data and  their successes in 
numerous downstream natural language processing tasks \cite{mikolov2013distributed}. 

The fundamental question in such a representation learning is:   
what makes one representation better than the others?   
Our guiding principle is that 
a good representation is the one that defines a generative model  that best explains the given data 
in the maximum likelihood sense.
To this end, we focus on a parametric generative model known as 
MultiNomial Logit (MNL) model, widely used and studied in revenue management. 
The MNL model  has a natural interpretation of human choices as 
an outcome of maximizing a utility by agents with noisy perception of the utility, also known as  
{\em random utility model} in \cite{WB02,SPX12}, defined as follows. 
Each user and item has a latent low-dimensional feature $u_i\in\reals^r$ and $v_j\in\reals^r$ respectively.   
The  true utility of an item is the inner product of these two features $\Theta_{ij} \triangleq \Iprod{u_i}{v_j} = \sum_k u_{ik}v_{jk}$. 
The inherent low-rank structure of $\Theta=[\Theta_{ij}]$ 
captures the collaborative nature of the problem, 
where users with similar preferences in the past are  likely to prefer similar items in the future. 

When presented with a set of items, 
a user reveals a noisy ordering of the items sorted according to her perceived utilities of the items, 
each of which is perturbed by an i.i.d.~noise added to the true utility $\Theta_{ij}$. 
The MNL model is a special case where the noise follows the standard Gumbel distribution, and is one of the most popular models in  choice theory for its simplicity and empirical success \cite{McF73,MT00}. 
The MNL model has several important properties,  
making this model realistic in various domains, including marketing 
\cite{GL83}, transportation \cite{McF80,BL85}, biology \cite{BIOLOGY}, sports games \cite{TNK18} and natural language processing \cite{MCCD13}. 
The MNL model 
$(i)$ satisfies the `independence of irrelevant alternatives' in social choice theory \cite{Ray73}; 
$(ii)$ has a maximum likelihood estimator (MLE) which is a convex program in $\Theta$; and 
$(iii)$ has a simple characterization of sequential (random) choices as follows.   
Let $\prob{a>\{b,c,d\}}$ denote the probability $a$ was chosen as the best alternative among the set $\{a,b,c,d\}$. 
Then, the probability that user $i$ reveals a linear order $(a>b>c>d)$ is $\prob{a>\{b,c,d\}}\prob{b>\{c,d\}}\prob{c>d}$, 
where $\prob{a>\{b,c,d\}} = e^{\Theta_{ia}}/(e^{\Theta_{ia}}+e^{\Theta_{ib}}+e^{\Theta_{ic}}+e^{\Theta_{id}})$. 
Essentially the user is modeled as making a sequence of  choices, choosing the best alternative first and then 
making choices on the remaining ones. 
We give the precise definition of the MNL model in Section \ref{sec:model}  for pairwise comparisons and 
in Section \ref{sec:kwise} for higher order comparisons and choices. 
Beyond its success in 
classical applications such as transportation and marketing, 
the MNL model and its variants are being rediscovered and successfully applied 
in more modern applications such as embedding images using crowdsourcing \cite{tamuz2011adaptively} 
and word embedding \cite{mikolov2013distributed}, 
whose connections we make precise in 
Section \ref{sec:discussion}.

Motivated by recent advances in learning low-rank models, e.g.~\cite{negahban2009unified,DPVW14}, we ask the fundamental question of learning the MNL model 
from data on comparisons and choices.  
We provide a general framework using convex relaxations for learning the model. 
As data is collected in various forms on modern social computing systems, 
we consider the following four canonical scenarios: 
\begin{itemize} 
  \item {\em Pairwise comparisons.} 
    The most simple and canonical piece of ordinal data one can collect from a user at a time is a pairwise comparison; 
    given two options, we ask the user which one is better. Such data is prevalent in the real world 
    and is the most popular scenario studied  in ranking literature, e.g.~\cite{shah2014better}.
    However, one significant aspect of the real data that has not been addressed in the literature is irregularities in the sampling. 
    Consider an online seller with various products, say cars and watches. 
     It does not make sense to ask a user to compare a car and a watch; 
     one cannot sample an  outcome of a comparison between a watch and a car. 
    However, knowing a user's preference on cars can help in learning her preference on watches. 
    We want to propose a model and design an inference algorithm that can 
    take into account such  restrictions in sampling.  
    We further want to quantify the gain in using all such data together in inference, as opposed to 
    running inference in each category separately.  
    To this end, we propose a new model for sampling that we call {\em graph sampling}. 
    This model explains such irregularities in the real world data.  
      We propose a novel inference algorithm tailored for the given sampling pattern. 
    Our analysis captures precisely how the accuracy depends on the different topologies of the sampling.
    
  \item {\em Higher order comparisons. } 
    Consider an online market that 
    collects each of its  user's preference as a ranking over a subset of items that is `seen' by the user. 
    Such data can be obtained by directly asking to compare some items, or 
    by indirectly tracking online activities on 
    which items are viewed, how much time is spent on the page or 
    how the user rated the items. 
    However, collecting such comparisons over multiple items might come at a cost. 
    We, therefore, want to quantify the gain in the accuracy of 
    the inference when higher order comparison outcomes are collected. 
    We characterize the optimal trade-off between accuracy and the number of items compared,  
    and show that our proposed algorithm seamlessly generalizes to this setting and also achieves the optimal trade-off.

  \item {\em Customer choices.} 
    One of the most widely applicable data collection scenarios is customer purchase history. 
    Online and offline service providers can track each customer on 
    which subset of items is offered and which item is chosen. 
    Given historical data on such choices on best-out-of-a-subset, 
    we extract features on the users and items that best explains the collected data.

  \item {\em Bundled choices.} 
    Another data collection scenario that is gaining interest 
    recently is bundled choices \cite{bundle1,BKT18}. 
    Typical choice models assume that the willingness to buy an item is independent of what else the user bought. 
    In many cases, however, we make `bundled' purchases: 
    we buy particular ingredients together for one recipe or  we buy two connecting flights. 
    One choice (the first flight) has a significant impact on the other (the connecting flight).  
    In order to optimize the assortment (which flight schedules to offer) for maximum expected revenue, 
    it is crucial to accurately predict the willingness of the consumers to purchase bundled items, 
    based on past history.   
    We propose a model that can  capture such interacting preferences for bundled items (e.g. jeans and shirts),   
    and use this model to extract the features of the items in each category from 
    historical bundled purchase data. 
    Both our inference algorithm and the analyses extend to this setting, achieving the optimal trade-off 
    between  sample size and accuracy. 
\end{itemize} 

\bigskip\noindent
{\bf Contribution.} 
We first study the canonical scenario of pairwise comparisons 
from the MNL model in Section \ref{sec:graphsampling}. 
Our contribution in the modeling is 
a new sampling scenario we call {\em graph sampling} 
that captures how different pairs of items have varying likelihood of being compared together. 
Our algorithmic contribution is 
a convex relaxation with a new regularizer using a variation of 
the standard  nuclear norm tailored for the graph sampling topology. 
Our theoretical contribution is 
in the analysis of the proposed estimator and a matching 
fundamental lower bound (up to a poly-logarithmic factor). 
This $(a)$ characterizes the minimax sample complexity of the problem; 
$(b)$ proves that the proposed estimator cannot be improved upon; 
and $(c)$ identifies  how the accuracy depends on 
the  topology of sampling. This in turn provides a guideline for 
designing surveys when one has a choice on  which pairs are to be compared. 
This is accompanied by experiments on synthetic and real data sets confirming our theoretical predictions.

This framework is extended to  higher order comparisons in 
Section \ref{sec:kwise}. We establish minimax optimality (up to a poly-logarithmic factor) of our estimator and identify the 
fundamental trade-off between accuracy and sample size. When each user provides 
a total linear ordering among $k$ items, we show that the required sample size effectively is reduced by a factor of $k$. 
When the user provides her best choice (as in purchase history) instead of the total linear ordering, 
we extend our framework and establish minimax optimality in Section \ref{sec:customer}. 
We also consider a bundled purchase scenario in Section \ref{sec:bundle}, 
where customers buy pairs of items from each of the two categories. 
We extend our framework and establish minimax optimality under the bundled purchase setting. 
We present experimental results on both synthetic and real-world data sets confirming our theoretical predictions and showing 
the improvement of the proposed approach in predicting users' choices\footnote{Code for our experiments are available at  https://github.com/POLane16/Nucnorm-Ranking.}.

Technically, 
we borrow analysis tools from 
1-bit matrix completion \cite{DPVW14}, matrix completion \cite{NW11}, and restricted strong convexity \cite{negahban2009unified}, 
and  crucially utilize the Random Utility Model (RUM) \cite{Thu27,Mar60,Luce59} 
interpretation (outlined in Section \ref{sec:MNL}) of the MNL model to 
prove both the upper bound and the fundamental limit.  This could be of interest to analyzing more general class of RUMs.

\bigskip\noindent
{\bf Notations.} 
We use $\fnorm{A}$ and $\lnorm{A}{\infty}$
to denote  the 
Frobenius norm and the $\ell_\infty$ norm, 
$\nucnorm{A}=\sum_{i}\sigma_i(A)$ to denote the nuclear norm 
where $\sigma_i(A)$ denotes the $i$-th singular value, 
and $\lnorm{A}{2}=\sigma_1(A)$ for the spectral norm.
We use $\llangle u,v \rrangle = \sum_i u_iv_i$ 
and $\|u\|$ to denote the inner product and the Euclidean norm. 
All ones vector is denoted by $\ones$, $\id$ denotes the identity matrix and 
$\ind{(A)}$ is the indicator function of the event $A$.
The set of the first $N$ integers are denoted by $[N]=\{1,\ldots,N\}$.

\subsection{Related Work} 

{\bf Bradley-Terry and Plackett-Luce models.} 
The simplest form of the MNL model is when all users are sharing the same feature vector such 
that each item is parametrized by a scalar value. 
This is known as Bradley-Terry (BT) model when pairwise comparisons are concerned and 
Plackett-Luce (PL) model when higher order comparisons are concerned. 
This has been proposed and rediscovered several times in the last century  
\cite{Zer29,Thu27,BT55,Luce59,Pla75,McF73,McF80} in the context of 
ranking teams in sports games, ranking items based on surveys, and ranking routes in transportation systems.   
Unlike the general MNL model, maximum likelihood estimator for the BT and PL models are 
naturally convex programs. 
However, learning the BT model has first been addressed in \cite{Ford57} 
where the convergence of the iterative algorithm is analyzed, without explicitly 
 relying on the convexity of the problem. 
A new algorithm based on Majorize-Minimize framework was proposed in \cite{Hun04}. 
First sample complexity of learning BT model was provided in \cite{NOS12} 
where a novel estimator, called Rank Centrality, of the BT parameters was proposed. 
The authors construct a random walk over a graph where the nodes are the items and 
the transition probability is constructed from the comparisons outcomes. 
This spectral approach is proven to achieve a minimax optimal sample complexity. 
This has been a building block for several ranking algorithms, 
which further process the Rank Centrality to get better accuracy on top of it \cite{CS15,JKSO16,jang2017optimal,CFM17}. 
For higher order comparisons, the sample complexity of learning PL model was provided in \cite{HOX14,shah2014better}, 
where the Maximum Likelihood (ML) estimator is shown to achieve the minimax optimality. 
Later, \cite{MG15} made the connection between the spectral approach of Rank Centrality 
and the ML estimator precise by providing a unifying random walk view to the problem. 
This led to a novel Accelerated Spectral Ranking algorithm introduced in 
\cite{APA18}, which not only finds the parameters of the PL model more efficiently in computation, 
but also achieves optimal sample complexity under general sampling graphs. 
Recently, \cite{16BKM} treat the learning problem as solving a noisy linear system, and propose an algorithm that is 
amenable to on-line, distributed and asynchronous variants. 
\cite{vojnovic2016parameter} analyzes a more general class of random utility models 
 known as  Thurstone models, 
 and provide the minimax sample complexity by analyzing the ML estimator.  
Note that the ML estimators for Thurstone models in general are computationally intractable.

\bigskip\noindent
{\bf Generalized BT and PL models.}
As studied in  \cite{RA14}, 
the BT model covers a subset of probabilistic models over comparisons. 
There is a hierarchy of models with increasing complexity and descriptive power. 
One popular extension is the mixture of BT or PL models. 
It is known that any choice model can be approximated  arbitrarily close 
with a mixed PL model with sufficient number of mixture components \cite{MT00}. 
The sample complexity of learning a mixed PL model was analyzed in 
\cite{OS14} where a tensor decomposition for learning a mixture model was proposed and analyzed 
under some separation conditions between the weights of the mixtures.
For a mixture of two PL model, \cite{CKT18} shows identifiability and uniqueness of the mixture weights,  
when all marginal probability over all possible rankings among two items and three items are known. 
In a crowdsourced setting, \cite{chen2013pairwise} models pairwise comparisons using a mixture of PL models consisting of hammer distribution, which reports the true output of a comparison, and spammer distribution, which reports the exact opposite of a comparison. 
A different approach that tackles the problem by learning to  cluster the users based on the pairwise comparisons is proposed in 
\cite{RXR15}.
The MNL model we study in this paper can be thought of as a  generalization of the mixed PL models, 
where each user has her own preference. To make learning feasible, we inherently impose similarities among users 
via a low-rank condition. Note that a mixed PL model with $r$ mixture is a special case of the MNL model with rank $r$, 
where each user's membership is encoded as a $r$-dimensional feature in standard basis.
In the context of collaborative ranking, 
 algorithms for learning the MNL model from  
 pairwise comparisons have been proposed in \cite{PNZSD15}. 
Instead of nuclear norm regularization as we propose in this paper, \cite{PNZSD15} proposes solving a convex relaxation of maximizing the likelihood over matrices with bounded nuclear norm. 
Under the standard assumption of uniformly chosen pairs, 
it is shown that this approach achieves statistically optimal generalization error rate, instead of Frobenius norm error that we analyze.

\bigskip\noindent
{\bf Beyond BT and PL models.}
Modeling choice is an important problem where the ultimate goal is to 
find the right parametric model to capture human choices. 
 \cite{RU16,BGG13} use Markov chains to model choices with the parameters in the transition matrix defining the probability model. Ideal point model \cite{massimino2018you, kazemi2018comparison} assumes that the pairwise comparisons of two items by a user depends on their distance from an ideal item (ideal point) for the user in some metric embedding space of the items. 
Novel nonparametric models have also been proposed to model human choices, for example 
\cite{SBGW16, pananjady2017worst,FJO18} uses strong stochastic transitivity to model pairwise choices 
and \cite{FJS09} uses  distribution over all permutations with sparse support to model higher order choices.
We also note that in the context of (non-collaborative) ranking, \cite{gleich2011rank} has proposed nuclear norm minimization based algorithm when comparisons between all pairs items are modeled as a low-rank skew-symmetric matrix.
Other non-parametric approaches to solving ranking include empirical risk minimization. \cite{clemenccon2005ranking} analyses risk minimization of U-statistics and a more feasible surrogate convex loss minimization to estimate ranking. \cite{katz2017nonparametric} assumes that rating of an item by a user is a Lipschitz function of the user-item pair and analyses a nonparametric collaborative ranking algorithm from partial observation of such ratings.

While we are interested in the (parameters of) full ranking over all items, 
 there have been several recent works which aim to only approximately rank the items, such as retrieving only the top-$m$ items \cite{CS15, JKSO16, jang2017optimal} or partitioning the items into ordered buckets of fixed size \cite{katariya2018adaptive,heckel2018approximate}.
\section{Model and Approach for Pairwise Comparisons}
\label{sec:model}
The MultiNomial Logit (MNL) model is one of the most popular models that 
explains how people make choices when given multiple options and 
is widely used in behavioral psychology and revenue management.
For brevity, we focus our discussion on data collected in the form of pairwise comparisons in Sections \ref{sec:model} and \ref{sec:graphsampling}, 
and defer the discussion of the MNL model in its full generality to Sections \ref{sec:kwise} and \ref{sec:bundle} .
We give a precise definition of the model for paired comparisons and provide a novel algorithmic solution to learn this model from samples.

\subsection{MultiNomial Logit (MNL) Model for Pairwise Comparisons}
\label{sec:MNL}
Let $\Theta^*$ be a $d_1 \times d_2$ dimensional matrix capturing preferences of $d_1$ users on $d_2$ items. The probability with which a user, $i \subseteq [d_1]$, when presented with two items $j_1, j_2 \subseteq [d_2]$, prefers item $j_1$ over item $j_2$ is,
\begin{align}
\prob{j_1 > j_2} = \frac{e^{\Theta^*_{ij_1}}}{e^{\Theta^*_{ij_1}} + e^{\Theta^*_{ij_2}}}\,.
\label{eq:defbtl}
\end{align}

This implies that, more preferred items (as per the ordering of $\Theta^*_{ij}$) are more likely to be ranked higher, with the randomness in choices captured by the probabilistic model. 

If we do not impose any further constraints on $\Theta^*$, one entry of $\Theta^*$ is not related in any way to any other entries. 
This implies that one user's preference is completely independent of others' and no efficient learning is possible. 
Each user's preference has to be learned separately. 
On the other hand, in real applications, it is reasonable to say that preferences of users depend only on a handful of factors for example, quality, price, and aesthetics. 
We do not know which features affect users' choices, but we assume that there are  $r$-dimensional latent features for each of the users and items that govern such choices, and that $r\ll d_1,d_2$.  
This assumption mathematically captures the conventional belief that 
 when two people have similar preferences over a subset of items, they tend to have similar tastes on other items as well. 
 Formally, MNL model  assumes that $\Theta^*$ is a rank $r$ matrix with $r \ll d_1, d_2$. 
In this paper, we do not impose a hard constraint on the rank and 
provide general results for matrices of any rank. 
In this case, we identify how the accuracy depends on the rate of decay of the singular values. 

This MNL model has many roots. 
In revenue management, 
this has been proposed as  a special case of Random Utility Model (RUM). 
RUM explains choices that a person makes as the result of maximizing perceived random utilities associated with the set of alternatives presented. In the case of MNL, each decision maker and each alternative are associated with an $r$-dimensional vector, $u_i$ and $v_j$, resulting in a low-rank $\Theta^*$ if $\Theta^*_{ij} = \Iprod{u_i}{v_j}$. The perceived utility of the item $j$ for decision maker $i$ is,
\begin{align}
U_{ij} = \Iprod{u_i}{v_j} + \xi_{ij}\,,
\end{align}
where $\xi_{ij}$'s are i.i.d. random variables following the standard Gumbel distribution.
Different choices of distributions give different variants of RUMs. 
In our analyses, we utilize this RUM interpretation of the MNL model to 
prove a particular concentration in Section \ref{sec:kwise_hessian3_proof}, for example. 
The model in Equation.~\eqref{eq:defbtl} has also been re-discoverd several times in the literature 
\cite{Zer29,Thu27,Luce59,BT55} in several domains. 

\subsection{Low-rank Regularization using Nuclear Norm Minimization}
\label{sec:opt}

Given the low-rank structure of the model, a natural but inefficient approach is to minimize the negative of the log likelihood, $\cL(\cdot)$, regularized by the rank: 
\begin{align}
  \widehat{\Theta} \,\; \in\; \, \arg \min_{\Theta \in \reals^{d_1\times d_2 }} \,-\cL(\Theta) \,+\, \lambda \, {\rm rank}(\Theta),
\end{align} 
for some parameter $\lambda>0$. 
As this rank minimization is a notoriously challenging problem, we instead solve a convex relaxation of it.
Note that the nuclear norm ball is the convex hull of rank-$1$ matrices \cite{RFP10}. 
Analogous to $l_1$-norm in the case of sparse vectors, nuclear norm is a tight convex surrogate for low-rank solutions. We propose the following  nuclear norm regularized optimization problem,
\begin{align}
  \widehat{\Theta} \,\; \in\; \, \arg \min_{\Theta \in \Omega} \,-\cL(\Theta) \,+\, \lambda \, \nucnorm{\Theta},
  \label{eq:nucnorm-opt}
\end{align} 
where $\Omega$ is a convex constraint which takes care of identifiability and Lipschitz smoothness conditions. Nuclear norm regularization has been widely used \cite{RFP10} for rank minimization; however, provable guarantees exist only for quadratic loss functions $\calL(\Theta)$ \cite{CR09,NW11}. Our analyses extend such results to a convex loss, by first proving that $-\cL(\cdot)$ satisfies restricted strong convexity property with high probability.
Similar to how (non-collaborative) rank aggregation has been generalized to any strongly log-concave distribution 
in \cite{shah2014better}, 
our analysis can naturally be extended to a general class of strongly log-concave distributions.
We give the expression for the log likelihood in Equation.~\eqref{eq:defpairlikelihood} for pairwise comparisons.

\section{Learning MNL Model from  Pairwise Comparisons under Graph Sampling}
\label{sec:graphsampling}

\noindent
{\bf Probabilistic model for sampling.} In order to provide performance guarantees on the proposed approach, we need to specify 
how we sample the pairs that are to be compared.  
We provide a novel sampling model, which we call {\em graph sampling} with respect to a weighted  graph $\cG$.  
This naturally generalizes Bernoulli sampling typically studied under matrix completion literature \cite{CR09,KMO10IT,NW11,JNS13}, and 
the resulting analysis captures how the performance depends on the topology of the samples. 
Note that the proposed graph sampling is different from 
deterministic sampling graphs studied in \cite{HOX14, shah2016estimation}. 
This is analytically tractable only in 
the simpler case of estimating the weight vector of the PL model where there is only one user and the ML estimator is a convex program.  
However, such deterministic sampling is notoriously hard to handle for matrix estimation, even in the simpler case of matrix completion \cite{bhojanapalli2014universal}. 
Hence, we introduce a probabilistic model that allows enough flexibility to capture the interesting aspects of sampling biases, i.e.~grouping.

Precisely,
we have  a weighted undirected graph $\cG=([d_2],E,\{P_{j_1,j_2}\}_{(j_1,j_2) \in E})$ 
with $d_2$ nodes, which represent items, a set of edges $E$ and the edge weight $P_{j_1,j_2}$ between nodes $j_1$ and $j_2$. 
The weights can be written in a symmetric matrix $P\in\reals^{d_2\times d_2}$, 
and $P_{j_1,j_2}+P_{j_2,j_1}=2P_{j_1,j_2} $ represent the 
probability with which the pair $(j_1, j_2)$ is chosen for comparison. 
Note that $P_{j,j} = 0\,\;, \forall j \in [d_2]$, $P_{j_1,j_2} = P_{j_2,j_1}$ and $\sum_{j_1,j_2 \in [d_2]} P_{j_1,j_2} = 1$.
We assume we get i.i.d. samples from first choosing a random user among $[d_1]$ users, 
 and then choosing a pair $(j_1,j_2)$ of items at random from $P$, and finally 
 getting a random comparison from the MNL model, i.e.  
the probability with which user $i$ prefers item $j_1$ over item $j_2$ is $\exp{\Theta^*_{ij_1}} / \left( \exp{\Theta^*_{ij_1}} + \exp{\Theta^*_{ij_2}} \right)$. 

One of the most important aspects of real-world data that is captured by this graph sampling model is grouping. 
Consider two groups of items, say, cars and phones. 
It does not make sense to ask an individual to compare a phone with a brand of a car 
(i.e.~direct comparison is not feasible), 
but knowing an individual's preference on cars can help in learning her preference on phones.  
In graph sampling terms, 
we are sampling from a graph 
$\cG$ consisting of  two disjoint cliques: one for cars and another for phones. 
By analyzing such a sampling scenario, we want to characterize the gain in  
using the data from both groups of items together, although there are no inter-group comparisons.

In the preference matrix $\Theta^*$, the values in the set of columns corresponding to each connected component in the sampling graph can be 
arbitrarily shifted together, without changing the pairwise comparisons outcome distributions. 
This is because adding the same constant to those items that are compared does not change the probability (for those items within the same group), i.e.~
\begin{eqnarray*}
  \prob{j_1 > j_2} \;=\; \frac{e^{\Theta^*_{ij_1}}}{e^{\Theta^*_{ij_1}} + e^{\Theta^*_{ij_2}}} \;=\; 
    \frac{e^{\Theta^*_{ij_1}+c }}{e^{\Theta^*_{ij_1}+c} + e^{\Theta^*_{ij_2}+c}} \;\,,
\end{eqnarray*}
and adding different constants to those items that are not in the same group does not change the probability of the outcome as those items are  never compared. 
Hence, to handle this unidentifiability, we let a centered version of $\Theta^*$ represent all those shifted versions defining the same probability distribution. 
Formally,  let a zero-one vector $g_k \in \{0, 1\}^{d_2}$ denote the group membership such that 
$g_{i,k} = 1$ if item $j$ is in group $k$, else $g_{i,k}=0$. Note that, by definition, 
no item can be present in more than one group, that is, $\sum_{k=1}^G g_k = \ones$, where $G$ is the number of groups. 
We define an equivalence class of $\Theta^*$ which represent the same probabilistic model as 
\begin{align}
[\Theta^*] \;=\; \Big\{\Theta^* + \sum_{k=1}^G u_k g_{k}^\top \;\text{ for all } u_k \in \reals^{d_1} \Big\}\;. 
\end{align}
To overcome the identifiability issue, we represent each equivalence class with the centered matrix satisfying 
\begin{align}
\label{eq: graph_def_group}
\Theta^* g_k = 0,\;\;\; \forall\; k \in \{1,2, \ldots, G\}
\end{align} 

As matrices with large ``spikiness'' are known to be hard to estimate \cite{NW11}, 
we capture the dependence of the sample complexity on the spikiness as measured by 
 $\bb := \lnorm{\Theta^*}{\infty}$. 
 This captures the dynamic range of the underlying preference matrix. 
For a related problem of matrix completion, where the loss $\calL(\theta)$ is quadratic, 
either a similar condition on $\ell_\infty$ norm is required or 
another condition on incoherence is required.

\bigskip
\noindent
\textbf{Graph Laplacian.} The performance of our approach depends on the sampling graph $P$ via its 
 graph Laplacian defined as 
\begin{align}
\label{eq:graph_laplacian}
  L \;=\; \diag(P \ones) - P\,
\end{align}
where $\diag(P \ones)$ is a diagonal matrix with $\sum_v P_{u, v}$ in the diagonals. 
Notice that, $L$ is singular and the nullspace is spanned by vectors $\{g_k\}_{k=1}^{G}$. Let $\sigma_{\max}(L) = \opnorm{L}$ and $\sigma_{\min}(L)$ be the smallest eigenvalue of $L$ discounting the $G$ zero-valued eigenvalues. Since the graph has $G$ disconnected maximal components and $L$ is real symmetric, by spectral theorem, $L = U \Sigma U^\top$, where $U$ is a matrix of size $d_2 \times (d_2 - G)$ and its $d_2 - G$ columns form an orthonormal set, and $\Sigma$ is a diagonal matrix such that its diagonal elements are the singular values of $L$. Let $L^\dagger : = U \Sigma^{-1} U^\top$ and $L^{x} := U \Sigma^{x} U^\top$ for all $x \in \reals$. We also define the Laplacian induced norms of matrices as,
\begin{align*}
\Lnorm{\Theta} := \Fnorm{\Theta L^{1/2}}, \text{ and, } \Lnucnorm{\Theta} := \nucnorm{\Theta L^{1/2}} \;.
\end{align*}
These Laplacian induced norms are more appropriate to analyze and quantify the distance between the estimated matrix $\hTheta$ and $\Theta^*$. 

When items $k(i)$, $l(i)$ are chosen for comparison by user $j(i)$ as the $i$-th pair of items, we capture this choice with the matrix $X^{(i)} =  e_{j(i)} (e_{k(i)} - e_{l(i)})^\top$. The outcome of the comparison is represented by $y_i$, with $y_i=1$ when item $k(i)$ wins over item $l(i)$ and $y_i=0$ if otherwise. 
The log-likelihood of the comparison outcomes with respect to a parameter matrix $\Theta$ is,
\begin{align}
  \cL(\Theta) \; =\;  \frac1n \sum_{i=1}^n \left[ y_i \Iprod{\Theta}{X^{(i)}} - \log\left(1 + \exp\left(\Iprod{\Theta}{X^{(i)}}\right) \right) \right]\;.
  \label{eq:defpairlikelihood}
\end{align}
We propose and analyze the following convex optimization problem,
\begin{align}
\label{eq:graph_obj}
  \hTheta  \; \in \; \underset{\Theta \in \Omega_\bb}{\text{argmin}} -\calL(\Theta) + \lambda \Lnucnorm{\Theta},
\end{align}
where,
\begin{align}
\label{eq:defgraphomega}
\Omega_\bb \; =\;  \left\lbrace \, \Theta \in \reals^{d_1 \times d_2} \;|\; \lnorm{\Theta}{\infty} \leq \bb, \Theta g_k = 0,\; \forall\, k\in [G] \, \right\rbrace,
\end{align}
with an appropriately chosen  $\lambda = 8\sqrt{2}\max\left\{ \sqrt{\frac{ \sigma \log (2d)}{n}}, \frac{\sigma_{\min}(L)^{-1/2}\log(2d)}{n} \right\}$ with $\sigma = \max \{(d_2-G)/d_1,1\}$, where $d=(d_1+d_2)/2$. 
In practice, the sampling probability distribution $P$ and the corresponding Laplacian $L$ might not be known. 
In those cases, we propose using the empirical sampling probability distribution $\hat P$ and corresponding empirical Laplacian $\hat L$ instead. 
We describe this version of the algorithm formally in Section \ref{sec:robust}, where 
we empirically demonstrate the robustness of this approach. 
Further, in experiments with real data sets, we 
use the empirical Laplacian Section~\ref{sec:food}.

\subsection{Performance Guarantee}
\label{sec:main_ub}

We consider the graph sampling scenario where each sample is i.i.d., 
the $\ell$-th sample consists of user $i_\ell$ chosen uniformly at random, pair of items $(j_{1,\ell},j_{2,\ell})$ chosen 
according to the sampling graph $\cG=([d_2],E,P)$, and the resulting outcome $y_\ell$ distributed as the MNL model with parameter $\Theta^*$. 

\begin{theorem} \label{thm:graph_ub}
  Under the graph sampling with respect to $\cG=([d_2],E,P)$ with a graph Laplacian 
  $L$, and under the MNL preference model with  
  preference matrix $\Theta^*$, solving the optimization problem in \eqref{eq:graph_obj} with $n$ i.i.d.~samples  achieves, 
  with probability greater than $1-1/4d^3$, 
  \begin{align}
  \frac{1}{d_1}{\Fnorm{\left(\Theta^* - \hTheta\right) L^{1/2}}}^2 \le 36 \lambda \left( \bb + \frac1{\psi(2\bb)} \right) \bigg(&\sqrt{2r}   \Fnorm{\left(\Theta^* - \hTheta\right) L^{1/2}} + \nonumber \\& \sum^{\min\{d_1,d_2 - G\}}_{j = r+1}\sigma_j(\Theta^* L^{1/2})\bigg),
  \end{align}
  for any $r\in \{ 1,2,\ldots, \min\{d_1,d_2-G\}\}$, any 
   $\lambda \geq 8\sqrt{2}\max\left\{ \sqrt{\frac{ \sigma \log (2d)}{n}}, \frac{\sigma_{\min}(L)^{-1/2}\log(2d)}{n} \right\}$ 
   where $\sigma= \max\{(d_2-G)/d_1, 1\}$ and $d=(d_1+d_2)/2$, 
   $\psi(x) \triangleq {e^x}/(1+e^x)^2$, 
   and  for 
  $n \leq \min \{2^2 \left(d_1 \sigma_{\min}(L)^{-1}\right)^{2/3}\,\log(2d), \\ 2^6 d_1^2 \sigma^2\,\log(2d)\}$.  
\end{theorem}
\noindent
We provide a proof in Appendix \ref{sec:graph_ub_proof}. The above bound holds for any $r$, where $r$ allows us to trade off the two types of errors: the estimation error and the approximation error. 
Concretely, the above bound shows a natural splitting of the error into two terms;
the first term corresponding to the {\em estimation error}  for the top rank-$r$ component of $\Theta^*$ and 
the second term corresponding to the {\em approximation error} for how well one can approximate $\Theta^*$ with a rank-$r$ matrix. 
If we know the singular values of $\Theta^*$, we can optimize over $r$ to get the tightest bound. 
If $\Theta^*$ is exactly low-rank then applying a matching rank in the bound gives the following guarantee.

\begin{corollary}[{\bf Exact rank-$r$ matrix}]
  \label{coro:graph_ub} 
  Under the same hypothesis as in Theorem \ref{thm:graph_ub} with a choice of 
  $\lambda = c_0 \max\left\{ \sqrt{\frac{ \sigma \log (2d)}{n}}, \frac{\sigma_{\min}(L)^{-1/2}\log(2d)}{n} \right\} $ 
  for some $c_0>0$, if $\Theta^*$ is exactly rank $r$,  there exists a positive constant $c_1$ such that  the proposed estimator achieves,  
  \begin{align}
  \frac{1}{\sqrt{d_1}}{\Fnorm{\left(\Theta^* - \hTheta\right) L^{1/2}}} \le c_1 \left (\bb+\frac{1}{\psi(2\bb)} \right)\sqrt{r} \max\bigg\{&\sqrt{\frac{\sigma\; d_1 \log (2d)}{n}}, \nonumber\\ &\frac{ \sqrt{\left( \sigma_{\min}(L)^{-1} d_1\right)}\; \log(2d)}{n} \bigg\},
  \label{eq:main_ub}
  \end{align}
  with probability at least $1-2/(d_1+d_2)^{3}$ 
  and $\sigma= \max\{(d_2-G)/d_1, 1\}$.
\end{corollary}
The second term in the maximization is 
an artifact of the weakness of current analysis technique 
and does not reflect the actual error. 
This is confirmed in our simulation results on graphs with very small spectral gap in 
Figures \ref{fig:graph_errorVn}.(b), (d), and (f), where the error in Laplacian-induced norm error does not decrease with spectral gap of $L$ as the line graph has a much smaller spectral gap compared to a complete graph, for example. In fact, for a special $\Theta^*$ in Figure \ref{fig:graph_errorVn}.(d) it is the other way, for which we do not have a theoretical explanation. 

The number of entries in $\Theta^*$ is $d_1d_2$ and we want to rescale the 
Frobenius norm error appropriately by $1/\sqrt{d_1d_2}$. 
As a typical scaling of $L^{1/2}$ is $1/\sqrt{d_2}$ in spectral norm, 
we only need to rescale the Laplacian-induced norm error by $1/\sqrt{d_1}$ in the 
left-hand side of the above bound. For a rank-$r$ $\Theta^*$, 
the number of degrees of freedom in describing it is $r(d_1+d_2)-r^2=O(r(d_1+d_2))$. 
The above theorem shows that the total number of samples $n$ needs to scale as 
$O(r(d_1+d_2)\log d)$ in order to achieve an arbitrarily small error. 
This is only a poly-logarithmic factor larger than the degrees of freedom. 
In Section \ref{sec:graph_info_lb} we make this comparison precise by 
providing a lower bound that matches the upper bound up to a logarithmic factor. 

The upper-bound constraint in Theorem \ref{thm:graph_ub} on the number of samples $n$ can be met for large enough $d_1$ and $d_2$.  
For simplicity, assume that $d_1=d_2=d$ and  $r$ is a constant. 
Since $\sigma_{\min}(L) = O(1/d_2)$, the upper-bound on $n$ becomes $O(\max\{d^2, d^{4/3}\})$. 
For large enough $d$, the upper bound on the RHS of Eq. \eqref{eq:main_ub} 
can be made arbitrarily small with $n$ only scaling as $O(r\,d)$. 
This is significantly smaller than the upper-bound of $O(d^2)$ on $n$. 
Further, in the experiments in Section \ref{sec:graph_expts}, we show that $n$ has no practical upper-bound constraint since the error decreases at the same rate as predicted, for arbitrarily large values of $n$. This constraint may not be necessary and might be a by-product of the proof techniques.

The dependence 
on the dynamic range $\bb$, however, is sub-optimal. 
It is expected that the error increases with $\bb$, since the $\Theta^*$ scales as $\bb$, but the exponential dependence in the bound 
seems to be a weakness of the analysis (for example as seen from numerical experiments in the right panel of Figure \ref{fig:kwise}). 
Although the error increase with $\bb$, numerical experiments suggest that it only increases at most linearly. 
However, tightening the scaling with respect to $\bb$ is a challenging problem, and 
such sub-optimal dependence is also present in 
existing literature for learning even  simpler models, such as the Bradley-Terry model \cite{NOS12} 
 or the Plackett-Luce model \cite{HOX14}, which are special cases of the MNL model studied in this paper.

Another issue is that the underlying matrix might not be exactly low rank. 
It is more realistic to assume that it is approximately low rank. 
Following \cite{NW11} we formalize this notion with ``$\ell_q$-ball'' of matrices defined as 
\begin{eqnarray}
  \bB_q(\rho_q) &\equiv & \{ \Theta\in\reals^{d_1\times d_2} \,|\, \sum_{j\in[\min\{d_1,d_2\}]}^{} |\sigma_j(\Theta^*)|^q\leq \rho_q  \} \;. \label{eq:defBq}
\end{eqnarray}
When $q=0$, this is a set of rank-$\rho_0$ matrices. 
For $q\in(0,1]$, this is set of matrices whose singular values decay relatively fast. 
By optimizing the choice of $r$ in Theorem \ref{thm:graph_ub}, we get the following result. 

\begin{corollary}[{\bf Approximately low-rank matrices}]
  Suppose $\Theta^* \in \bB_q(\rho_q)$ for some $q\in(0,1]$ and $\rho_q>0$. 
  Under the hypotheses of Theorem \ref{thm:graph_ub}, 
   with a choice of 
  $\lambda = c_0 \max\bigg\{\sqrt{\frac{ \sigma \log (2d)}{n}}, \frac{\sigma_{\min}(L)^{-1/2}\log(2d)}{n} \bigg\} $  
  for some constant $c_0>0$ 
  there exists a constant $c_1>0$ such that 
    solving the optimization \eqref{eq:graph_obj}  
  achieves  with probability at least $1-2/(d_1+d_2)^{3}$, 
  \begin{eqnarray}
  \frac{1}{\sqrt{d_1}}{\Fnorm{\left(\Theta^* - \hTheta\right) 
  L^{1/2}}}
   \,\leq\, 
   \frac{c_1 \, \sqrt{  \rho_q}}{\sqrt{d_1}}  \left( \Big(\bb+ \frac{1}{\psi(2\bb)}\Big) \,
   \sqrt{\frac{d_1^2 \, \sigma \, \log (2d)) }{n}} \right)^{\frac{2-q}{2}}\;,
  \label{eq:graph_appxlowrank}
  \end{eqnarray}
  provided $n \geq \sigma \, \log (2d) / \sigma_{\min}(L)$.
  \label{cor:graph_appxlowrank}
\end{corollary}

This is a strict generalization of Corollary \ref{coro:graph_ub}. 
For $q=0$ and $\rho_0=r$, this recovers the exact low-rank estimation bound up to a factor of two. 
For approximate low-rank matrices in an $\ell_q$-ball, we lose in the error exponent, which reduces from one to $(2-q)/2$.

\subsection{Information-theoretic Lower Bound} 
\label{sec:graph_info_lb}
For a polynomial-time algorithm of convex relaxation, we gave in the previous section a bound on the  achievable error. We next compare this to the fundamental limit of this problem, 
by giving a lower bound on the achievable error  by any algorithm (efficient or not). 
A simple parameter counting argument indicates that 
it requires the number of samples to scale as the number of degrees of freedom i.e.,   
$n \propto r(d_1+d_2)$, 
to estimate a $d_1\times d_2$ dimensional matrix of rank $r$. 
We construct an appropriate packing over the set of low-rank matrices with 
bounded entries in $\Omega_\bb$ defined as \eqref{eq:defgraphomega}, 
and  show that no algorithm can accurately estimate 
 the true matrix with high probability using the generalized Fano's inequality. 
This provides a constructive argument to lower bound the minimax error rate, which in turn establishes that the bounds 
in Theorem \ref{thm:graph_ub} is sharp up to a logarithmic factor, and proves no other algorithm can significantly improve over the nuclear norm minimization. 

\begin{theorem}
  Suppose $\Theta^*$ has a rank $r$. Under the previously described graph based sampling model, there exists a  
   constant $c>0$ such that
  \begin{align}
  \inf_{\hTheta} \sup_{\Theta^*\in\Omega_\bb}\E\Big[
  \frac{1}{\sqrt{d_1}}{\Fnorm{\left(\Theta^* - \hTheta\right) L^{1/2}}}
  \Big] \geq c\,\min\bigg\{&e^{-\bb} \sqrt{\frac{\,r\,d_1}{n}}  
  \,,\, \nonumber\\ &\bb 
  \max\left\{\sqrt{\frac{r}{\trace{L^\dagger_r}}} \,,\,
  \frac{d_2}{\sqrt{d_1 \log d}}
  \right\}
  \Bigg\} \;,
  \end{align}
  where the infimum is taken over all measurable functions over the observed comparison results and $L_r^\dagger$ is the pseudo inverse of the rank $r$   
  approximation of the graph Laplacian.
  \label{thm:graph_lb}
\end{theorem}

A proof of this theorem is provided in Appendix \ref{sec:graph_lb_proof}. 
The term of primary interest in this bound is the first one, 
which shows the scaling of the (rescaled) minimax rate as $\sqrt{rd_1/n}$ and matches the upper bound in \eqref{eq:main_ub} up to a logarithmic factor.
It is the dominant term in the bound whenever 
the number of samples is larger than 
$n\geq d_1 \max \{\trace{L_r^\dagger}\,, d_1 \log d / d_2^2 \}$. 
As suggested in numerical simulations 
 on graphs with very small spectral gap in 
Figures \ref{fig:graph_errorVn}.(b), (d), and (f), 
the dependence in $\trace{L_r^\dagger}$ is an artifact of the weakness of the 
current analysis technique. 
Here we note that, while the lower bound in Theorem \ref{thm:graph_lb} is in expectation, 
the upper bound in Theorem \ref{thm:graph_ub} is a high-probability result. 
The upper bound can immediately be translated into a bound in expected error with an 
additional term scaling as  $\alpha \sigma_{\rm max}(L)^{1/2} \sqrt{d_2} d^{-3}$, which is smaller than other terms in the bound. 

\subsection{Performance Guarantee and Lower Bound for Complete Graph} 
\label{sec:graph_complete_up_lb}

It follows from a simple relation $\vert\kern-0.25ex\vert\kern-0.25ex\vert (\Theta^* - \hTheta) L^{1/2}\vert\kern-0.25ex\vert\kern-0.25ex\vert_F 
\ge \sigma_{min}^{1/2} \vert\kern-0.25ex\vert\kern-0.25ex\vert\Theta^* - \hTheta\vert\kern-0.25ex\vert\kern-0.25ex\vert_F $, which is true since $\Theta^*$, $\hTheta$ are in the range space of $L$, that  
the above upper bounds automatically give the error bound in the Frobenius norm. 
When the sampling graph is uniform, i.e. a complete graph with equal weights $P_{j_1,j_2}=1/d_2(d_2-1)$, $\forall j_1 \neq j_2$, 
Frobenius norm is the right metric and we show matching upper and lower bounds.

\begin{corollary}[{\bf Complete graph upper-bound}]
\label{cor:graph_complete_ub} 
Under the same hypothesis as in Corollary \ref{coro:graph_ub}, if $\cG$ is a complete graph, with a choice of 
  $\lambda = c_0 \max\left\{ \sqrt{\frac{ \sigma \log (2d)}{n}}, \frac{\sqrt{(d_2-1)}\log(2d)}{n} \right\} $ 
  for some $c_0>0$, if $\Theta^*$ is exactly rank $r$,  there exists a positive constant $c_1$ such that  the proposed estimator achieves,  
  \begin{equation}
  \frac{\Fnorm{\Theta^* - \hTheta }}{\sqrt{d_1(d_2-1)}}{} \le c_1 \left (\bb+\frac{1}{\psi(2\bb)} \right) \sqrt{r}\max\Bigg\{ \sqrt{\frac{\sigma\; d_1 \log (2d)}{n}}, \frac{ \sqrt{(d_2-1) d_1}\; \log(2d)}{n} \Bigg\},
  \label{eq:complete_ub}
  \end{equation}
  with probability at least $1-2/(d_1+d_2)^{3}$ 
  and $\sigma= \max\{(d_2-1)/d_1, 1\}$.
\end{corollary}

\begin{corollary}[{\bf Complete graph lower-bound}]
\label{thm:graph_complete_lb} 
Suppose $\Theta^*$ has rank $r$. Under the previously described graph based sampling model with graph being a complete graph, there is a universal numerical constant $c>0$ such that
\begin{align}
  \inf_{\hTheta} \sup_{\Theta^*\in\Omega_\bb} \E\Big[ \frac{1}{\sqrt{d_1\,(d_2-1)}}\Fnorm{\hTheta-\Theta^*} \Big] \geq c\,\min\bigg\{&e^{-\bb} \sqrt{\frac{\,r\,d_1}{n}}  
  \,,\, \nonumber \\ &\bb \max\left\{ \frac{1}{\sqrt{(d_2-1)}}
  \,,\, \frac{d_2}{\sqrt{d_1 \log d}} \right\}
  \Bigg\} \;,
\end{align}
where the infimum is taken over all measurable functions over the observed comparison results.
\end{corollary}

\subsection{Experiments} \label{sec:graph_expts}

We provide a first-order method to solve the proposed convex optimization, and provide numerical experiments using this algorithm. 
We present two simulation results followed by an experiment on real data.

For the synthetic experiments,  we generate random rank-$r$ matrices of dimension $d\times d$, of the form
$\Theta^*=UV^T$ with $U\in\reals^{d\times r}$ and $V\in\reals^{d\times r}$ entries generated i.i.d 
from uniform distribution over $[0,1]$. 
Then the connected-component-mean is subtracted form each connected component, and then the whole matrix is scaled such that the largest entry is $\bb=5$. Note that this operation does not increase the rank of the matrix $\Theta$. This is because this de-meaning can be written as $\Theta - \sum_{k} \Theta g g^\top/(g^\top\ones)$ and both terms in the operation are of the same column space as $\Theta$ which is of rank $r$.

\label{sec:graph_expt}
\subsubsection{Algorithm} \label{sec:graph_algo}

Let $\Theta'  \triangleq \Theta L^{1/2}$. 
As the nuclear norm regularizer in \eqref{eq:graph_obj} 
is non differentiable, 
we use the proximal gradient descent \cite{ANW10,cai2010singular}.  
At each iteration, we apply the following two operations on the current estimate, $\Theta_t'$, of $\Theta^* L^{1/2}$,
\begin{align}
\widetilde{\Theta'}_{t+1} &= \Theta_t' - \eta_{t} \nabla_\Theta \cL(\Theta'_t L^{-1/2})L^{-1/2} &&\text{(gradient descent)}\\
{\Theta'}_{t+1} &= M_t (\Gamma_t - \eta_t \lambda \id )^+ N_t^T &&\text{(singular value shrinkage and thresholding)}
\end{align}
where $M_t \Gamma_t  N_t^T := \widetilde{\Theta'}_t$ is the singular value decomposition of $\widetilde{\Theta'}_t$, such that $\Gamma_t$ is a diagonal matrix with positive entries, $(\cdot)^+$ is the entry-wise thresholding operation $\max(0, x)$, and $\eta_t$ is an appropriate step-size. Constraint of zero row sum, is taken care of by initializing the descent algorithm with $\Theta_0' = 0$, since rows of gradients sum to zero.
In practice we do not know the value of $\alpha$, and hence in experiments we do not enforce the $\| \Theta \|_{\infty}\leq\alpha$ constraint. 

Another issue in the implementation is that the convergence rate can be significantly slower for some  graph topologies. 
We accelerate the proximal gradient descent with the following (modified) Barzilai-Borwein (BB) rule \cite{barzilai1988two} for choosing the step-size $\eta_t$,
\begin{align}
\eta_t =\begin{cases}
    \frac{\lnorm{\Theta'_t - {\Theta'}_{t-1}}{2}^2}{\Iprod{\Theta'_t - {\Theta'}_{t-1}}{\nabla_{\Theta'}\cL'(\Theta'_t) - \nabla_{\Theta'}\cL'({\Theta'}_{t-1})}}, &\text{when $t$ is odd}  \\
    \frac{\Iprod{{\Theta'}_t - {\Theta'}_{t-1}}{\nabla_{\Theta'}\cL'(\Theta'_t) - \nabla_{\Theta'}\cL'({\Theta'}_{t-1})}}{\lnorm{\nabla_{\Theta'}\cL'(\Theta'_t) - \nabla_{\Theta'}\cL'({\Theta'}_{t-1})}{2}^2}, &\text{when $t$ is even}
  \end{cases}\,,
\end{align}
where $\nabla_{\Theta'}\cL'(\Theta') := \nabla_\Theta \cL(\Theta' L^{-1/2})L^{-1/2}$. We stop the descent algorithm whenever an upper bound of the KKT error is smaller than $10^{-5}$.

\subsubsection{The Role of the Topology of the Sampling Pattern}
\begin{figure}
  \centering
  \subcaptionbox{RMSE for i.i.d.  $\Theta^*_{ij}$\label{fig:graph_errVn_iid}}[.48\linewidth][c]{%
    \includegraphics[width=.5\linewidth]{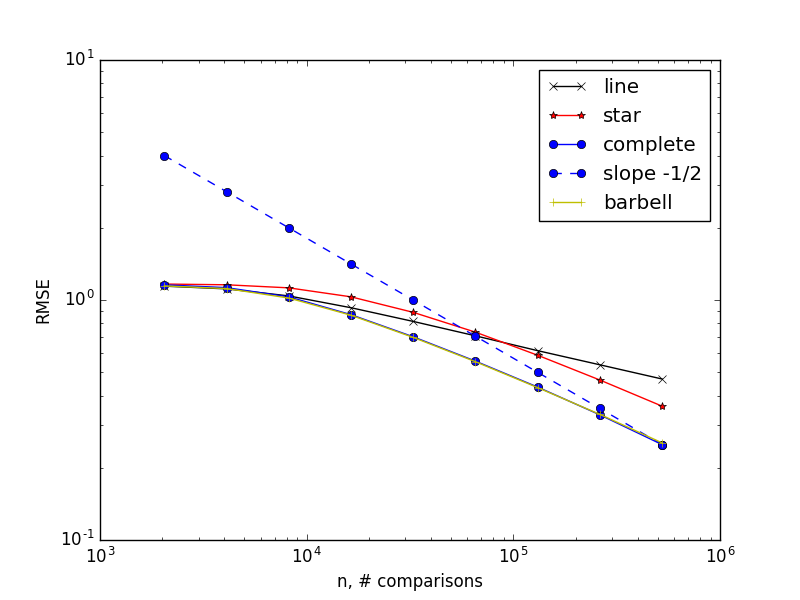}}\quad
  \subcaptionbox{L-RMSE for i.i.d.  $\Theta^*_{ij}$\label{fig:graph_errLVn_iid}}[.48\linewidth][c]{%
    \includegraphics[width=.5\linewidth]{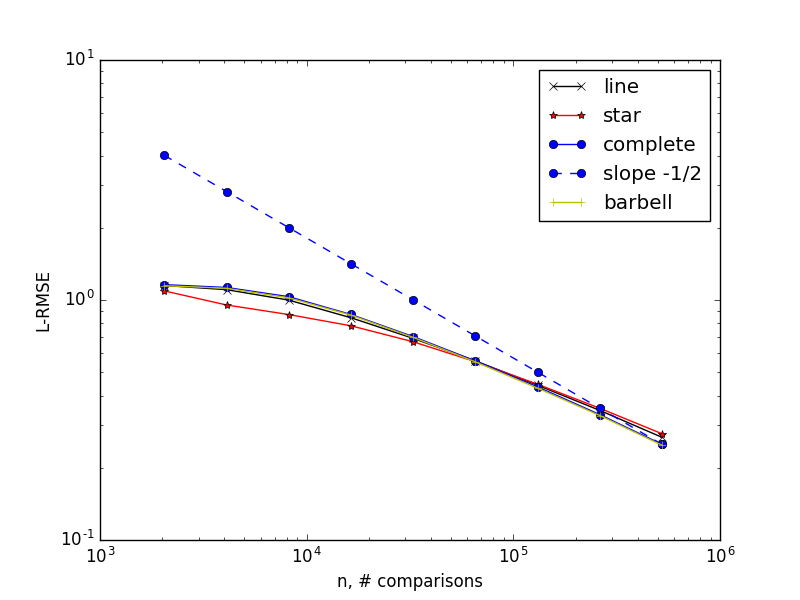}}

  \bigskip

  \subcaptionbox{RMSE for barbell bias $\Theta^*_{ij}$\label{fig:graph_errVn_barbell}}[.48\linewidth][c]{%
    \includegraphics[width=.5\linewidth]{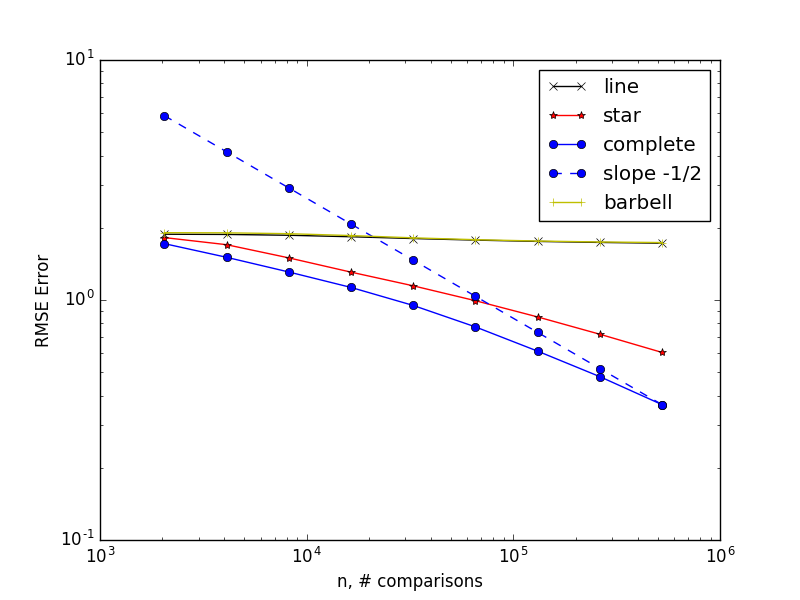}}\quad
  \subcaptionbox{L-RMSE for barbell bias $\Theta^*_{ij}$\label{fig:graph_errLVn_barbell}}[.48\linewidth][c]{%
    \includegraphics[width=.5\linewidth]{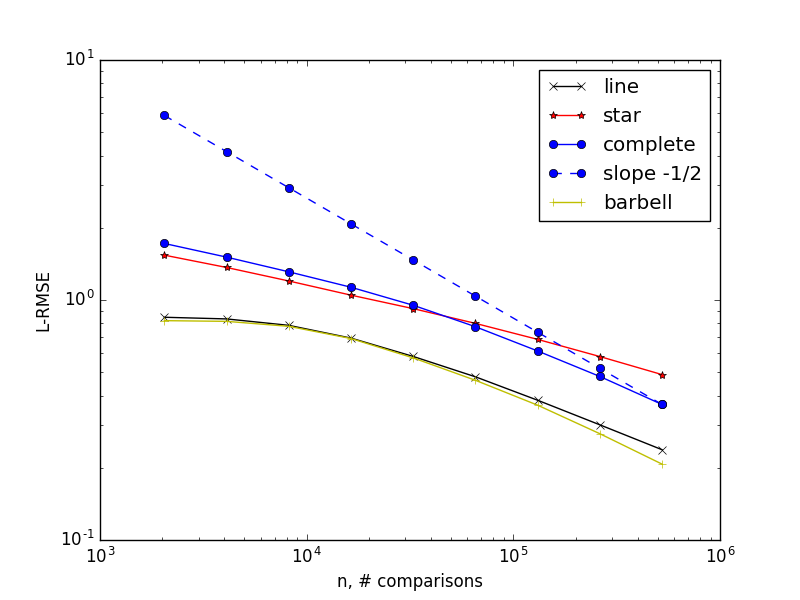}}    

  \bigskip

  \subcaptionbox{RMSE for line bias $\Theta^*_{ij}$\label{fig:graph_errVn_line}}[.48\linewidth][c]{%
    \includegraphics[width=.5\linewidth]{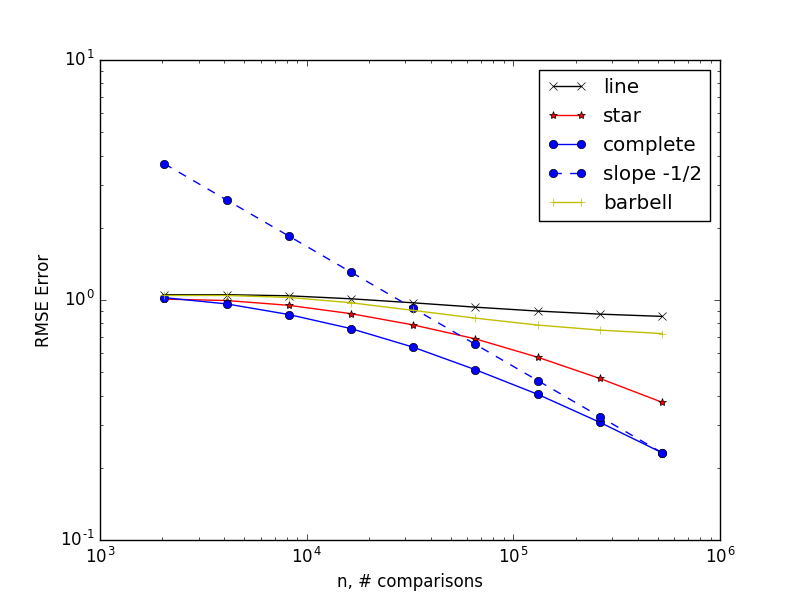}}\quad
  \subcaptionbox{L-RMSE for line bias $\Theta^*_{ij}$\label{fig:graph_errLVn_line}}[.48\linewidth][c]{%
    \includegraphics[width=.5\linewidth]{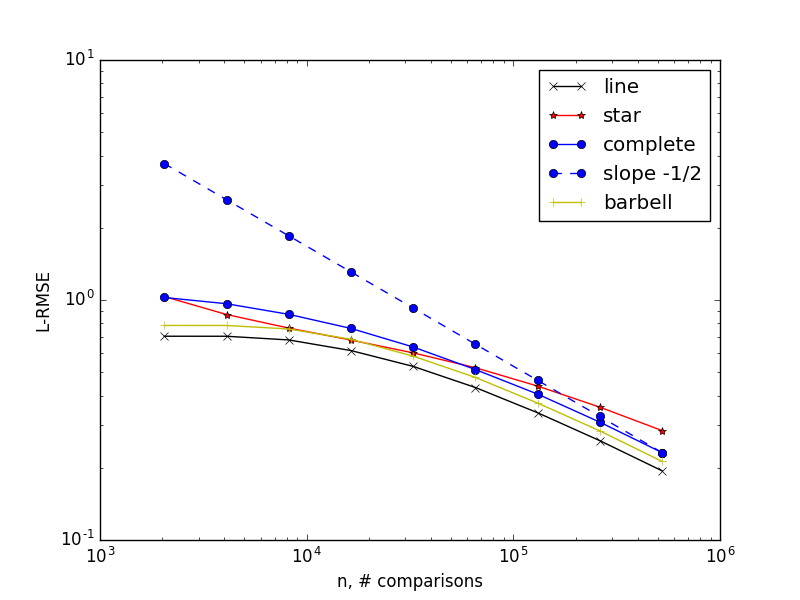}}
    
    \caption{Graphs with small spectral gap achieve significantly larger Frobenius norm error (RMSE) ${\Fnorm{\Theta - \hTheta}}/{\sqrt{d_1 d_2}}$, whereas the Laplacian-induced norm error (L-RMSE) ${\Fnorm{(\Theta - \hTheta) L^{1/2}}}/{\sqrt{d_1}}$ is not sensitive to the spectral gap.}
    \label{fig:graph_errorVn}
\end{figure}

In figure \ref{fig:graph_errorVn}, we plot the error of our nuclear norm minimization based algorithm versus number of samples (in log-scale), $n$ for $d_1=d_2=300$, $r=4$, $\bb=5.0$, $G=1$.
We consider two errors here; root mean squared error (RMSE) = ${\Fnorm{\Theta - \hTheta}}/{\sqrt{d_1 d_2}}$ 
and Laplacian induced RMSE (L-RMSE)= ${\Fnorm{(\Theta - \hTheta) L^{1/2}}}/{\sqrt{d_1}}$. We plot these errors for four topologies of varying spectral gaps.
As discussed Section \ref{sec:main_ub}, 
we do not expect the L-RMSE error to change much as we change the topology of sampling. 
However, as seen from the simple relation 
$\vert\kern-0.25ex\vert\kern-0.25ex\vert (\Theta^* - \hTheta) L^{1/2}\vert\kern-0.25ex\vert\kern-0.25ex\vert_F 
\ge \sigma_{min}^{1/2} \vert\kern-0.25ex\vert\kern-0.25ex\vert\Theta^* - \hTheta\vert\kern-0.25ex\vert\kern-0.25ex\vert_F $
Frobenius norm error is more sensitive to the topology of the sampling pattern, captured via the spectral gap, i.e.~$\sigma_{\rm min}(L)$. Specifically we use the following graph topologies.
\begin{itemize}
  \item {\bf Complete graph.} We first consider  a uniform sampling over a complete graph where 
  $P_{j_1,j_2}=1/d_2(d_2-1)$ for all $j_1,j_2\in[d_2]$. The resulting spectral gap is $1/(d_2-1)$, which is the maximum possible value. Hence, complete graphs are optimal for learning MNL models, compared in the error metric of the Frobenius norm for fairness. 

  \item {\bf Star graph.} Here we choose one item to be the center, and every other items can only be compared to this center item uniformly at random. Let item 1 be the center one, then $P_{j_1,1}=P_{1,j_2}=1/2(d_2-1)$. Standard spectral analysis shows that 
  the spectral gap is $\Theta(1/d_2)$, and thus the graph is near-optimal for learning MNL models. 
    
  \item {\bf Line graph.} Next, we consider a line graph with $d_2-1$ edges where 
  $P_{j,j+1}=P_{j+1,j}=1/2(d_2-1)$. It has a spectral gap of $\Theta(1/d_2^2)$, and is strictly sub-optimal for learning MNL models.
  
  \item {\bf Barbell graph.} Consider two equal sized groups of items. 
  Within each group the sub-graph is complete, and between the groups there is a single edge connecting one of the node from group one and one of the node from group two. 
  Each edge is chosen uniformly at random for comparisons. The resulting spectral gap is $\Theta(1/d_2^2)$, and this graph too is strictly sub-optimal for learning MNL models. 
  
\end{itemize}
\noindent
First in sub-figures \ref{fig:graph_errVn_iid}, \ref{fig:graph_errLVn_iid}, we plot RMSE and L-RMSE errors for different graphs using randomly generated $\Theta^*_{ij}$. We see that L-RMSE curves for different graphs are the same (and slopes in log-scale are as expected approaches $-1/2$ with more samples). 
Further, we do not see any significant difference w.r.t the graph topology even when 
error is measured in Frobenius norm.
The reason is that since $\Theta^*_{ij}$'s are generated i.i.d., the empirical distributions of any large sub-group of items would be similar. Thus, the means of the two cliques of the barbell graph or the means of the items on the two far ends of the line graph are similar. Thus although barbell and line graphs have small spectral gap (high mixing-time), its effect is minimized because these sub-groups can individually be solved without having them to mix since the empirical distributions of the $\Theta^*_{ij}$ in the two sub-groups are similar.

To illustrate the role of the topology of the graph, 
we choose specific $\Theta^*$ which depends on the topology of the graph as guided by our analysis on the lower bound (Theorem \ref{thm:graph_lb}) 
in sub-figures \ref{fig:graph_errVn_barbell}, \ref{fig:graph_errLVn_barbell}.  
The items are divided into two sets (corresponding to each side of barbell graph), such that corresponding $\Theta^*_{ij}$ are i.i.d. inside a set but have similar but shifted means across the sets. We call this type of preference data as \textit{barbell biased}. As expected from theoretical analyses, $L$-RMSE behave similar to the i.i.d. case. 
However, we see the Frobenius norm error significantly worse in the case of line and barbell shaped graphs, as expected from the Frobenius error bound. 
In sub-figures \ref{fig:graph_errVn_line}, \ref{fig:graph_errLVn_line}, 
we simulate  \textit{line biased} preference data $\Theta^*$. 
Items are ordered (in the order of the line graph), such that $\Theta^*_{ij}$'s have similar distributions but their means get shifted in an arithmetic progression as you go down the ordering. Again, Frobenius norm error is significantly larger for line and barbell graphs as spectral gaps are small. 

\subsubsection{The Gain in Inference over Multiple Groups of Items}
\begin{figure}
  \begin{center}
  \includegraphics[width=0.7\textwidth]{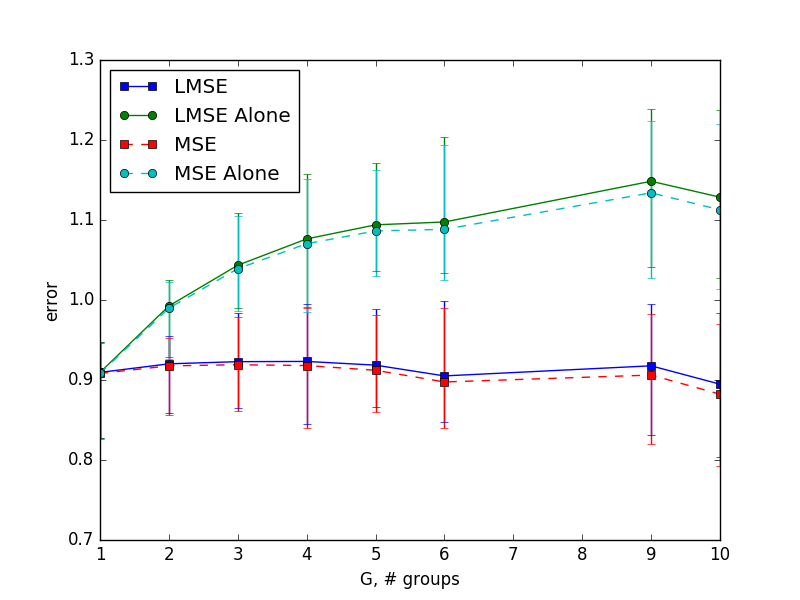}
  \hspace{0.4cm}
  \end{center}
  \caption{As the number of groups increase, the gain in joint inference increases.}
  \label{fig:graph_errVG}
\end{figure}

Consider $G$ groups of items such that, within each group, every pair of items is uniformly likely to get compared, but items from different group are never compared with each each other. 
As a baseline, one can run inference on each group separately. 
On the other hand, we propose running inference on all the $G$ groups jointly. 
Let $\hTheta$ be the estimate of $\Theta^*$ when solving the groups together, and 
let $\bar{\Theta}$ be the estimate when groups are estimated separately. 
Let $L$, $L^{(k)}$ be the graph Laplacians of the whole graph and $k$-th connected component (group) respectively. 
suppose, for simplicity, that 
$d_1=d_2$ and the groups are equally sized complete sub-graph components,
\begin{align}
 L &= \frac1{(d_2-G)}\left(\id_{d_2 \times d_2} - \frac{G}{d_2} \sum_{k=1}^G g_k g_k^T \right) \, \text{, and,}\\
 L^{(k)} &= \frac1{(d_2/G-1)}\left(\id_{\frac{d_2}{G} \times \frac{d_2}{G}} - \frac{G}{d_2} \ones \ones^T\right)\,. \label{eq:L_alone}
\end{align}
\\ According to Theorems \ref{thm:graph_ub} and \ref{thm:graph_lb}, the L-RMSE error of $\hTheta$ satisfies,
\begin{align}
  \frac{1}{\sqrt{d_1}}{\Fnorm{\left(\Theta^* - \hTheta\right) L^{1/2}}} = \frac{1}{\sqrt{d_1 (d_2-G)}}{\Fnorm{\Theta^* - \hTheta}} = \tilde{O} \left( \sqrt{\frac{r d_1}{n}}\right)\,.
\end{align} 
Similarly L-RMSE error (with respect to the full Laplacian $L$) of $\bar\Theta$ satisfies,
\begin{align}
  \frac{1}{d_1}{\Fnorm{\left(\Theta^* -  \bar\Theta \right) L^{1/2}}}^2 &= \frac{1}{d_1 (d_2-G)}{\Fnorm{\Theta^* - \bar\Theta}}^2 \\
  &\overset{(a)}{=} \frac{(d_2/G-1)}{(d_2-G)} \sum_{k=1}^G \frac{\Fnorm{\left(\Theta^*_{k} - \bar\Theta_k \right) (L^{(k)})^{1/2}}^2}{d_2} \\
  &= \frac{1}{G} \sum_{k=1}^G \tilde{O} \left(\frac{r d_1}{n/G} \right) \\
  &=  \tilde{O} \left(\frac{G r d_1}{n} \right)\,,
\end{align} 
where $(a)$ follows from Eq. \eqref{eq:L_alone} and assuming $\Theta^*_k,\bar\Theta_k$ are sub-matrices restricted to the columns in group $k$.
Thus the estimation errors when running joint inference and separate inference for each group are of the order of $O_G(1)$ and $O_G(\sqrt{G})$ respectively. That is, a user's preference in one group of items will be useful in inferring the same user's preference in another group of items. We illustrate this gain of joint inference in Figure \ref{fig:graph_errVG}.
Concretely, 
the sampling graph $\cG$ has $G$ groups where each component is a complete graph and $d_1=d_2=360$, $r=4$, $\bb=5.0$, $n=2^{14}$. Figure \ref{fig:graph_errVG} plots the L-RMSE (RMSE) $= {\Fnorm{(\Theta - \hTheta) L^{1/2}}}/{\sqrt{d_1}}\; ({\Fnorm{(\Theta - \hTheta)}}/{\sqrt{d_1\,d_2}})$ errors vs.~$G$, when all the groups are solved together (labelled as LMSE and MSE) or when the groups are solved separately (labelled as LMSE Alone and MSE Alone) using our algorithm. We see that solving the components together keeps the error relatively similar as the number of groups increase, but if we solve the groups separately the error increases with number of groups, although it is at a lower rate than predicted by the upper bound.

\subsubsection{Robustness to the Mismatched $L$-nuc Norm Regularizer}
\label{sec:robust}

In practical scenarios, one might not have access to the sampling graph Laplacian $L$. 
We propose using empirical Laplacian $\hat L$ defined as 
\begin{eqnarray}
	\hat L \;\; \triangleq \;\;  {\rm diag}\big( \hat P \ones \big) - \hat P\;,
\end{eqnarray}
where $\hat P \in \reals^{d_2 \times d_2}$ is the empirical distribution of sampled pairs in the given data. 
Under the experimental setting from Figure~\ref{fig:graph_errLVn_iid}, 
we run additional experiments with this empirical Laplacian $\hat L$ in the optimization: 
minimize $-\cL(\Theta) + \lambda \hLnucnorm{\Theta}$\;.
Figure~\ref{fig:Laplacian} illustrates that 
the effect on the performance of not knowing the true $L$ is marginal. 
Both approaches achieve the same error. 

\begin{figure}[t]
  	\centering
   	\includegraphics[width=.8\linewidth]{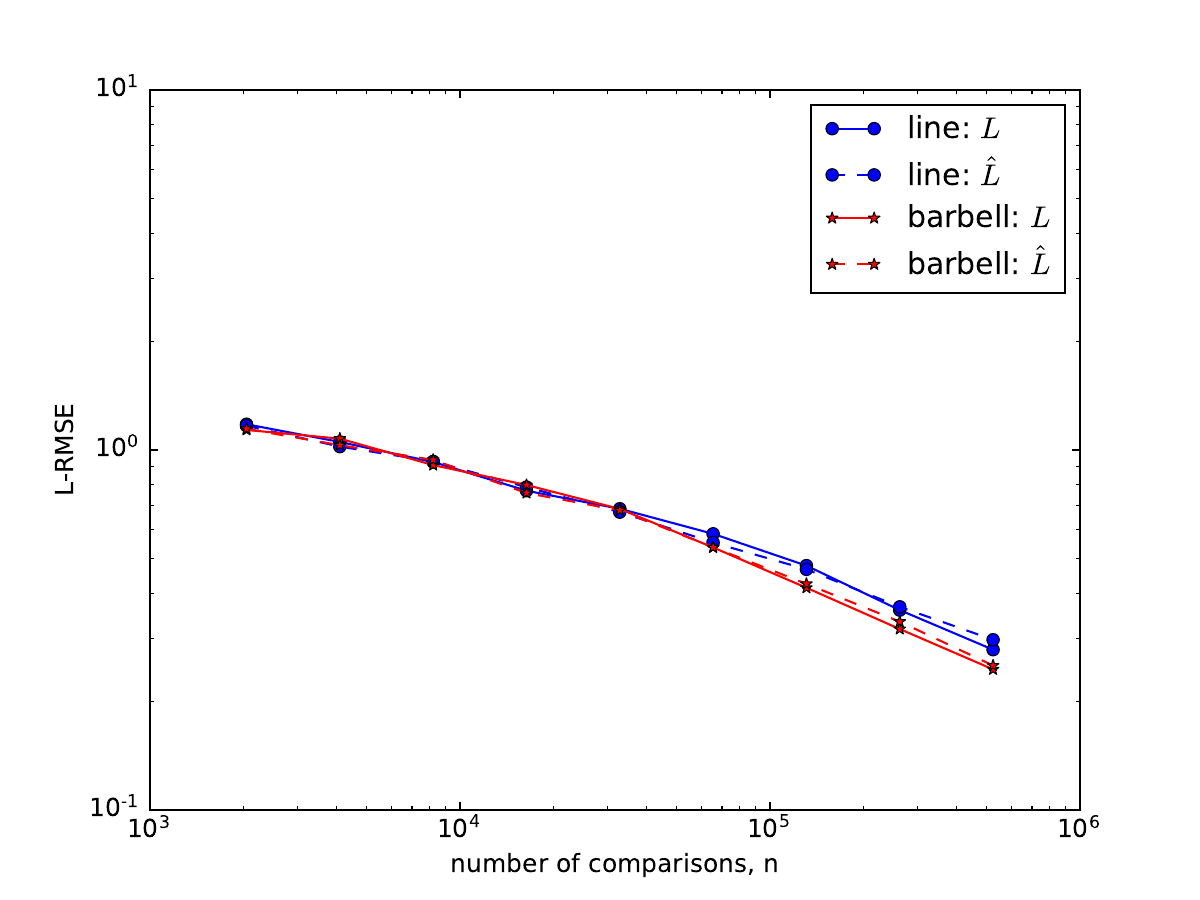}
  	\caption{L-RMSE for various sampling graphs when true Laplacian $L$ is known (solid) and when empirical Laplacian $\hat{L}$ is used (dashed)}. 
	\label{fig:Laplacian}
\end{figure}
\subsubsection{Real data: Food100}
\label{sec:food}
To showcase the practicality of our nuclear norm based algorithm \eqref{eq:graph_obj} we apply our algorithm to the Food100 Data set\footnote{Data set is from \url{https://vision.cornell.edu/se3/projects/cost-effective-hits}.} \cite{wilber2014cost}. In the data set, $n=250320$ triplets, denoted by $\{(a_i, b_i, c_i)\}_i$, of $3$ distinct food dishes from a selection of $d=100$  were sampled. Then in a crowdsourcing setting, users were asked if, $a_i$ is more similar to $b_i$ than to $c_i$. The goal is to learn an low-dimensional embedding of the $100$ food items where the above similarities are captured. We model the problem as learning an MNL model, parameterized by $\Theta^*$, which gives the following probability distribution for $i$-th user's answer,
\begin{align*}
\prob{a_i \text{ is more similar to } b_i \text{ than to } c_i} \;\; = \;\; \frac{e^{\Theta^*_{a_i, b_i}}}{e^{\Theta^*_{a_i, b_i}} + e^{\Theta^*_{a_i, c_i}} }\;.
\end{align*}
This is  the same model  as the pairwise comparisons from Section \ref{sec:model}, 
except for the fact that instead of a user (row) comparing two items (columns), 
here we compare a food item (row) to two other food items (columns). 
We implement three different algorithms: our nuclear norm based algorithm (`nucnorm'), 
unregualrized ($\lambda=0$) likelihood maximization (`fullrank') and 
maximum likelihood based algorithm to learn rank-$1$ Plackett-Luce model \cite{Luce59,Pla75} (`plackett').

In Figure~\ref{fig:food_likelihood}, we plot the mean log-likelihood of the learned model versus fraction $x$ of the data used for training for the various algorithms for testing (a) and training (b) data. If $x$ fraction of the data is used for training, we use the rest ($1-x$) of the data for testing. For the nuclear norm minimization, we estimate the Laplacian $L$ using the empirical distribution of the triplets and $\lambda$ is chosen to be $0.1\, \sqrt{\log(d)/2\,d\,xn}$.

In the Fig.~\ref{fig:food_likelihood}(b) (to the left) on the testing data set, we see that our MNL model based nuclear norm regularized algorithm clearly outperforms both unregularized algorithm and the Placket-Luce model estimator, especially when there is less training data. In fact, the mean likelihood 
($\log (P_{\rm model}(\text{test data}))$) on the testing data remains relatively the same when we decrease the size of the training data, which supports our claim that real data has low-rank structure. 
In the Fig.~\ref{fig:food_likelihood}(b) (right) on the training data sets, 
the non-regularized approach of `fullrank' achieves higher likelihood on the training data, 
indicating that it overfits to training data. 

\begin{figure}[h]
  \centering
  \subcaptionbox{Mean log-likelihood on testing data set}[.48\linewidth][c]{
    \includegraphics[width=.5\linewidth]{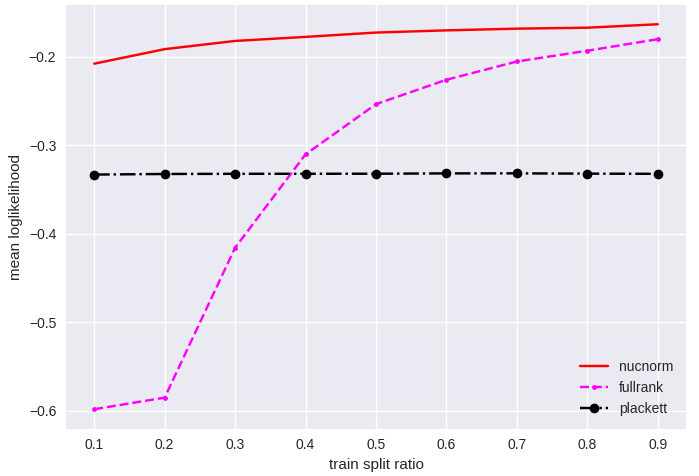}}\quad
  \subcaptionbox{Mean likelihood on training data set}[.48\linewidth][c]{%
    \includegraphics[width=.5\linewidth]{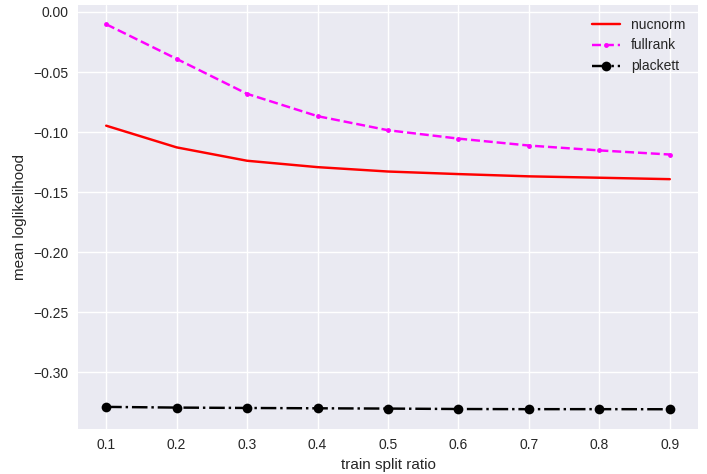}}
    
    \caption{Mean log-likelihood vs fraction of the total data used for training. Our nuclear norm regularized algorithm (`nucnorm') fits the test data better than both unregularized algorithm (`fullrank') and Plackett-Luce model based estimation, especially when training data is small in size.}
    \label{fig:food_likelihood}
\end{figure}

\begin{figure}[h]
  	\centering
   	\includegraphics[width=.8\linewidth]{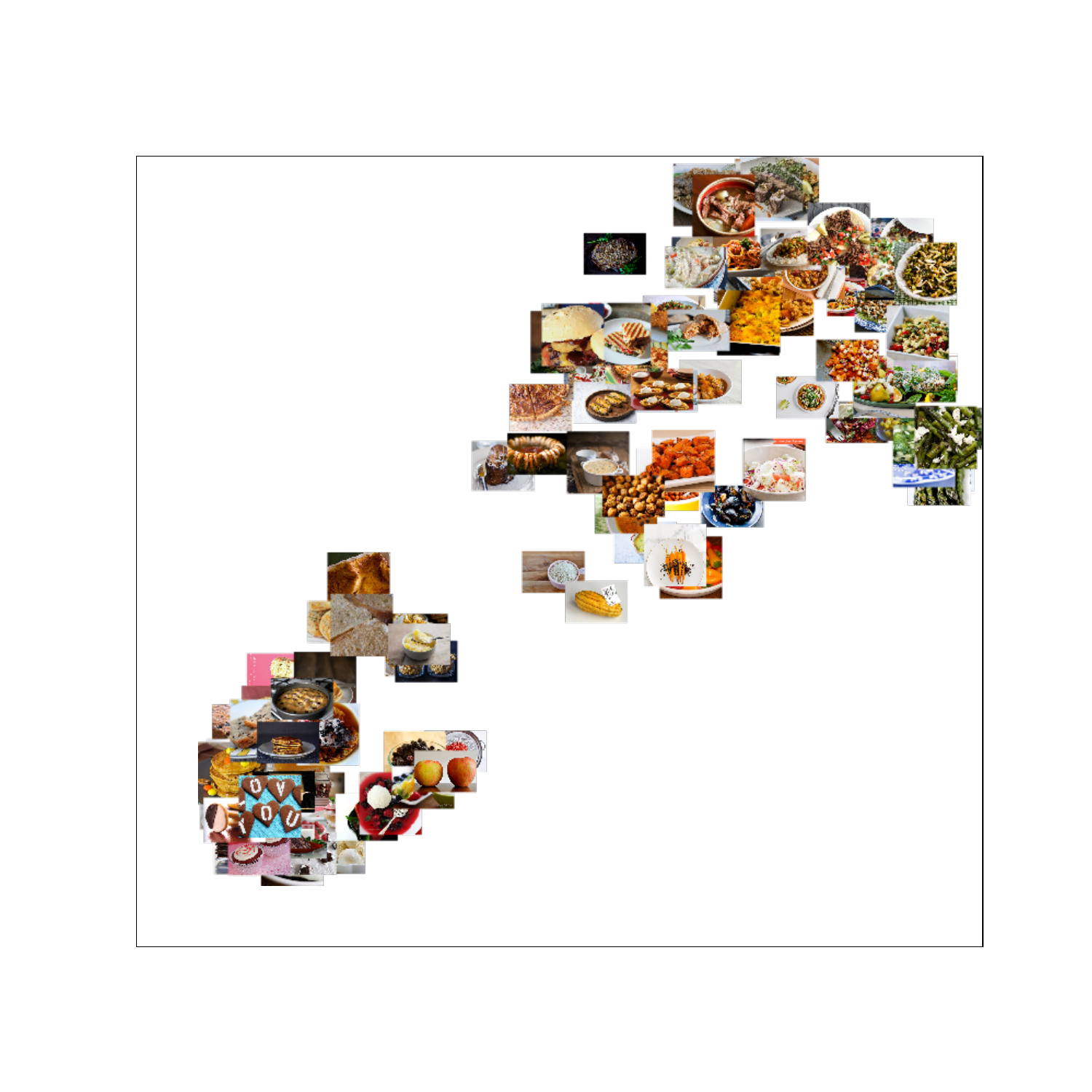}
   	\put(-310,157){Dessert}
   	\put(-235,230){Meat}
   	\put(-68,185){Salad}
    \caption{Food100: t-SNE embedding of the columns of the learned MNL parameter $\hTheta$. Desserts (bottom left) are seperated from other dishes. Meat dishes are also separated from vegetable dishes.}
    \label{fig:food_embed}
\end{figure}

In Fig.~\ref{fig:food_embed} we plot the t-SNE embedding \cite{maaten2008visualizing} of the columns of the estimated MNL parameter matrix $\hTheta$ when all the data is used for training. 
The desserts (left bottom) are separated from other dishes, 
and meat dishes and salad dishes form two clusters (top right). 


\section{Learning the MNL Model under Higher Order Comparisons} 
\label{sec:kwise}

Higher order comparisons, where a subset of $k$ items are 
offered to a user who then provides a 
complete ranking (total linear ordering) of those item, is a natural generalization 
of pairwise comparisons, that captures some aspect of heterogeneous and complex modern data sets.  
We refer to such scenarios as $k$-wise comparisons or $k$-wise rankings. 
The MNL model generalizes to such comparisons. 
Let $\Theta^*$ be the $d_1\times d_2$ dimensional matrix 
capturing the preference of $d_1$ users on $d_2$ items, 
where the rows and columns correspond to users and items, respectively. 
In this $k$-wise ranking setting, when a user $i$ is 
presented with a set, $S_i\subseteq [d_2]$, of $k$ alternatives she reveals her preferences 
as  a  ranked list over those items. 
To simplify the notations, we assume that all the users compare the same number $k$ of items, 
but the analysis naturally generalizes to the case when the size might differ from  a user to a user and when each user provides more than one $k$-wise ranking.
Let $v_{i,\ell} \in S_i$ denote the (random) $\ell$-th best choice of user $i$. Each user gives a ranking,  independent of other users' rankings, from 
\begin{align}
   \prob{v_{i,1},\ldots,v_{i,k}|\text{$S_i$ is presented to user $i$}} & = \, \prod_{\ell = 1}^k \frac{e^{\Theta^*_{i,v_{i,\ell}}}}{\sum_{j\in S_{i,\ell}} e^{\Theta^*_{i,j}} } 
  \;, \label{eq:defkwiseMNL}
\end{align}
where 
with $S_{i,\ell} \equiv S_i \setminus \{ v_{i,1}, \ldots, v_{i,\ell-1} \}$ and $S_{i,1} \equiv S_i$. 
For a user $i$, the $i$-th row of $\Theta^*$ 
represents the underlying preference vector of 
the user, and the more preferred items are more likely to be ranked higher.

Similar to the pairwise comparisons, the  distribution \eqref{eq:defkwiseMNL} is independent of shifting each row of $\Theta^*$ by a constant. Since we can only estimate $\Theta^*$ up to this equivalent class, we search for the one 
whose rows sum to zero, i.e. $\sum_{j\in[d_2]} \Theta^*_{i,j}=0$ for all $i\in[d_1]$.
For capturing the ``spikiness'' \cite{NW11} of $\Theta^*$, we define $\alpha \equiv \max_{i, j_1,j_2}  |\Theta_{ij_1}^*- \Theta^*_{ij_2}| $ to denote the dynamic range of 
the underlying $\Theta^*$, such that when $k$ items are compared, we always have 
\begin{eqnarray}
 \frac1k e^{-\bb} \;\leq\; \frac1{1 + (k-1) e^{\bb}} \;\leq\; \prob{v_{i,1} = j} \;\leq\; \frac1{1 + (k-1) e^{-\bb}} \;\leq\; \frac1k  e^\bb \;,
 \end{eqnarray}
for all $j\in S_i$, all $S_i\subseteq [d_2]$ satisfying $|S_i|=k$ and all $i\in[d_1]$.  
We do not make any assumptions on $\bb$ other than that 
$\bb=O(1)$ with respect to $d_1$ and $d_2$. 
Given this definition, we solve the following optimization 
\begin{align}
\label{eq:kwiseopt}
\widehat{\Theta} \,\in\, \arg \min_{\Theta \in \Omega_\bb} \,-\cL(\Theta) \,+\, \lambda \nucnorm{\Theta},
\end{align}
where,
\begin{align}
  \cL(\Theta)\,=\, \frac{1}{k\, d_1} \sum_{i=1}^{d_1}  \sum_{\ell=1}^{k} \left( \llangle \Theta, e_{i} e_{v_{i,\ell} }^\top    \rrangle - \log \left( \sum_{j \in S_{i,\ell} } \exp\left( \llangle \Theta, e_{i} e_j^\top \rrangle \right) \right) \right) \;,
  \label{eq:defkwiseL}
\end{align}
over 
 \begin{eqnarray}
  \Omega _\bb = \Big\{ A \in \reals^{d_1 \times d_2} \,\big|\, \lnorm{A}{\infty} \le \bb, \text{ and }\forall i\in[d_1]\text{ we have }\sum_{j\in[d_2]} A_{ij}=0 \Big\} \;.\label{eq:defkwiseomega}
\end{eqnarray}
Note that unlike graph sampling for pairwise comparisons, 
we assume that each user is presented a subset of $k$ items and provides 
a complete ranking over those $k$ items. 
This choice of sampling scenario, together with independent choices of the items in subset $S_i$'s, is crucial for getting a bound that is tight in its scaling with respect to not only $d_1$, $d_2$, and $r$, but also $k$, 
as  a certain independence is required to apply the symmetrization technique (in Lemma \ref{lmm:kwise_hessian2}) which gives us the desired tight bound on the error.  
It trivially follows from our analysis that 
  one can relax the assumptions in the sampling scenario significantly (e.g.~sampling without replacement, heterogeneous sampling probabilities for each item-user pair, etc.), 
  and the only change in the upper bound of Eq.~\eqref{eq:kwiselambda} will be a weaker dependence $k$. 

\subsection{Performance Guarantee} 

We provide an upper bound on the resulting error of our convex relaxation, 
when a {\em multi-set} of items $S_i$ presented to user $i$ is drawn uniformly at random 
with replacement. 
Precisely, for a given $k$, $S_i = \{ j_{i,1},\ldots, j_{i,k}\}$ where 
$j_{i,\ell}$'s are independently drawn uniformly at random over the $d_2$ items. 
Further, if an item is sampled more than once, i.e. if  there exists 
$j_{i,\ell_1}=j_{i,\ell_2}$ for some $i$ and $\ell_1\neq \ell_2$, then we assume that the user 
treats these two items as if they are two distinct items with the same MNL weights 
$\Theta_{i,j_{i,\ell_1}}^*=\Theta_{i,j_{i,\ell_2}}^*$. The resulting preference is therefore always over $k$ items 
(with possibly multiple copies of  the same item), and distributed according to \eqref{eq:defkwiseMNL}. 
For example, if $k=3$, it is possible to have $S_i=\{j_{i,1}=1,j_{i,2}=1,j_{i,3}=2\}$, in which case the resulting 
ranking can be $ (v_{i,1} = j_{i,1},v_{i,2}=j_{i,3}, v_{i,3}=j_{i,2} )$ with probability $ (e^{\Theta^*_{i,1}})/(2\,e^{\Theta^*_{i,1}}+e^{\Theta^*_{i,2}}) \times (e^{\Theta^*_{i,2}})/(e^{\Theta^*_{i,1}}+e^{\Theta^*_{i,2}}) $. 
Such a sampling with replacement is necessary for the analysis, where 
we require independence in the choice of the items in $S_i$ 
in order to apply the symmetrization technique (e.g. \cite{BLN13}) to bound the expectation of the deviation (cf. Appendix  \ref{sec:kwise_hessian3_proof}).  Similar sampling assumptions have been made in existing analyses on learning 
low-rank models from noisy observations, e.g. \cite{NW11}. 
Let $d \equiv (d_1+d_2)/2$, and  let $\sigma_j(\Theta^*)$ denote the $j$-th singular value of the matrix $\Theta^*$. Define  
\begin{eqnarray}
  \lambda_0 &\equiv& e^{2\bb} \sqrt{\frac{d_1\log d + d_2\,(\log d)^2 (\log 2d)^4}{k\,d_1^2 \,d_2}} \;.  
  \label{eq:kwiselambda}
\end{eqnarray} 

\begin{theorem}
\label{thm:kwise_ub}
Under the described sampling model, assume $24 \, \leq k \leq \min\{ d_1^2 \log d, (d_1^2 + d_2^2)/(2d_1) \log d ,\\ \;(1/e)\,d_2(4\log d_2 + 2 \log d_1)\}$, and
$\lambda \in [480 \lambda_0, c_0 \lambda_0]$
with  any constant $c_0=O(1)$ larger than 480. Then, solving the optimization \eqref{eq:kwiseopt} achieves 
\begin{eqnarray}
    \frac{1}{d_1d_2} \fnorm{\hTheta-\Theta^\ast}^2  \; \leq \; 288\sqrt{2}  \, e^{4\bb}  c_0 \lambda_0 \sqrt{r} \,  \fnorm{ \hTheta-\Theta^\ast } + 288 e^{4\bb} c_0 \lambda_0 \sum_{j=r+1}^{\min\{d_1,d_2\}} \sigma_j(\Theta^*) \;, 
  \label{eq:kwise_ub}
\end{eqnarray}
for any $r\in \{1,\ldots, \min\{d_1,d_2\} \}$ 
with probability at least $1-2d^{-3} - d_2^{-3}$ where $d=(d_1+d_2)/2$.
\end{theorem}

A proof   is provided in Appendix \ref{sec:kwise_ub_proof}. 
This bound holds for all values of $r$ and one could potentially optimize over $r$. 
We show such results in the following corollaries.

\begin{corollary}[{\bf Exact low-rank matrices}]
  Suppose $\Theta^*$ has rank at most $r$. 
  Under the hypotheses of Theorem \ref{thm:kwise_ub}, 
  solving the optimization \eqref{eq:kwiseopt} with the choice of the regularization parameter 
  $\lambda\in[480\lambda_0,c_0\lambda_0]$ achieves  with probability at least $1-2d^{-3} - d_2^{-3}$, 
  \begin{eqnarray}
  \frac{1}{\sqrt{d_1d_2}} \fnorm{\widehat{\Theta}-\Theta^\ast} \,\leq\, 
  288 \sqrt{2}   e^{6 \bb}  c_0 \sqrt{\frac{r(d_1\log d + d_2\,(\log d)^2 (\log 2d)^4) }{k\,d_1}} \;.  
  \label{eq:kwise_lowrank}
  \end{eqnarray}
  \label{cor:kwise_lowrank}
\end{corollary}
The number of entries is $d_1d_2$ and we rescale the Frobenius norm error appropriately by $1/\sqrt{d_1d_2}$. 
For a rank-$r$ matrix $\Theta^*$ with
$r(d_1+d_2)-r^2 = O(r(d_1+d_2))$ degrees of freedom, 
the above theorem shows that the total number of samples, which is
$(k \, d_1)$, needs to scale as $O(rd_1(\log d) + r d_2 \,(\log d)^2 (\log 2d)^4)$ 
in order to achieve an arbitrarily small error. This is only poly-logarithmic 
factor larger than the degrees of freedom. 
In Section \ref{sec:kwise_lb}, we provide a lower bound on the error directly, that matches the upper bound up to a logarithmic factor. 
The dependence 
on the dynamic range $\bb$ is sub-optimal. 
 The exponential dependence in the bound 
seems to be a weakness of the analysis, as seen from numerical experiments in the right panel of Figure \ref{fig:kwise}. Although the error increase with $\bb$, numerical experiments suggests that it only increases at most linearly. 
A practical issue in achieving the above rate is the choice of $\lambda$, since 
the dynamic range $\bb$ is not known in advance. 
Figure \ref{fig:kwise} illustrates that the error is not sensitive to the choice of $\lambda$ for a wide range. 

For approximately low-rank matrices in $\ell_q$-ball defined in 
\eqref{eq:defBq}, 
optimizing the choice of $r$ in Theorem \ref{thm:kwise_ub}, we get the following result. 
This is a strict generalization of Corollary \ref{cor:kwise_lowrank} and  
a proof of this Corollary is provided in Appendix \ref{sec:kwise_cor_proof}. 

\begin{corollary}[{\bf Approximately low-rank matrices}]
  Suppose $\Theta^* \in \bB_q(\rho_q)$ for some $q\in(0,1]$ and $\rho_q>0$. 
  Under the hypotheses of Theorem \ref{thm:kwise_ub}, 
  solving the optimization \eqref{eq:kwiseopt} with the choice of the regularization parameter 
  $\lambda \in [480\lambda_0,c_0\lambda_0]$ achieves 
    with probability at least $1-2d^{-3}$, 
  \begin{eqnarray}
  \frac{1}{\sqrt{d_1d_2}} \fnorm{\widehat{\Theta}-\Theta^\ast} \,\leq\, 
   \frac{2\sqrt{\rho_q}}{\sqrt{d_1d_2}}  \left(  288 \sqrt{2}  c_0 e^{6 \bb} \,\sqrt{\frac{d_1d_2(d_1\log d + d_2\,(\log d)^2 (\log 2d)^4) }{k\,d_1}} \right)^{\frac{2-q}{2}}\;. 
  \label{eq:kwise_appxlowrank}
  \end{eqnarray}
  \label{cor:kwise_appxlowrank}
\end{corollary}

\subsection{Information-theoretic Lower Bound for  Low-rank Matrices} 
\label{sec:kwise_lb} 
A simple parameter counting argument indicates that 
it requires the number of samples to scale as the degrees of freedom i.e.,   
$k d_1 \propto r(d_1+d_2)$, 
to estimate a $d_1\times d_2$ dimensional matrix of rank $r$. 
By applying Fano's inequality with appropriately chosen hypotheses, the following lower bound establishes that the bound in Theorem \ref{thm:kwise_ub} is sharp up to a logarithmic factor.

\begin{theorem}
Suppose $\Theta^*$ has rank $r$. 
Under the described sampling model, for large enough $d_1$ and $d_2\geq d_1$, 
there is a universal numerical constant $c>0$ such that 
\begin{eqnarray}
  \inf_{\hTheta} \sup_{\Theta^*\in\Omega_\bb} \E\Big[ \frac{1}{\sqrt{d_1d_2}}\fnorm{\hTheta-\Theta^*} \Big] &\geq& c\,\min\left\{ \bb e^{-\bb} \sqrt{\frac{\,r\,d_2}{k\,d_1}}  \,,\,  \frac{\bb d_2}{\sqrt{d_1d_2 \log d}}\right\} \;,
\end{eqnarray}
where the infimum is taken over all measurable functions over the observed ranked lists  
$\{(v_{i,1},\ldots,v_{i,k})\}_{i\in[d_1]}$. 
\label{thm:kwise_lb}
\end{theorem}
A proof of this theorem is provided in Appendix \ref{sec:kwise_lb_proof}. 
The term of primary interest in this bound is the first one, 
which shows the scaling of the (rescaled) minimax rate as $\sqrt{r(d_1+d_2)/(kd_1)}$ (when $d_2\geq d_1$), and matches the upper bound in \eqref{eq:kwise_ub}.
It is the dominant term in the bound whenever 
the number of samples is larger than the degrees of freedom by a logarithmic factor, i.e.,  
$kd_1 >  r(d_1+d_2) \log d$, ignoring the dependence on $\bb$.
This is a typical regime of interest, where the sample size is comparable to the latent dimension of the problem.
In this regime,  
Theorem \ref{thm:kwise_lb} establishes that the upper bound in 
Theorem \ref{thm:kwise_ub} is minimax-optimal up to a logarithmic factor in the dimension $d$. 

\subsection{Rank Breaking for Higher Order Comparisons} 
\label{sec:rb}

A common approach in practice to 
handle higher order comparisons is {\em rank breaking}, 
which refers to the practice 
of breaking the higher order comparisons into a set of 
pairwise comparisons and applying an estimator tailored for pairwise comparisons treating each pair as independent 
\cite{ACPX13,APX14}. 
When the higher order comparison is given as partial rankings (as opposed to total linear ordering as we assume) then rank breaking can be inconsistent, and special 
algorithms are needed for weighted rank breaking \cite{KO16,KO16nips}. However, when $k$-wise rankings (also called total linear  orderings)  are observed as we assume, simple and standard rank breaking achieves a similar performance as the higher order estimator in \eqref{eq:kwiseopt}. 
 Assume that $u_{i,m}$, $i \in [d_1],\;m \in [k]$, denotes the $m$-th element observed by the $i$-th user. 
 Concretely, in rank breaking,  we convert the $k$-wise ranking data into pairwise ranking data and then we solve the following optimization problem: 
\begin{align}
\calL (\Theta) = \frac{1}{d_1 {k \choose 2}} \sum_{i \in [d_1]} \sum_{(m_1, m_2) \in \calP_0} \left( \Theta_{i,\; h_i\left(m_1,m_2\right)} - \log\left( \exp\left(\Theta_{i,\; u_{i,m_1}}\right) + \exp\left(\Theta_{i,\; u_{i,m_2}}\right)\right) \right)\;,\label{eq:pairwise-LL}
\end{align}
where $\calP_0 = \{(i,j) :\; 1 \leq i < j \leq k\}$, and $h_i\left(m_1,m_2\right)$ and $l_i\left(m_1,m_2\right)$ is defined as the higher and lower ranked index among $u_{i, m_1}$ and $u_{i, m_2}$ respectively. Then modified optimization problem becomes,
\begin{align}
{\hTheta} \; \in \; \arg \min_{\Theta \in \Omega_{\alpha}} -\calL(\Theta) + \lambda \nucnorm{\Theta} \label{eq:pairwise-opt}
\end{align}
Let $d \equiv (d_1+d_2)/2$, and  let $\sigma_j(\Theta^*)$ denote the $j$-th singular value of the matrix $\Theta^*$. Define  
\begin{eqnarray}
  \label{eq:pairwise-lambda}
  \lambda_0 &\equiv&  \sqrt{\frac{d \log d}{k\,d_1^2 \,d_2}} \;.
\end{eqnarray} 
With this choice of regularization coefficient, we get the following 
upper bounds on the rank breaking estimator \eqref{eq:pairwise-opt} that are comparable to 
the upper bounds of $k$-wise ranking estimator in 
Theorem \ref{thm:kwise_ub} and Corollary \ref{cor:kwise_lowrank}. 

\begin{theorem}
\label{thm:pairwise_ub}
Under the described sampling model, assume $  2(c+4) \log d \, \leq k \leq \max\{ d_1, d_2^2/d_1\} \log d$, $d_1 \geq 4$, and
$\lambda \in [ 2\sqrt{32(c+4)}\lambda_0, c_p \lambda_0]$
with  any constant $c=O(1)$ larger than $2\sqrt{32(c+4)}$. Then, solving the optimization \eqref{eq:pairwise-opt} achieves 
\begin{eqnarray}
    \frac{1}{d_1d_2} \fnorm{\hTheta-\Theta^\ast}^2  \; \leq \; 144\sqrt{2}  \, e^{2\bb}  c \lambda \sqrt{r} \,  \fnorm{ \hTheta-\Theta^\ast } + 144 e^{2\bb} c \lambda \sum_{j=r+1}^{\min\{d_1,d_2\}} \sigma_j(\Theta^*) \;, 
  \label{eq:pairwise_ub}
\end{eqnarray}
for any $r\in \{1,\ldots, \min\{d_1,d_2\} \}$ 
with probability at least $1-2d^{-c} - 2d^{-2^{13}}$ where $d=(d_1+d_2)/2$.
\end{theorem}

A proof   of this theorem is provided in Appendix \ref{sec:pairwise-theorem}, and the following corollary follows for rank-$r$ matrices. 

\begin{corollary}[{\bf Exact low-rank matrices}]
  Suppose $\Theta^*$ has rank at most $r$. 
  Under the hypotheses of Theorem \ref{thm:pairwise_ub}, there exists a constant $c_1>0$ such that 
  solving the optimization \eqref{eq:pairwise-opt} with the choice of the regularization parameter 
  $\lambda \in [ 2\sqrt{32(c+4)}\lambda_0, c \lambda_0]$ achieves   with probability at least $1-2d^{-c} - 2d^{-2^{13}}$, 
  \begin{eqnarray}
  \frac{1}{\sqrt{d_1d_2}} \fnorm{\widehat{\Theta}-\Theta^\ast} \,\leq\, 
  144 \sqrt{2}   e^{2 \bb}  c_1 \sqrt{\frac{rd\log d}{k\,d_1}} \;.  
  \label{eq:pairwwise_lowrank}
  \end{eqnarray}
  \label{cor:pairwwise_lowrank}
\end{corollary}

\subsection{Experiments}

We provide results from numerical experiments on both synthetic and real data sets. 
\subsubsection{Algorithm} \label{sec:kwise_algo}
Similar to the case of pairwise comparisons in Section \ref{sec:graph_algo}, we use proximal gradient descent \cite{ANW10,cai2010singular} along with modified Barzilai-Borwein (BB) step-size selection rule \cite{barzilai1988two} with the initial point $\Theta_0 = 0$. Each iteration of the algorithm applies the following two operations on the current estimate, $\Theta_t$, of $\Theta^*$,
\begin{align}
\widetilde{\Theta}_{t+1} &= \Theta_t - \eta_{t} \nabla_\Theta \cL(\Theta_t) &&\text{(gradient descent)}\\
{\Theta}_{t+1} &= M_t (\Gamma_t - \eta_t \lambda \id )^+ N_t^T &&\text{(singular value shrinkage and thresholding)}
\end{align}
where $M_t \Gamma_t  N_t^T := \widetilde{\Theta}_t$ is the singular value decomposition of $\widetilde{\Theta}_t$, such that $\Gamma_t$ is a diagonal matrix with positive entries, $(\cdot)^+$ is the entry-wise thresholding operation $\max(0, x)$, and $\eta_t$ is an BB step-size calculated as,
\begin{align}
\eta_t =\begin{cases}
    \lnorm{\Theta_t - {\Theta}_{t-1}}{2}^2/\Iprod{\Theta_t - {\Theta}_{t-1}}{\nabla_{\Theta}\cL(\Theta_t) - \nabla_{\Theta}\cL({\Theta}_{t-1})}, &\text{when $t$ is odd}  \\
    \Iprod{{\Theta}_t - {\Theta}_{t-1}}{\nabla_{\Theta}\cL(\Theta_t) - \nabla_{\Theta}\cL({\Theta}_{t-1})}/\lnorm{\nabla_{\Theta}\cL(\Theta_t) - \nabla_{\Theta}\cL({\Theta}_{t-1})}{2}^2, &\text{when $t$ is even}
  \end{cases}\,.
\end{align}

\subsubsection{Simulation: Higher Order Comparisons}
\label{sec:kwise_expt}
The left  panel of Figure \ref{fig:kwise} confirms the scaling of the error rate as predicted by Corollary \ref{cor:kwise_lowrank}. 
The lines merge to a single line when the sample size is rescaled appropriately (inset). 
We make a choice of $\lambda= \sqrt{ (\log d)/(kd^2)}$.
This choice is independent of $\bb$ and is smaller than proposed in Theorem \ref{thm:kwise_ub}. 
We generate the random rank-$r$ true MNL parameters matrices of dimension $d\times d$ using the process mentioned in Section \ref{sec:graph_algo}.
The root mean squared error (RMSE) is plotted where 
${\rm RMSE} = (1/\sqrt{d_1\,d_2}) \triplenorm{ \Theta^*-\hTheta}_{\rm F} $.   
We implement and solve the convex optimization \eqref{eq:kwiseopt} 
using proximal gradient descent method as analyzed in \cite{ANW10}. 
The right panel in Figure \ref{fig:kwise} illustrates that the actual  
error is insensitive to the choice of $\lambda$ for a broad range of 
$\lambda\in[\sqrt{ (\log d)/(kd^2)},2^{8}\sqrt{ (\log d)/(kd^2)}]$, after which it increases with $\lambda$.
\begin{figure}[h]
  \begin{center}
  \includegraphics[width=.45\textwidth]{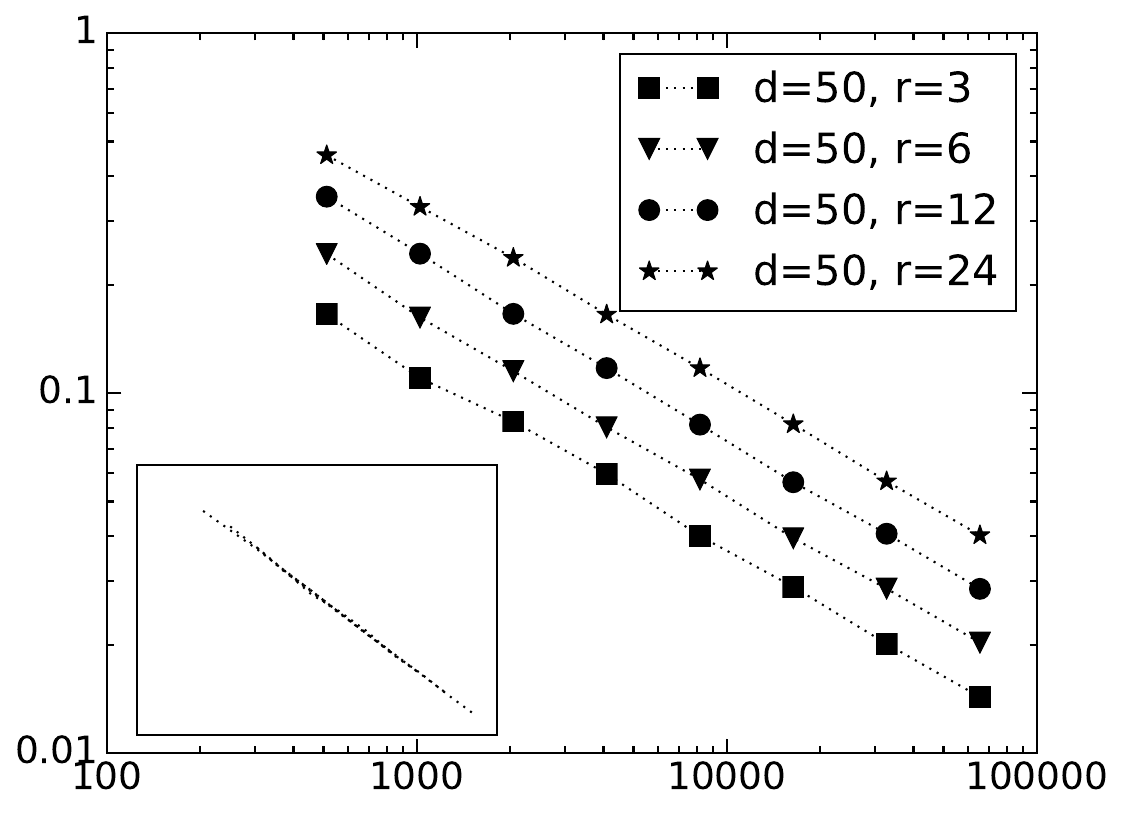}
  \put(-215,160){\small RMSE}
  \put(-135,-10){\small sample size $k$}
  \hspace{0.4cm}
  \includegraphics[width=.49\textwidth]{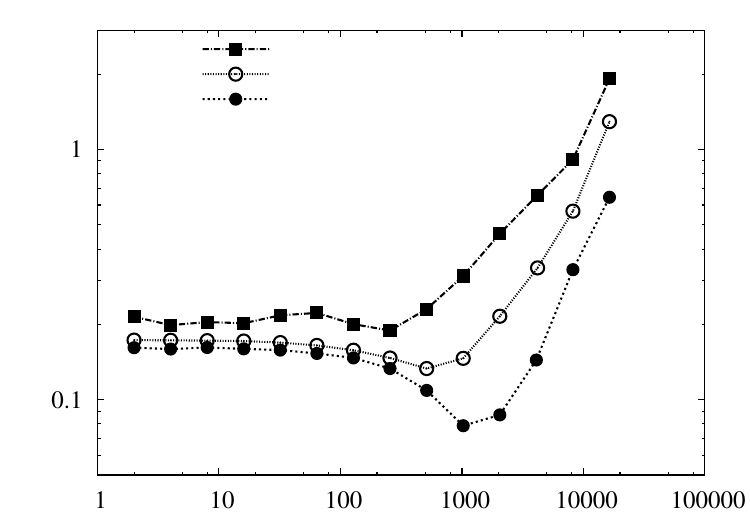}
  \put(-225,160){\small RMSE}
  \put(-125,-5){\small $\frac{\lambda}{\sqrt{(\log d) / (k d^2)}}$}
  \put(-195,144){\tiny $\bb=15$}
  \put(-195,137){\tiny $\bb=10$}
  \put(-195,130){\tiny $\bb=5$}
  \end{center}
  \caption[Error vs. number of samples and lambda for $k$-wise rankings]{ The (rescaled) RMSE scales as $\sqrt{r(\log d)/k}$ as expected from Corollary  \ref{cor:kwise_lowrank} for fixed $d=50$ (left).  In the inset, the same data is plotted versus rescaled sample size $k/(r\log d)$. 
      The (rescaled) RMSE is stable for a broad range of $\lambda$ and $\bb$ for fixed $d=50$ and $r=3$ 
      (right).  }
  \label{fig:kwise}
\end{figure}
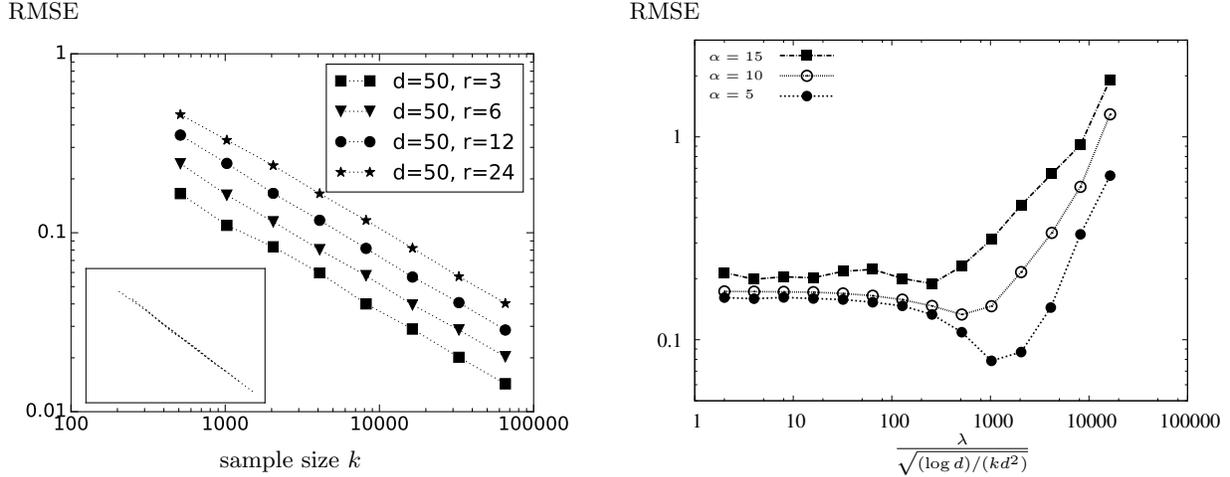

\subsubsection{Simulation: Rank Breaking}
In this section we compare the higher order $k$-wise comparison algorithm \eqref{eq:kwiseopt} (`kwise') with the pairwise rank breaking algorithm \eqref{eq:pairwise-opt} (`kbreak'). We use the same setting as in Section \ref{sec:kwise_expt}, where we observe samples from $k$-wise ranking from an underlying true MNL model and the aim is to recover the true parameter $\Theta^*$ of the model. We use $\lambda= 0.45\,\sqrt{ (\log d)/(kd^2)}$ and $\lambda= 0.1\,\sqrt{ (\log d)/(kd^2)}$ for $k$-wise and pairwise rank breaking algorithms respectively. In Fig.~\ref{fig:rankbreak} we plot the RMSE for both the algorithm for $d=50$ and $r=3,12$. We note that the even though the RMSE decreases in the rate as predicted by the theorem, we see that pairwise rank breaking is worser than the higher order $k$-wise algorithm which directly uses the $k$-wise rankings. This is consistent with the experimental observation made previously in \cite{HOX14}. Further we note that rank breaking is much slower than the other algorithm, since gradient computation of the former takes $O(k^2)$ time whereas for the latter it can be computed in $O(k)$ time.
\begin{figure}[h!]
  \begin{center}
  \includegraphics[width=.52\textwidth]{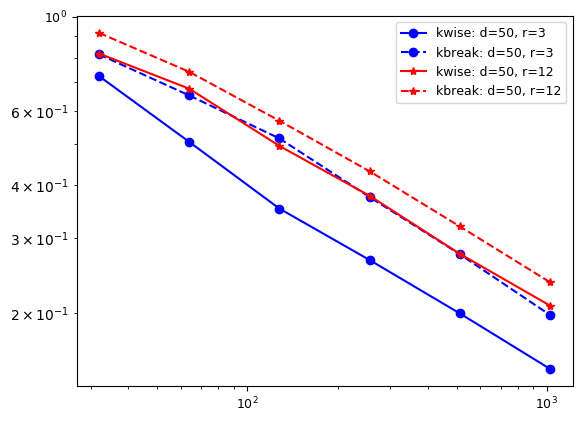}
  \put(-260,165){\small RMSE}
  \put(-135,-3){\small sample size $k$}
  \end{center}
  \caption{Rank Breaking: RMSE error versus number of samples per user $k$. $k$-wise (`kwise') algorithm performs better than the rank breaking (`kbreak') approach.}
  \label{fig:rankbreak}
\end{figure}

\subsubsection{Real data: Jester}
Jester data set\footnote{Data set is from \url{http://eigentaste.berkeley.edu/dataset/}.} \cite{goldberg2001eigentaste} has $24,982$ users, each rating a subset of $100$ jokes on continuous scale of $[-10, 10]$. As the scale is continuous, we derive  ordinal data from the scores (ties broken uniformly at random). We use only the $7200$ users who rated all the jokes for our experiments. For each user, $k=100x$ jokes were randomly selected uniformly at random for training, rest of the $100-k=100(1-x)$ jokes where used for testing, where $x$ is the fraction of jokes selected for training. We implment four algorithms: nuclear norm minimization (`nucnorm') \eqref{eq:kwiseopt}, unregularized ($\lambda = 0$) log-likelihood maximization (`fullrank'), rank-$1$ Plackett-Luce model estimation (`plackett'), and rank breaking algorithm (`rankbreak') \eqref{eq:pairwise-opt}. We use $\lambda= 0.7\,\sqrt{ (0.5 \log (d_1 d_2))/(k d_1 \sqrt{d_1 d_2})}$ and $\lambda= 0.16\,\sqrt{ (0.5 \log (d_1 d_2))/(k d_1 \sqrt{d_1 d_2})}$ for $k$-wise and pairwise rank breaking algorithms respectively. In Fig.~\ref{fig:jester} (a) we plot the multiplicative bias in the mean log-likelihood on the testing data versus the fraction $x$ of training data used. 
For each model in \{`nucnorm', `fullrank', `plackett', `rankbreak'\}, 
we plot in the y-axis 
 $$
 \frac{\log (P_{\rm model}(\text{test data})) - \log (P_{\rm fullrank}(\text{test data})) } {\abs{\log(P_{\rm fullrank}(\text{test data}))}} \;,
 $$
using fullrank model as a baseline as it has the least test likelihood.
Plackett-Luce model achieves the best performance when sample size is small, 
as this simplest model avoids overfitting. 
However, for most regimes of sample size, both the nuclear norm minimization and rank breaking 
achieve similar performance improving upon the others. 

The same trend holds when we measure the perfomrance in 
the normalized Spearman's footule distance \cite{diaconis1977spearman} $F(\pi_1, \pi_2) \in [0, 1]$ between two rank-lists $\pi_1$, $\pi_2$ of length $k$: 
\begin{align*}
F(\pi_1, \pi_2) = \frac{2}{k^2} \sum_{i=1}^k \abs{\pi_1(i) - \pi_2(i)}
\end{align*}
In Fig.~\ref{fig:jester} (b) we plot the average normalized Spearman's footrule distance between the ground truths and  the most likely ranking on the testing data under the estimated model parameters. 
We see that $k$-wise nuclear norm minimization and rank breaking algorithms perform the best in recovering the true ranking, except when the fraction of training data used is very small so that the rank-1 Plackett-Luce recovers better ranking. 

\begin{figure}[h!]
  \centering
  \label{fig:jester_likelihood}
  \subcaptionbox{Multiplicative bias in the mean log-likelihood}[.48\linewidth][c]{
    \includegraphics[width=.47\linewidth]{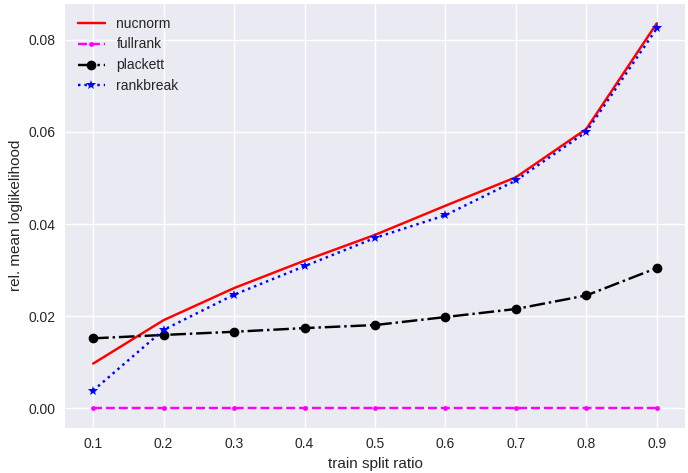}}\quad
  \label{fig:jester_footrule}
  \subcaptionbox{Spearman's footrule distance}[.48\linewidth][c]{%
    \includegraphics[width=.47\linewidth]{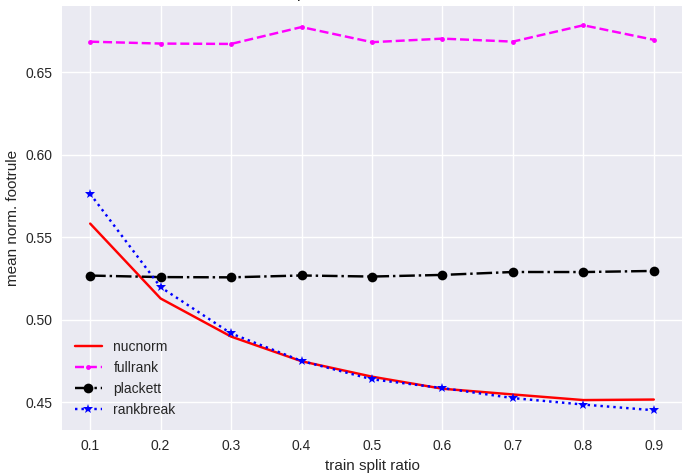}}
    
    \caption{Jester data set: Performance on test data vs.~fraction of the total data used for training. 
    The proposed nuclear norm regularized algorithm (`nucnorm') and rank breaking (`rankbreak')
    improves upon both the unregularized algorithm (`fullrank') and the Plackett-Luce model estimation (`plackett') for most regimes of the sample size.}
    \label{fig:jester}
\end{figure}

\subsubsection{Real data: Irish Election} 
\label{sec:irish}

The Irish Election data set\footnote{Data set is from \url{https://projecteuclid.org/euclid.aoas/1231424218\#supplemental}.} is an opinion poll conducted among 1083 participants during the 1997 Irish presidential election campaign \cite{GM09}. Each participant responded with a ranking the of their top 1, 2, 3, 4, or 5 choices from the 5 candidates: Banotti, McAleese, Nally, Roche, and Scallon. For our experiments we use only the 807 participants who gave their top-5 choices, i.e.~full-rankings of all the candidates. Next we divide these participants into 60 (2x3x5x2) group according to a Cartesian product of four categorizations: sex (male/female), marital status (single/married/widowed+divorced), social class (F/AB/C1/C2/DE)\footnote{Social classes are F: farmer, AB: middle class. C1: lower middle class, C2: skilled working class, and DE: other working class.}, location (rural/city+town). We assume that within each group the responses of its member follow the same distriubtion and these distributions of all all the groups are captured by an MNL model with parameter $\Theta^* \in \reals^{60 \times 5}$. We implement three algorithms: nuclear norm minimization (`nucnorm') \eqref{eq:kwiseopt}, unregularized ($\lambda = 0$) log-likelihood maximization (`fullrank'), and rank-$1$ Plackett-Luce model estimation (`plackett'). We use $\lambda= 0.8\,\sqrt{ (0.5 \log (d_1 d_2))/(k d_1 \sqrt{d_1 d_2})}$. If $x$ randomly sampled fraction of the data is used for training, then rest of the data is used for testing.
In Fig.~\ref{fig:irish_likelihood} we plot the mean log-likelihood 
($\log (P_{\rm model}(\text{test data}))$) on the testing data versus the fraction of training data used. We see that nuclear norm minimization and Plackett-Luce model estimation tie for the first place and both improves significantly upon the un-regularized full-rank MNL model estimation. 
In Fig.~\ref{fig:irish_embed} we plot the t-SNE \cite{maaten2008visualizing} embedding of the rows of the estimated parameter matrix $\hTheta$ when all the data is used for training. In Fig.~\ref{fig:irish_embed_marital} the markers represent the marital status of the group: single/married/divorced+widowed. 
In Fig.~\ref{fig:irish_embed_social} the markers represent the social class of the groups. 
We see that married (left) and divorced+widowed (right) groups are clearly separated in the embedding, indicating that marital status influences the preference of candidates. 
However, we see that the social classes are less influential. 

\begin{figure}[h!]
\begin{minipage}{.48\textwidth}
  \begin{center}
  \includegraphics[width=1\linewidth]{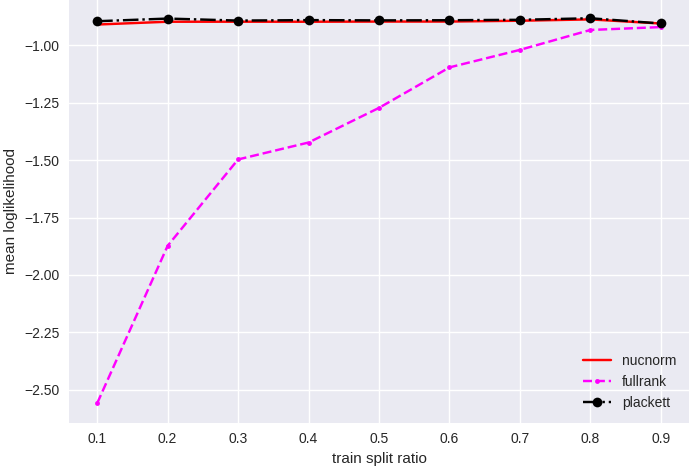}
  \end{center}
  \caption{Irish Election: Mean log-likelihood on the test data versus fraction of the data used for training. Nuclear norm minimization and Plackett-Luce model estimator tie for the best performance.}
  \label{fig:irish_likelihood}
\end{minipage}%
\hfill
\begin{minipage}{.48\textwidth}
  \centering
  \includegraphics[width=1\linewidth]{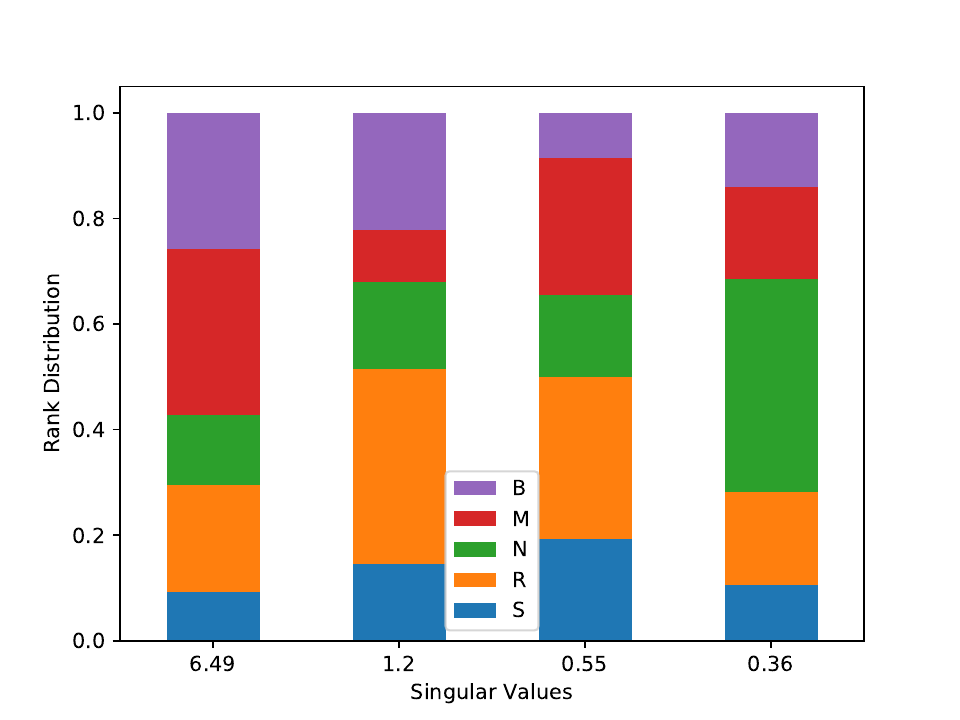}
  \captionof{figure}{Irish Election: Each bar corresponds to rank distribution of one of the singular 4 values (x-axis) of the $\hTheta$. Heights of the partitions represent the probability with which the distribution ranks the corresponding candidate as first (Section \ref{sec:irish}).}
  \label{fig:irish_blocs}
\end{minipage}
\end{figure}
In Fig.~\ref{fig:irish_blocs} we represent the voting characteristics of the top $4$ right singular vectors $\{\hat{v}_j\}_{j=1}^4$ of $\hTheta$, which has a rank of $4$ when all the data is used for training. The each stacked bar corresponds to the singular value $\sigma_j$ marked on the x-axis. Partition of a bar represents the choice model distribution of the corresponding singular vector: $\exp(v_j)/(\ones^T \exp(v_j))$, where $\exp(v_j)$ is the element-wise exponentiation operator. We see that there are $2$ majors voting ``basis" distributions; one favoring McAleese and another favoring Roche. Similar voting blocs have been observed earlier \cite{GM09}. Even though Placket-Luce model estimator achieves the similar likelihood as our nuclear norm regularized algorithm, the latter helps us in identifying voter ``basis'', solely from rank data without using the side-information on the voters as in \cite{GM09}.
\begin{figure}[h!]
\centering
\begin{subfigure}{0.5\textwidth}
  \centering
  \includegraphics[width=1\linewidth]{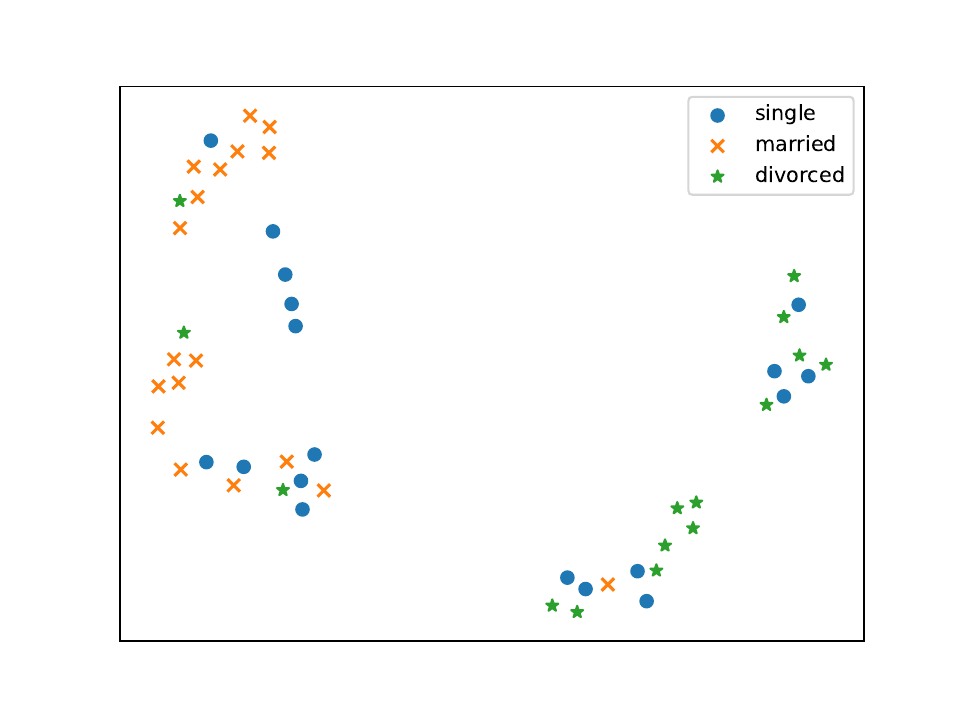}
  \caption{Martial status}
  \label{fig:irish_embed_marital}
\end{subfigure}%
\begin{subfigure}{.5\textwidth}
  \centering
  \includegraphics[width=1\linewidth]{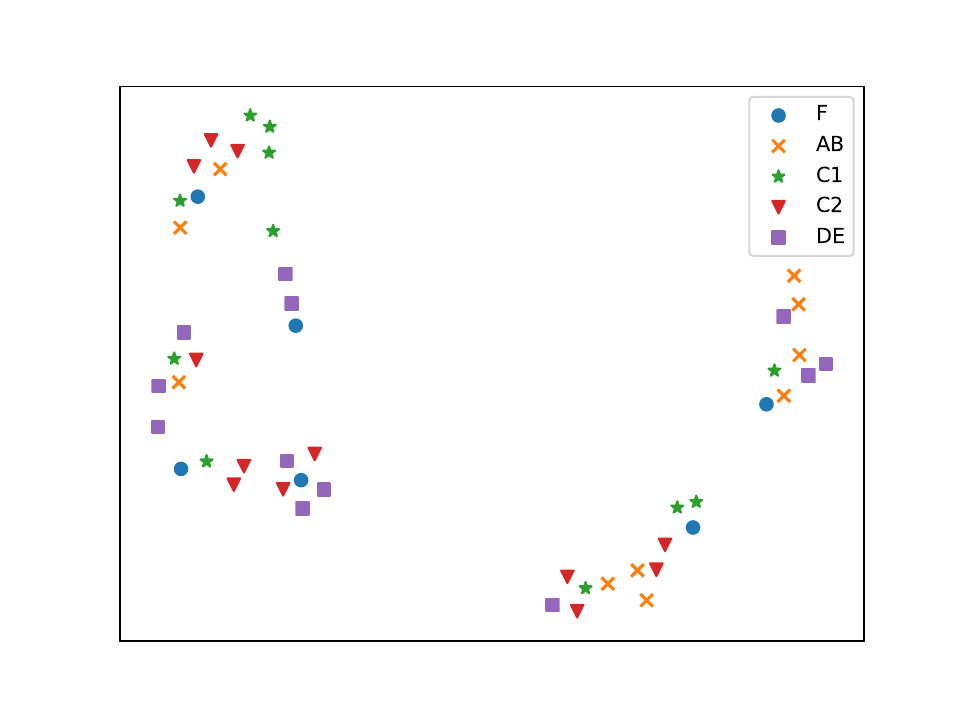}
  \caption{Social class}
  \label{fig:irish_embed_social}
\end{subfigure}
\caption{Irish Election: t-SNE embedding of rows of the estimated parameter matrix $\hTheta$. The markers correspond to the marital or social status (F: farmer, AB: middle class. C1: lower middle class, C2: skilled working class, DE: other working class) of the rows.}
\label{fig:irish_embed}
\end{figure}

\section{Learning the MNL Model from Choices} 
\label{sec:bundle}

Choice modeling  has had widespread success in numerous 
application domains such as transportation and marketing \cite{Tra86,GL83}. 
Choice models stem from 
revenue management 
to tackle the fundamental problem of 
maximizing expected revenue where the expectation is taken over 
a probabilistic choice model that is learned from historical purchase data.
Revenue management has focused on designing 
 efficient solvers for the optimization problem with exact or approximation guarantees, 
and has less to do with 
{\em learning} the parameters of probabilistic choice model of interest.

In this section, we tackle this unexplored  domain of learning choice models from samples with 
provable guarantees on the sample complexity.  
In particular, we study learning the MNL model from choices. 
We study two types of choices under the MNL model  
that together include all practical scenarios of interest:  
{\em bundled choice} and 
{\em consumer choice}. 

\bigskip\noindent
{\bf Bundled choice.} 
We consider a novel scenario of significant practical interest: choice modeling from bundled purchase history. 
In this setting, we assume that we have bundled purchase history data from $n$ users.  
Precisely, there are two categories of interest with $d_1$ and $d_2$ alternatives in each category respectively. For example, there are $d_1$ tooth pastes to choose from and $d_2$ tooth brushes to choose from. 
For the $i$-th user, a subset $S_i \subseteq [d_1]$ of alternatives from the first category is presented along with a subset $T_i\subseteq[d_2]$ of alternatives from the second category. 
We use $k_1$ and $k_2$ to denote the number of alternatives presented to a single user, i.e. $k_1=|S_i|$ and $k_2=|T_i|$, and we assume that the number of alternatives presented to each user is fixed, to simplify notations. However, the analysis naturally generalizes if the number differs from a user to another user. 
Given these sets of alternatives, each user makes a `bundled' purchase, {of an item from $S_i$ and another item from $T_i$ together,}
and we use $(u_i,v_i)$ to denote these bundled pair of alternatives (e.g. a tooth brush and a tooth paste) purchased by the $i$-th user. 
Each user makes a choice of the best alternative, independent of other users's choices, according to the MNL model as 
\begin{eqnarray}
  \prob{(u_i,v_i)=(j_1,j_2)} &=& \frac{e^{\Theta^*_{j_1,j_2}}}{ \sum_{j'_1\in S_i,j'_2\in T_i} e^{\Theta^*_{j_1',j_2'}}}\;,
  \label{eq:defbundleMNL}
\end{eqnarray}
for all $j_1\in  S_i$ and $j_2 \in T_i$. 
We emphasize here that the preference matrix is indexed by 
items of type one (in the rows) and items of type two (in the columns). 
We are taking the existing standard MNL model over user-item pairs to 
propose a novel choice model for bundled purchases over two types of items. 
One could go beyond paired bundled choices and include the user identity as another dimension, 
or add other types of items and consider higher order bundled purchases. 
This would require MNL model over higher order tensors, which is outside the scope of this paper, but are interesting generalizations. The main challenge in learning such tensor MNL models is that 
nuclear norm of a higher order tensor is not a computable quantity and hence 
minimizing the nuclear norm is not algorithmically feasible \cite{YC14}. 
Efficient methods exist based on alternating minimizations, 
but existing analysis tools can handle only quadratic losses  \cite{JO14completion}.

The distribution \eqref{eq:defbundleMNL} is independent of shifting all the values of $\Theta^*$ by a constant. Hence, there is an equivalent class of $\Theta^*$ that gives the same distribution for the choices:
  $ [\Theta^*] \equiv \{A \in \reals^{d_1\times d_2}\,|\, A = \Theta^* + c \ones \ones^T \text{ for some }c\in\reals  \}\;.$ 
Since we can only estimate $\Theta^*$ up to this equivalent class, we search for the one 
that sums to zero, i.e. $\sum_{j_1\in[d_1],j_2\in[d_2]} \Theta^*_{j_1,j_2}=0$. 
Let $\bb=\max_{j_1,j_1' \in [d_1],j_2,j_2'\in[d_2]} |\Theta^*_{j_1,j_2}-\Theta^*_{j_1',j_2'}|$, denote the dynamic range of the underlying $\Theta^*$, such that when $k_1\times k_2$ alternatives are presented, we always have
\begin{eqnarray}
  \frac{1}{k_1k_2}e^{-\bb} \;\leq\; \prob{(u_i,v_i)=(j_1,j_2)} \;\leq\; \frac{1}{k_1k_2}e^{\bb}\;,
\end{eqnarray}
for all $(j_1,j_2)\in S_i\times T_i$ and for all $S_i \subseteq [d_1]$ and $T_i \subseteq [d_2]$ such that $|S_i|=k_1$ and $|T_i|=k_2$. We do not make any assumptions on $\bb$ other than that $\bb=O(1)$ with respect to $d_1$ and $d_2$. 
 Assuming $\Theta^*$ is well approximated by a low-rank matrix, we solve the following convex relaxation, given the observed bundled purchase history $\{(u_i,v_i,S_i,T_i)\}_{i\in[n]}$:  
 \begin{eqnarray}
  \hTheta &\in& \arg\min_{\Theta\in\Omega_\bb} \calL(\Theta) + \lambda \nucnorm{\Theta} \;,
  \label{eq:bundleopt}
 \end{eqnarray}
 where the negative log likelihood function according to \eqref{eq:defbundleMNL} is 
\begin{eqnarray}
  \cL(\Theta) &=& -\frac{1}{n} \sum_{i=1}^{n} \left( \llangle \Theta, e_{u_i} e_{v_i }^\top    \rrangle - \log \left( \sum_{j_1 \in S_i,j_2\in T_i } \exp\left( \llangle \Theta, e_{j_1} e_{j_2}^\top \rrangle \right) \right) \right),\text{ and }\\
  \label{eq:defbundleL}
  \Omega_{\bb} &\equiv & \Big\{ A \in \reals^{d_1 \times d_2}   \,\big|\, \lnorm{A}{\infty} \leq \bb \text{, and }\sum_{j_1\in[d_1],j_2\in[d_2]} A_{j_1,j_2}=0 \Big\} \;. 
  \label{eq:defbundleomega}
\end{eqnarray}

Compared to collaborative ranking, 
$(a)$ rows and columns of $\Theta^*$ correspond to an alternative from the first and second category, respectively; 
$(b)$ each sample corresponds to the purchase choice of a user which follow the MNL model with $\Theta^*$; 
$(c)$ each person is presented subsets $S_i$ and $T_i$ of items from each category; 
$(d)$ each sampled data represents the most preferred bundled pair of alternatives.

\bigskip\noindent
{\bf Customer choice.}
The standard customer choice can be thought of as 
either a special case of {\em bundled choice} or as a special case of {\em higher order comparisons}. 
We consider the standard customer choice data from purchase history. 
In this setting, we assume that we have purchase history data from $d_1$ users over $d_2$ 
 alternatives. 
The $i$-th  sample is i.i.d.~with user $u_i$ chosen uniformly at random and 
  a subset $S_i \subseteq [d_2]$ of alternatives of size $k$. 
  We fix $k$ in order to be efficient in the notations and any variable size offerings can be handled seamlessly. 
  We assume $S_i$ is chosen uniformly at random with replacement, in a similar way as 
  bundled choice and higher order comparisons.  
  
  Given these sets of alternatives, the user $u_i$ makes a `choice' 
and we use $v_i$ to denote the purchased alternative
 by the $i$-th (sampled) user. 
Each user makes a choice of the best alternative, independent of other users's choices, according to the MNL model as 
\begin{eqnarray}
  \prob{v_i=j_2|u_i=j_1} &=& \frac{e^{\Theta^*_{j_1,j_2}}}{ \sum_{j'_2\in S_i} e^{\Theta^*_{j_1,j_2'}}}\;,
  \label{eq:defcustomerMNL}
\end{eqnarray}
for all $j_2\in  S_i$. 
Up to the fact that we index rows by users and not items of one category, 
this is a special case of the {\em bundled choice} model where we fix $k_1=1$. 
Mathematically, all of our results  under consumer choices are derived as corollaries from 
our results under bundled choices, but given the prevalent interest in customer choice models, 
we emphasize the implications of our framework under customer choice models in a separate section (see Section \ref{sec:customer}).

\subsection{Learning the MNL Model from Bundled Choices }
We provide an upper bound on the error achieved by our convex relaxation, 
when the {\em multi-set} of alternatives $S_i$ from the first category  and $T_i$ from the second category 
are drawn uniformly at random with replacement from $[d_1]$ and $[d_2]$ respectively. Precisely, for given $k_1$ and $k_2$, we let $S_i = \{ j_{1,1}^{(i)},\ldots,j_{1,k_1}^{(i)} \}$ 
and $T_i=\{j_{2,1}^{(i)},\ldots, j_{2,k_2}^{(i)}\}$, where 
$j^{(i)}_{1,\ell}$'s and $j^{(i)}_{2,\ell}$'s are independently drawn uniformly at random over 
the $d_1$ and $d_2$ alternatives, respectively. Similar to the previous section, this 
sampling with replacement is necessary for the analysis. 
Define  
\begin{eqnarray}
  \lambda_0=   \sqrt{\frac{e^{2\bb} \max\{d_1,d_2\} \log d}{n\,d_1\,d_2}} \;. 
  \label{eq:bundle_deflambda}
\end{eqnarray} 
\begin{theorem}
\label{thm:bundle_ub}
Under the described sampling model, assume $16e^{2\bb} \min\{d_1,d_2\} \log d  \, \leq n$ and \\ $n \leq \min\{ d^5 ,k_1k_2 \max\{d_1^2,d_2^2\}   \}\log d $, and
$\lambda \in [8 \lambda_0, c_1 \lambda_0]$
with  any constant $c_1=O(1)$ larger than \\ $\max\{8,128/\sqrt{\min\{k_1,k_2\}}\}$. Then, solving the optimization \eqref{eq:bundleopt} achieves 
\begin{eqnarray}
    \frac{1}{d_1d_2} \fnorm{\hTheta-\Theta^\ast}^2  \; \leq \; 
    48 \sqrt{2}  \, e^{2\bb}  c_1 \lambda \sqrt{r} \,  \fnorm{ \hTheta-\Theta^\ast } + 48 e^{2\bb}   c_1 \lambda \sum_{j=r+1}^{\min\{d_1,d_2\}} \sigma_j(\Theta^*) \;, 
  \label{eq:bundle_ub}
\end{eqnarray}
for any $r\in\{1,\ldots, \min\{d_1,d_2\}\}$ 
with probability at least $1-2d^{-3}$ where $d=(d_1+d_2)/2$.
\end{theorem}
A proof  is provided in Appendix \ref{sec:bundle_ub_proof}. 
Optimizing over $r$ gives the following corollaries. 
\begin{corollary}[{\bf Exact low-rank matrices}]
  Suppose $\Theta^*$ has rank at most $r$. 
  Under the hypotheses of Theorem \ref{thm:bundle_ub}, 
  solving the optimization \eqref{eq:bundleopt} with the choice of the regularization parameter 
  $\lambda\in[8\lambda_0 , c_1\lambda_0 ]$ achieves   with probability at least $1-2d^{-3}$, 
  \begin{eqnarray}
  \frac{1}{\sqrt{d_1d_2}} \fnorm{\widehat{\Theta}-\Theta^\ast} \,\leq\, 
  48 \sqrt{2}   e^{3 \bb}  c_1 \sqrt{\frac{r(d_1 + d_2) \log d }{n }} \;.  
  \label{eq:bundle_lowrank}
  \end{eqnarray}
  \label{cor:bundle_lowrank}
\end{corollary}
\noindent
This corollary shows that the number of samples $n$ needs to scale as $O(r(d_1+d_2)\log d)$ in order to achieve an arbitrarily small error. This is only a logarithmic factor larger than the number of degrees of freedom. 
For approximately low-rank matrices in an $\ell_1$-ball 
as defined in \eqref{eq:defBq}, we show 
an upper bound on the error, whose error exponent reduces from one to $(2-q)/2$.

\begin{corollary}[{\bf Approximately low-rank matrices}]
  Suppose $\Theta^* \in \bB_q(\rho_q)$ for some $q\in(0,1]$ and $\rho_q>0$. 
  Under the hypotheses of Theorem \ref{thm:bundle_ub}, 
  solving the optimization \eqref{eq:bundleopt} with the choice of the regularization parameter 
  $\lambda \in [8\lambda_0 , c_1\lambda_0]$ achieves 
    with probability at least $1-2d^{-3}$, 
  \begin{eqnarray}
  \frac{1}{\sqrt{d_1d_2}} \fnorm{\widehat{\Theta}-\Theta^\ast} \,\leq\, 
   \frac{2\sqrt{\rho_q}}{\sqrt{d_1d_2}}  \left(  48 \sqrt{2}  c_1 e^{3 \bb} \,\sqrt{\frac{ d_1d_2(d_1+d_2)\log d }{n }} \right)^{\frac{2-q}{2}}\;. 
  \label{eq:bundle_appxlowrank}
  \end{eqnarray}
  \label{cor:bundle_appxlowrank}
\end{corollary}
This follows from the same line of proof as in the proof of Corollary \ref{cor:kwise_appxlowrank} in Appendix 
\ref{sec:kwise_cor_proof}.  
We next, provide a fundamental  lower bound on the error, that matches the upper bound up to a logarithmic factor. 
\begin{theorem}
  \label{thm:bundle_lb}
Suppose $\Theta^*$ has rank $r$. Under the described sampling model, 
there is a universal constant $c>0$ such that that 
the minimax rate where the infimum is taken over all measurable functions over the observed purchase history 
$\{(u_i,v_i,S_i,T_i)\}_{i\in[n]}$ is lower bounded by 
\begin{eqnarray}
  \inf_{\hTheta} \sup_{\Theta^*\in\Omega_\bb} \E\Big[ \frac{1}{\sqrt{d_1d_2}}\fnorm{\hTheta-\Theta^*} \Big] &\geq& c\,\min\left\{ \sqrt{\frac{ e^{-5\bb}\,r\,(d_1+d_2)}{n}}  \,,\,  \frac{\bb (d_1+d_2)}{\sqrt{d_1d_2 \log d}} \right\} \;.
  \label{eq:bundle_lb}
\end{eqnarray} 
\end{theorem}
We provide a proof in Appendix \ref{sec:bundle_lb_proof}.
The first term is dominant, and  
when the sample size is comparable to the latent dimension of the problem, 
Theorem \ref{thm:bundle_ub} is minimax optimal up to a logarithmic factor. 
We emphasize here that the bound in \eqref{eq:bundle_lowrank} 
and the matching lower bound in \eqref{eq:bundle_lb} 
do not depend on the size of the offerings $k_1$ and $k_2$. 
It is independent of how large $k_1$ and $k_2$ are because, we only observe one choice, and intuitively the information we get 
scales at best by a factor of $\log (k_1k_2)$. 
The theorems prove that there is no essential gain in learning from large offerings. 
One might be tempted to 
stop at proving an upper bound that scales as $O(\sqrt{k_1k_2 r(d_1+d_2)\log d / n})$, 
which is larger than \eqref{eq:bundle_lowrank} by a factor of $\sqrt{k_1k_2} $. 
Such a loose bound follows if one ignores the tight concentration analysis that we do using 
the symmetrization technique (e.g. in Lemma \ref{lmm:kwise_hessian2}).
Getting the tight dependency in $k_1$ and $k_2$ is one of the crucial technical challenges we overcome in this paper.  


\subsection{Learning the MNL Model from Customer Choices }
\label{sec:customer}
The results for the {\em customer choice} model follow immediately from the results in 
{\em bundled choice} model by simply setting $k_1=1$, and we explicitly write those corollaries in this section for completeness. 
The proposed estimator is minimax optimal up to a logarithmic factor under the standard customer choice model of sampling. 

\begin{corollary}
\label{thm:customer_ub}
Under the described sampling model, assume $16e^{2\bb} \min\{d_1,d_2\} \log d  \, \leq n \leq \min\{ d^5\,\log d, \\k \max\{d_1^2,d_2^2\}\,\log d\} $, and
$\lambda \in [8 \lambda_0, c_1 \lambda_0]$
with  any constant $c_1=O(1)$ larger than $128$. Then, solving the optimization \eqref{eq:bundleopt} achieves 
\begin{eqnarray}
    \frac{1}{d_1d_2} \fnorm{\hTheta-\Theta^\ast}^2  \; \leq \; 
    48 \sqrt{2}  \, e^{2\bb}  c_1 \lambda \sqrt{r} \,  \fnorm{ \hTheta-\Theta^\ast } + 48 e^{2\bb}   c_1 \lambda \sum_{j=r+1}^{\min\{d_1,d_2\}} \sigma_j(\Theta^*) \;, 
  \label{eq:customer_ub}
\end{eqnarray}
for any $r\in\{1,\ldots, \min\{d_1,d_2\}\}$ 
with probability at least $1-2d^{-3}$ where $d=(d_1+d_2)/2$.
\end{corollary}

\begin{corollary}[{\bf Exact low-rank matrices}]
  Suppose $\Theta^*$ has rank at most $r$. 
  Under the hypotheses of Theorem \ref{thm:customer_ub}, 
  solving the optimization \eqref{eq:bundleopt} with the choice of the regularization parameter 
  $\lambda\in[8\lambda_0 , c_1\lambda_0 ]$ achieves   with probability at least $1-2d^{-3}$, 
  \begin{eqnarray}
  \frac{1}{\sqrt{d_1d_2}} \fnorm{\widehat{\Theta}-\Theta^\ast} \,\leq\, 
  48 \sqrt{2}   e^{3 \bb}  c_1 \sqrt{\frac{r(d_1 + d_2) \log d }{n }} \;.  
  \label{eq:customer_lowrank}
  \end{eqnarray}
  \label{cor:customer_lowrank}
\end{corollary}

\begin{corollary}[{\bf Approximately low-rank matrices}]
  Suppose $\Theta^* \in \bB_q(\rho_q)$ for some $q\in(0,1]$ and $\rho_q>0$. 
  Under the hypotheses of Theorem \ref{thm:customer_ub}, 
  solving the optimization \eqref{eq:bundleopt} with the choice of the regularization parameter 
  $\lambda \in [8\lambda , c_1\lambda]$ achieves 
    with probability at least $1-2d^{-3}$, 
  \begin{eqnarray}
  \frac{1}{\sqrt{d_1d_2}} \fnorm{\widehat{\Theta}-\Theta^\ast} \,\leq\, 
   \frac{2\sqrt{\rho_q}}{\sqrt{d_1d_2}}  \left(  48 \sqrt{2}  c_1 e^{3 \bb} \,\sqrt{\frac{ d_1d_2(d_1+d_2)\log d }{n }} \right)^{\frac{2-q}{2}}\;. 
  \label{eq:customer_appxlowrank}
  \end{eqnarray}
  \label{cor:customer_appxlowrank}
\end{corollary}

We emphasize again that the bound in \eqref{eq:customer_lowrank} 
does not depend on the size of the offerings $k$. 
It is significantly easier to 
stop at proving an upper bounds that scale as $O(\sqrt{k r(d_1+d_2)\log d / n})$, 
which are larger than \eqref{eq:customer_lowrank} by a factor of $\sqrt{k} $. 
Such a loose bound follows if one ignores the tight concentration analysis that we do using 
the symmetrization technique (e.g. in Lemma \ref{lmm:kwise_hessian2}).
Getting the tight dependency in $k$  is one of the crucial technical challenges we overcome in this paper.  

\subsection{Experiments}
We applied our algorithm to a real world choice data set. The implementation is similar to that of higher order comparisons (see Section \ref{sec:kwise_algo}). 

\subsubsection{Real data: Extended Bakery}
\begin{figure}[h!]
\centering
\begin{subfigure}{0.5\textwidth}
  \centering
  \includegraphics[width=1\linewidth]{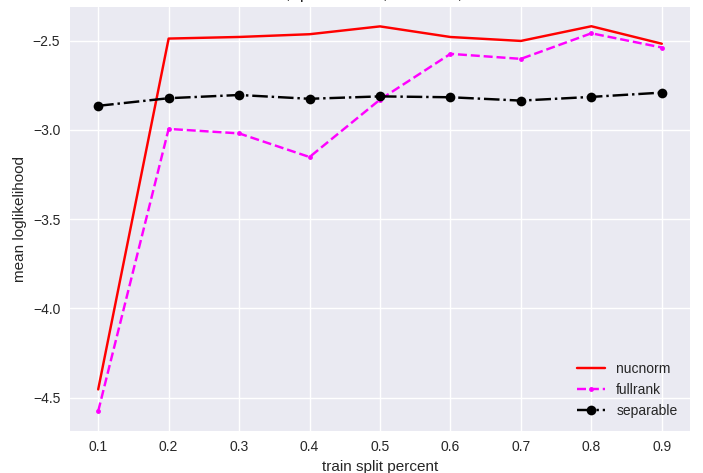}
  \caption{Cakes and Drinks}
  \label{fig:bakery_4_0}
\end{subfigure}%
\hfill
\begin{subfigure}{.5\textwidth}
  \centering
  \includegraphics[width=1\linewidth]{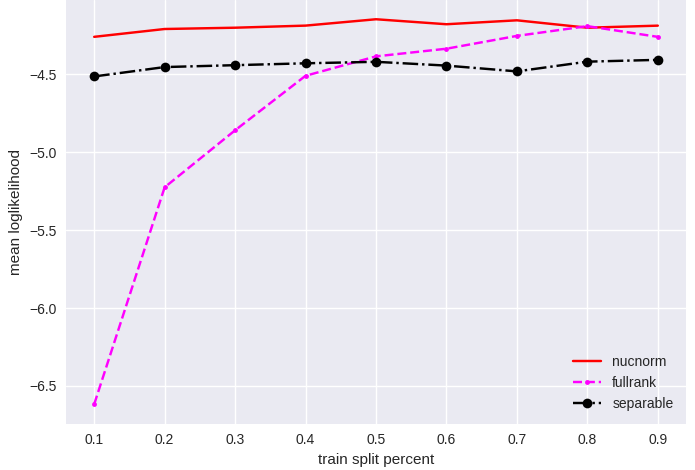}
  \caption{Tarts and Drinks}
  \label{fig:bakery_4_1}
\end{subfigure}
\vskip\baselineskip
\begin{subfigure}{0.5\textwidth}
  \centering
  \includegraphics[width=1\linewidth]{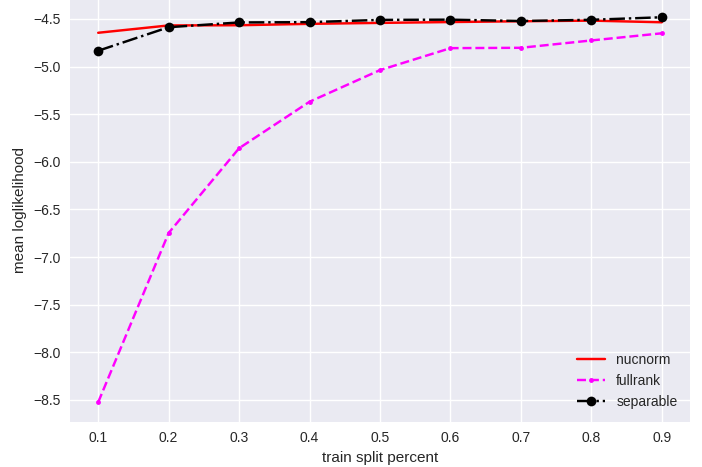}
  \caption{Cookies and Drinks}
  \label{fig:bakery_4_2}
\end{subfigure}%
\hfill
\begin{subfigure}{.5\textwidth}
  \centering
  \includegraphics[width=1\linewidth]{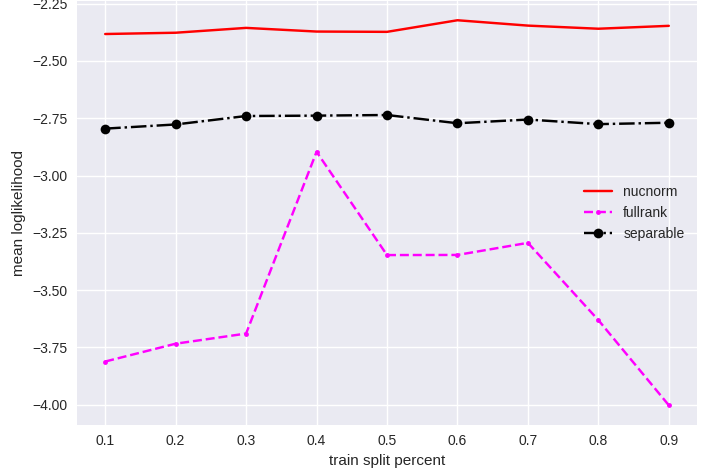}
  \caption{Pastries and Drinks}
  \label{fig:bakery_4_3}
\end{subfigure}
\caption{Bakery: Mean log-likelihood on test data versus fraction of data used for training. 
Nuclear norm minimization improves upon the un-regularized likelihood maximization and separable model, for most of the regimes we consider.}
\label{fig:bakery_likelihood}
\end{figure}
Extended Bakery data set \footnote{Data set is from \url{https://github.com/arbenson/discrete-subset-choice/tree/master/data}.} \cite{BKT18} consists of details of 75,000 purchases at a bakery from a selection of 50 items, specifically each purchase is recorded by the set of items bought together. We use only the 13,579 purchases where a pair of items where bought. We divide the items into five categories: cakes (1-10), tarts (11-21), cookies (22-30), pastries (31-40) and drinks (41-50). We study four cases of bundled pairs, cakes and drinks, tarts and drinks, cookies and drinks, and pastries and drinks. These cases have 1503, 910, 500 and 1791 purchases in them respectively. We model the data with an MNL model parameterized by a matrix $\Theta^*$, such that rows and columns corresponds to first and second categories respectively. 
For every purchase we assume that the subset of alternatives presented is the universal choice set, that is $k_1 = d_1$ and $k_2 = d_2$, so that the purchase of item $j_1$ from category 1 and $j_1$ from category 2 has a probability of $\exp(\Theta^*_{j_1, j_2})/{\sum_{j_1',j_2'=1}^{d_1, d_2} \exp(\Theta^*_{j_1', j_2'})}$, where $d_1, d_2$ are number of items in category 1 and 2 respectively. We also fit the separable model proposed in \cite{BKT18}, 
which is a simpler model with $\Theta^*_{j_, j_2} = a^*_{j_1} + b^*_{j_2}$, where $a \in \reals^{d_1}, b \in \reals^{d_2}$. 

In Fig.~\ref{fig:bakery_likelihood} we plot the mean log-likelihood on the test data versus the fraction of the data used for our nuclear norm minimization based algorithm (`nucnorm'), un-regularized log-likelihood maximization algorithm (`fullrank'), and maximum likelihood estimator for the separable model (`separable') over 10 trials. We see that nuclear norm minimization outperforms the `fullrank' and `separable' algorithms. This is consistent with the marketing practice, of providing different prices for bundled combinations of products, which uses the rationale that, worth of a bundle of products might be different from the sum of the worths of the individual products constituting the bundle. We also note that the un-regularized algorithm (`fullrank') has the worst performance and high variance and separable model fits the samples almost as good as the nuclear norm minimization in the case of cookies and drinks in the large training data regime.

\section{Conclusion}
\label{sec:discussion}

The sample complexity of learning one of the most popular choice models known as MultiNomial Logit model has not been addressed in the literature. 
The main challenge is in the inherent low-rank structure of the parameter to be learned, which leads to a non-convex likelihood maximized problem.  
Thanks to recent advances in learning low-rank matrices, in particular in 1-bit matrix completion \cite{DPVW14}, matrix completion \cite{NW11}, and restricted strong convexity \cite{negahban2009unified}, 
we have a polynomial time algorithm and the technical tools to characterize the fundamental sample complexity of learning MNL from samples. 
This provides a novel algorithm   to learn a low-dimensional representation of users and items from 
users' historical comparisons and choices. 
We study three types of data, pairwise comparison, higher order comparison, and choices, 
and  take the first principle approach of identifying the  fundamental limits and also developing efficient algorithms matching those fundamental trade offs. 
 We provide a unifying framework to learn the latent preferences by solving a convex program. 
For each of the data types, accompanied by natural sampling scenarios, we show that our framework achieves a minimax optimal performance, 
and hence cannot be improved upon other than a small logarithmic factor.  
This opens a new door to learn representations from comparisons and choices, 
and we propose new research directions and challenges below. 
Beyond the low-rank model studied in this paper, 
recent advances in modeling data in a matrix form such as 
low algebraic dimension by \cite{ongie2018tensor} and 
non-parametric approximation by \cite{borgs2017iterative} 
can provide new research directions for modeling choice. 

\bigskip\noindent{\bf Efficient implementations via non-convex optimization.} 
 Nuclear norm minimization, while polynomial-time, is still slow. 
We want first-order methods that are efficient with provable guarantees. 
Two main challenges are providing a good initialization 
to start such non-convex approaches and analyzing gradient descent on the likelihood maximization which is non-convex.
 
 Recent advances in non-convex optimization with rank-constraints have developed via a sequence of 
 innovations that can be summarized as follows, in a number of example problems including 
 matrix completion, robust PCA, matrix sensing, phase retrieval.
 First, a convex relaxation of nuclear norm minimization is analyzed, e.g. \cite{CR09}. 
 Then, a more efficient two-step non-convex optimization approach is proposed with provable guarantees 
 where a global initialization step is followed by a first-order method e.g. \cite{KMO10IT,KMO10JMLR}. 
 Next, first-order methods starting at any initialization point is analyzed via understanding the geometry 
 and checking the stationary points of the objective function e.g.~\cite{GLM16}.  
This recipe, spurred by the advances in the matrix completion problem, has been repeated for several interesting problems involving low rank matrices, over the last decade  and over numerous publications by collective effort of the machine learning community. 

For the problem of learning MNL, we are at the first stage of this progression where we propose a convex relaxation and provide minimax optimal guarantees. 
 We currently do not have the analysis tools to follow up in analyzing an efficient non-convex optimization problem, although writing the algorithm and implementing is straight forward, and also has been proposed in \cite{PNZSD15}. It is a promising research direction to 
 overcome the challenges in analyzing non-convex optimization methods for the MNL likelihood objective function.

\bigskip\noindent
{\bf Assumption on sampling with replacement.} 
As mentioned earlier, 
we assume sampling with replacement, where we can ask a user to compare the same pair more than once, 
and also we can ask a user to compare two copies of the identical item. 
Although such sampling with replacement does not happen in practice, 
the number of such collisions is also very low with high probability under the proposed model. 
Further, such assumption is critical for getting an upper bound that is tight not only in $r$, $d_1$ and $d_2$, 
but also in $k$ for higher order comparisons and choices. 
If, instead, one is interested in sampling without replacement, then one either can resort to 
proving a loose bound that is weaker in its dependence in $k$ (and follows trivially as a corollary of the proof of our results) 
or needs to invent new innovative concentration bounds that do not rely on the powerful symmetrization.
The first option is trivial, so we do not provide such corollaries in this paper, and the second option provides an interesting but technically challenging question of resolving between sampling with replacement and sampling without replacement. This we believe is outside the scope of this paper.

\bigskip\noindent
{\bf Modern data analysis applications.} 
As learning representation from ordinal data is of fundamental interest, 
there are numerous exciting applications that both the algorithmic framework and also the analysis techniques we develop could be naturally extended to. 
We present two such examples. 
First is a recent application of embedding objects with crowdsourced similarity measures, first proposed in \cite{tamuz2011adaptively}.  
Consider a crowdsourcing setting where you have $d$ images and want to learn similarities among those images such that one can embed those images 
in a lower dimensional Euclidean space. One can show to a person a triplet of images $(i_1,i_2,i_3)$ and ask 
whether the image in the middle is more similar to the one in the left or the right. 
 A natural model proposed in \cite{tamuz2011adaptively}  is to assume that there exists a similarity parameter matrix $\Theta\in \reals^{d\times d}$ such that 
 \begin{eqnarray*}
  \prob{\text{$i_1$ is more similar to $i_2$ than $i_3$ is to $i_2$}} \; = \;  \frac{e^{\Theta_{i_2,i_1}}}{e^{\Theta_{i_2,i_1}}+e^{\Theta_{i_2,i_3}}} \;.
 \end{eqnarray*}
A heuristic algorithm is proposed to learn a low-rank $\Theta$ without guarantees. 
Given the similarity of this model to MNL in \eqref{eq:defbtl}, 
both our algorithm and also the analysis will go through to provide a tight characterization of the sample complexity of this problem. 

The second application is in word embedding \cite{mikolov2013distributed}, where 
the goal is to find embeddings for English words in a lower dimensional Euclidean space. 
The most successful word embedding has been based on fitting a low-rank matrix $\Theta\in\reals^{d\times d}$ where $d$ is the size of the vocabulary, 
over an MNL-type model: 
 \begin{eqnarray*}
  \prob{\text{word $i$ and word $j$ appear within distance ten}|\text{word $j$ appear in a sentence}} \; = \;  \frac{e^{\Theta_{ij} }}{\sum_{i'} e^{\Theta_{i'j} }} \;.
 \end{eqnarray*}
As the denominator involves summation over millions of words in the vocabulary, 
efficient heuristics are proposed to learn such a model from skip-grams; a skip-gram is the count matrix counting how many times words co-appear in the same sentence 
within a predefined distance.
There are several challenges in applying our framework directly to such a setting mainly due to the size of the problem, 
but nevertheless our analysis can be applied directly to identify the fundamental minimax sample complexity of learning a word embedding from skip-grams.

\section*{Acknowledgments}
{SN acknowledges support from NSF Grant DMS 1723128. SO and KT acknowledge support from NSF grants CCF 1553452, CNS 1527754, CCF 1705007, and RI 1815535. The authors acknowledge Yu Lu and Ruoyu Sun for helpful and fruitful discussions.}


\bibliographystyle{abbrv}
\vskip 0.2in
\bibliography{ranking}

\newpage

\appendix

\section{Proof of the Upper Bound for Graph Sampling Theorem \ref{thm:graph_ub}}
\label{sec:graph_ub_proof}
The proof of the theorem relies on the following two lemmas. First lemma shows that the negative of the log-likelihood satisfies Restricted Strong Convexity with high probability.
\begin{lemma} \label{lem:graph_rsc}
\textbf{(Restricted Strong Convexity)} Let $R =\max \left\{\sqrt{\frac{\sigma \log(2d)}{n}}, \frac{\sigma_{\min}(L)^{-1/2}\log(2d)}{n} \right\}$ and the set $\calA(\bb) = \left\{\Theta \in \mathbf{R}^{d_1 \times d_2}, \lnorm{\Theta}{\infty} \leq \bb,  \Lnucnorm{\Theta} \le \frac{\Fnorm{\Theta L^{1/2}}^2}{16\bb d_1 R} \right\}$. Then we have,
\begin{align}
\frac1n \sum_{i=1}^n \left( \Iprod{\Theta}{X_i} \right)^2 \geq \frac1{3d_1} \Lnorm{\Theta}^2,\;\;\; \forall\; \Theta \in \calA(\bb),
\end{align}
with probability at least $1-2(2d)^{-4}$, provided that $n \leq \log{(2d)} \min \{2^2 (d_1 \sigma_{\min}(L)^{-1})^{2/3},\\ 2^6 d_1^2 \sigma^2\}$, .
\end{lemma}

Here the upper bound on $n$ may not be necessary, but it is present due to a technical difficulty in using the peeling argument. The intuition behind the above lemma is that the empirical average uniformly
concentrates around its expectation. Proof is in Section \ref{sec:graph_rsc_proof}. The next lemma says that the gradient of the log-likelihood at the actual parameter matrix, $\Theta^*$ is controllably small.
\begin{lemma} \label{lem:graph_gradient}
\textbf{(Bounded Gradient)} Let $R =\max \left\{\sqrt{\frac{\sigma \log(2d)}{n}}, \frac{\sigma_{\min}(L)^{-1/2}\log(2d)}{n} \right\}$. The spectral norm of gradient of the log-likelihood at the actual parameter matrix, $\nabla \calL(\Theta^*)$, can be upper-bounded with high probability as follows,
\begin{align}
\prob{ \opnorm{\nabla \calL(\Theta^*) L^{-1/2}} \geq \sqrt{32}R} \leq \frac{1}{(d_1 + d_2)^3}
\end{align}
\end{lemma}

\noindent
Proof the above lemma is in Section \ref{sec:graph_gradient_proof}. Let $\Delta = \hat{\Theta} - \Theta^*$.

\noindent
\textbf{Case 1: $\Delta \notin \calA(2\bb)$} Then,
\begin{align*}
\Lnorm{\Delta}^2 \leq 32\bb d_1 R \Lnucnorm{\Delta}
\end{align*}
\textbf{Case 2: $\Delta \in \calA(2\bb)$} We first write down the second order Taylor series expansion of $\calL(\hat{\Theta})$ at around $\Theta=\Theta^*$.
\begin{align}
-\calL(\hat{\Theta}) = -\calL(\Theta^*) + \Iprod{-\nabla \calL(\Theta^*)}{\Delta} + \frac1{2n} \sum_{i=1}^n \psi\left( \Iprod{\Theta^*}{X^{(i)}} + s\Iprod{\Delta}{X^{(i)}}\right) \Iprod{\Delta}{X^{(i)}}^2,
\end{align}
where $\psi(x) = {e^x}/(1+e^x)^2,\; x \in [-2\bb, 2\bb]$ and $s \in [0,1]$. Next using Lemma \ref{lem:graph_rsc} and the fact that $\psi(x)$ attains minimum at $x=2\bb$ we get,
\begin{align}
-\calL(\hat{\Theta}) + \calL(\Theta^*) + \Iprod{\nabla \calL(\Theta^*)}{\Delta} \geq  \frac1{2n} \sum_{i=1}^n \psi(2\bb) \Iprod{\Delta}{X^{(i)}}^2 \geq \frac{\psi(2\bb)}{6 d_1} \Lnorm{\Delta}^2,
\end{align}
with probability at least $1 - 1/(d_1 + d_2)^3$. Since $\hat{\Theta}$ is the minimizer for the objective function \ref{eq:graph_obj}, we have,
\begin{align*}
-\calL(\hat{\Theta}) + \lambda \Lnucnorm{\hat{\Theta}} \leq -\calL(\Theta^*) + \lambda \Lnucnorm{\Theta^*},
\end{align*}
which in turn gives us,
\begin{align}
\frac{\psi(2\bb)}{6 d_1} \Lnorm{\Delta}^2 &\leq  -\calL(\hat{\Theta}) + \calL(\Theta^*) + \Iprod{\nabla \calL(\Theta^*)}{\Delta} \nonumber \\  \\ &\leq \lambda\left(\Lnucnorm{\Theta^*} - \Lnucnorm{\hat{\Theta}}\right) + \Iprod{\nabla \calL(\Theta^*)}{\Delta} \nonumber \\
&\leq \lambda\left( \Lnucnorm{\Delta}\right) + \Iprod{\nabla \calL(\Theta^*) L^{-1/2}}{\Delta L^{1/2}} \nonumber\\
&\leq \lambda\left( \Lnucnorm{\Delta}\right) + \opnorm{\nabla \calL(\Theta^*) L^{-1/2}} \Lnucnorm{\Delta},
\end{align}
where last two inequalities follow from the triangle inequality for nuclear norm and generalized H\"{o}lder's inequality. Now we put $\lambda = 2\sqrt{32}R$ and use Lemma \ref{lem:graph_gradient} to get,
\begin{align}
\Lnorm{\Delta}^2 \leq \frac{6d_1}{\psi(2\bb)} \left(\lambda + \frac{\lambda}2 \right) \Lnucnorm{\Delta} \leq \frac{9 d_1 \lambda}{\psi(2\bb)} \Lnucnorm{\Delta},
\end{align}
with probability at least  $1 - 1/(d_1 + d_2)^3$. Combining Case 1 and 2 we get,
\begin{align*}
\Lnorm{\Delta}^2 &\leq 9 \left( \bb + \frac1{\psi(2\bb)} \right) d_1 \lambda \Lnucnorm{\Delta} \nonumber
\end{align*}

\begin{lemma}
\label{lmm:graph_deltabound}
If $\lambda\geq 2 \lnorm{\nabla\calL (\Theta^*)}{2}$, then we have
\begin{eqnarray}
  \Lnucnorm{\Delta} &\le& 4\sqrt{2 r} \Lnorm{\Delta} + 4 \sum^{\min\{d_1,d_2-G\}}_{j = r+1}\sigma_j(\Theta^* L^{1/2}) \;, \label{eq:graph_deltabound}
\end{eqnarray}
for all $r \in[\min\{d_1,d_2-G\}]$. (Proof in Section \ref{sec:graph_deltabound})
\end{lemma}
\noindent
Finally, utilizing the above lemma, we get,
\begin{align*}
\frac1{d_1}\Lnorm{\Delta}^2 \leq 36 \lambda \left( \bb + \frac1{\psi(2\bb)} \right) \left(\sqrt{2r} \Lnorm{\Delta} + \sum^{\min\{d_1,d_2 - G\}}_{j = r+1}\sigma_j(\Theta^* L^{1/2})\right)
\end{align*}

\subsection{Proof of Lemma \ref{lem:graph_rsc}}
\label{sec:graph_rsc_proof}
\begin{align}
&\prob{\frac1n \sum_{i = 1}^n \left( \Iprod{\Theta}{X^{(i)}}\right)^2 \geq \frac1{3d_1} \Lnorm{\Theta}^2,\; \forall\; \Theta \in \calA} \nonumber \\&= 1 - \prob{\exists\; \Theta \in \calA \text{, such that } \frac1n \sum_{i = 1}^n \left( \Iprod{\Theta}{X^{(i)}}\right)^2 < \frac1{3d_1} \Lnorm{\Theta}^2}
\end{align}
\noindent
When $\Theta \in \calA$,
\begin{align}
\Lnorm{\Theta}^2 \geq 16 \bb d_1 R \Lnucnorm{\Theta} \geq 16 \bb d_1 R \Lnorm{\Theta} \implies \Lnorm{\Theta} \geq 16 \bb d_1 R := \mu\,,
\end{align}
where the second inequality follows from $\Lnucnorm{\Theta} = \nucnorm{\Theta\,L^{1/2}} \geq \Fnorm{\Theta\,L^{1/2}} = \Lnorm{\Theta}$.

\begin{lemma} \label{lem:graph_sup_bnd}
Let $\calB(D) := \left\lbrace \Theta \in \reals^{d_1 \times d_2} | \lnorm{\Theta}\infty \leq \bb,\; \Lnorm{\Theta} \leq D,\; \Lnucnorm{\Theta} \leq \frac{D^2}{16\bb d_1 R}\right\rbrace$, and,
\\ $Z_D := \sup\limits_{\Theta \in \calB(D)} \left( -\frac1n \sum_{i = 1}^n \left( \Iprod{\Theta}{X^{(i)}}\right)^2 + \frac2{d_1} \Lnorm{\Theta}^2 \right)$, then,
\begin{align}
\prob{Z_D \geq \frac3{2d_1}D^2} \leq \exp\left( - \frac{nD^4}{32 \bb^4 d_1^2}\right).
\end{align}
\end{lemma}
\noindent
Above lemma is proved in Section \ref{sec:graph_sup_bnd_proof}. Let $\beta = \sqrt{\frac{10}9}$, then the sets,
\begin{align}
\calS_\ell  = \left\lbrace \Theta \in \reals^{d_1 \times d_2} | \lnorm{\Theta}{\infty} \leq \bb, \beta^{\ell - 1} \mu \leq \Lnorm{\Theta} \leq \beta^{\ell}\mu,\; \Lnucnorm{\Theta} \leq \frac{(\beta^{\ell} \mu)^2}{16 \bb d_1 R} \right\rbrace,\; \ell = 1,2,3, \ldots,
\end{align}
cover the set $\calA$, that is $\calA \subset \cup_{\ell = 1}^{\infty} \calS_\ell$ and $\calS_\ell \subseteq \calB(\beta^\ell \mu)$. This gives us,
\begin{align}
&\prob{\exists\; \Theta \in \calA \;\text{s.t.}\, \frac1n \sum_{i = 1}^n \left( \Iprod{\Theta}{X^{(i)}}\right)^2 < \frac1{3d_1} \Lnorm{\Theta}^2} \nonumber \\&\leq \sum_{\ell = 1}^ \infty \prob{\exists\; \Theta \in \calS_\ell \;\text{s.t.}\ \frac1n \sum_{i = 1}^n \left( \Iprod{\Theta}{X^{(i)}}\right)^2 < \frac1{3d_1} \Lnorm{\Theta}^2} \nonumber \\
&\leq \sum_{\ell = 1}^ \infty \prob{\exists\; \Theta \in \calB(\beta^\ell \mu) \;\text{s.t.}\ \frac1n \sum_{i = 1}^n \left( \Iprod{\Theta}{X^{(i)}}\right)^2 < \frac1{3d_1} \Lnorm{\Theta}^2}
\end{align}
\noindent
If there exists a $\Theta \in \calB({\beta^\ell \mu})$ such that $\frac1n \sum_{i = 1}^n \left( \Iprod{\Theta}{X^{(i)}}\right)^2 < \frac1{3d_1} \Lnorm{\Theta}^2$ then,
\begin{align}
Z_{\beta^\ell \mu} \geq -\frac1n \sum_{i = 1}^n \left( \Iprod{\Theta}{X^{(i)}}\right)^2 + \frac2{d_1} \Lnorm{\Theta}^2 > \frac5{3d_1} \Lnorm{\Theta}^2 \geq \frac5{3d_1} \beta^{2\ell - 2}\mu^2 = \frac3{2d_1} (\beta^{\ell}\mu)^2, \nonumber
\end{align}
which gives us,
\begin{align}
\prob{\exists\; \Theta \in \calA \ni \frac1n \sum_{i = 1}^n \left( \Iprod{\Theta}{X^{(i)}}\right)^2 < \frac1{3d_1} \Lnorm{\Theta}^2} &\leq \sum_{\ell = 1}^ \infty \prob{Z_{\beta^\ell \mu} > \frac3{2d_1} (\beta^{\ell}\mu)^2} \nonumber \\
&\overset{(a)}{\leq} \sum_{\ell = 1}^ \infty  \exp\left( - \frac{n(\beta^{\ell}\mu)^4}{32 \bb^4 d_1^2}\right)  \nonumber \\
&\overset{(b)}{\leq} \sum_{\ell = 1}^ \infty  \exp\left( - \frac{4\ell(\beta-1)n\mu^4}{32 \bb^4 d_1^2}\right)  \nonumber \\
&\overset{(c)}{\leq} 2 \exp\left( - \frac{4(\beta-1)n\mu^4}{32 \bb^4 d_1^2}\right)  \nonumber\, \\
\end{align}
where $(a)$ is from Lemma \ref{lem:graph_sup_bnd}, $(b)$ is true since $\beta^{4\ell} \geq 4\ell (\beta - 1)$ when $\beta \geq 1$ and $(c)$ is obtained by summing the geometric series, with common ratio less than $1/2$, in previous inequality. Finally if we assume that $n \leq 2^6 d_1^2 \sigma^2 \log{(2d)} \text{ and } n \leq 2^2 (d_1 \sigma_{\min}(L)^{-1})^{2/3} \log{(2d)}$, then
we have $2^2 \log (2d) \leq 4 (\beta-1)n \mu^4 / 32 \bb^4 d_1^2$ as follows
\begin{align}
\frac{4 (\beta-1)n \mu^4}{32 \bb^4 d_1^2} = \frac{4(\beta-1)n(16 \bb d_1 R)^4}{32 \bb^4 d_1^2} &= 2^{13}(\beta-1)d_1^2 \max\left\lbrace \frac{\sigma^2 \log^2 (2d)}{n}, \frac{\sigma_{\min}(L)^{-2} \log^4 (2d)}{n^3} \right\rbrace \nonumber\\ &\geq 2^2 \log (2d)
\end{align}

\subsection{Proof of Lemma \ref{lem:graph_sup_bnd}}
\label{sec:graph_sup_bnd_proof}
Notice that the $\frac2{d_1} \Lnorm{\Theta}^2$ is the mean of  $\frac1n \sum_{i = 1}^n \Iprod{\Theta}{X^{(i)}}^2$,
\begin{align*}
\expect{\frac1n \sum_{i = 1}^n \Iprod{\Theta}{X^{(i)}}^2} &= \frac{1}{d_1} \sum_{j \in [d1]} \sum_{k,l \in [d2]} (\Theta_{j,k} - \Theta_{j,l})^2 P_{k,l} \\
&= \frac{2}{d_1} \sum_{j} \sum_{k} \Theta_{j,k}^2 \sum_{l}P_{k,l} - 2 \sum_{k,l} \Theta_{j,k}\Theta_{j,l} P_{k,l} \\
&\overset{(a)}{=} \frac{2}{d_1} \sum_j \Iprod{\Theta_{j}\Theta_{j}^\top}{\diag(P_k)} - 2 \Iprod{\Theta_j\Theta_j^\top}{P} \\
&= \frac{2}{d_1} \sum_j \Iprod{\Theta_{j}\Theta_{j}^\top}{L} = \frac{2}{d_1} \Fnorm{\Theta L^{1/2}}^2
\end{align*}
where, in $(a)$ $P_k = \sum_{l \in [d2]} P_{k,l}$ and $\Theta_j$ is the $j$-th row of $\Theta$. Therefore we use the following standard technique to get a handle on supremum of deviation from mean.

First, we use bounded differences property to prove that $Z_D$ concentrates around its mean. We write $Z_D(X^{(1)},\ldots, X^{(n)})$ to represent $Z_D$ as a function of $n$ independent random variables. Now, let $X^{(i)}$ and $\tX^{(i)}$ be two realization of the $i$-th ($1 \leq i \leq n$) random parameter of $Z_D$, then,
\begin{align}
&\abs{Z_D(X^{(1)},\ldots,X^{(i)}, \ldots, X^{(n)}) - Z_D(X^{(1)},\ldots,\tX^{(i)}, \ldots, X^{(n)})} \nonumber \\
= \Bigg| &\sup_{\Theta \in \calB(D)} \left( -\frac1n \sum_{i = 1}^n \Iprod{\Theta}{X^{(i)}}^2 + \frac2{d_1} \Lnorm{\Theta}^2 \right) - \nonumber\\ &\sup_{\Theta' \in \calB(D)} \left( -\frac1n \left( \sum_{i = 1, i \neq i'}^n \Iprod{\Theta'}{X^{(i)}}^2 + \Iprod{\Theta'}{\tX^{(i')}}^2 \right) + \frac2{d_1} \Lnorm{\Theta'}^2 \right) \Bigg|
\end{align}
Now WLOG assume that $Z_D(X^{(1)},\ldots,X^{(i)}, \ldots, X^{(n)}) \geq Z_D(X^{(1)},\ldots,\tX^{(i)}, \ldots, X^{(n)})$ and the first supremum is achieved at $\bar{\Theta}$, which gives us.
\begin{align}
&= \sup_{\Theta \in \calB(D)} \left( -\frac1n \sum_{i = 1}^n \Iprod{\Theta}{X^{(i)}}^2 + \frac2{d_1} \Lnorm{\Theta}^2 \right) - \nonumber \\ &\sup_{\Theta' \in \calB(D)} \left( -\frac1n \left( \sum_{\substack{i = 1 \\ i \neq i'}}^n \Iprod{\Theta'}{X^{(i)}}^2 + \Iprod{\Theta'}{\tX^{(i')}}^2 \right) + \frac2{d_1} \Lnorm{\Theta'}^2 \right) \nonumber \\
&\leq \left( -\frac1n \sum_{i = 1}^n \Iprod{\bar{\Theta}}{X^{(i)}}^2 + \frac2{d_1} \Lnorm{\bar{\Theta}}^2 \right) - \left( -\frac1n \left( \sum_{\substack{i = 1 \\ i \neq i'}}^n \Iprod{\bar{\Theta}}{X^{(i)}}^2 + \Iprod{\bar{\Theta}}{\tX^{(i')}}^2 \right) + \frac2{d_1} \Lnorm{\bar{\Theta}}^2 \right) \nonumber \\
&\leq \sup_{\Theta \in \calB(D)} \frac1n \abs{ \Iprod{\Theta}{X^{(i)}}^2 - \Iprod{\Theta}{\tX^{(i)}}^2} \nonumber \\
&\leq \frac{4\bb^2}{n},
\end{align}
where the last inequality is true since, $\lnorm{\Theta}{\infty} \leq \bb$ for any $\Theta \in \calB(D) \subseteq \Omega_\bb$. Now we upper bound $\expect{Z_D}$ as follows,
\begin{align}
\expect{Z_D} &\overset{(a)}{\leq} 2\expect{\sup_{\Theta \in \calB(D)} \frac1n \sum_{i=1}^n \varepsilon_i \Iprod{\Theta}{X^{(i)}}^2 } \nonumber \\
       &\overset{(b)}{\leq} 4\bb\expect{\sup_{\Theta \in \calB(D)} \frac1n \sum_{i=1}^n \varepsilon_i \Iprod{\Theta L^{1/2}}{X^{(i)}L^{-1/2}} } \nonumber \\
       &\leq 4\bb \expect{\sup_{\Theta \in \calB(D)} \Lnucnorm{\Theta} \opnorm{\frac1n \sum_{i=1}^n \varepsilon_i X^{(i)}L^{-1/2}} } \nonumber \\
       &\leq 4\bb \sup_{\Theta \in \calB(D)} \Lnucnorm{\Theta} \expect{\opnorm{\frac1n \sum_{i=1}^n \varepsilon_i X^{(i)}L^{-`1/2}} }, \nonumber
\end{align}
where $(a)$ is standard symmetrization argument using i.i.d. Rademacher variables $\{\varepsilon_i\}_{i=1}^n$ and since $|\Iprod{\Theta}{X^{(i)}} | \leq 2\bb$ we use Ledoux-Talagrand contraction inequality \cite{ledoux2013probability} to obtain $(b)$.

\begin{lemma} \label{lem:graph_bernstein}
Let $R =\max \left\{\sqrt{\frac{\sigma \log(2d)}{n}}, \frac{\sigma_{\min}(L)^{-1/2}\log(2d)}{n} \right\}$. For $\{X^{(i)}\}_{i=1}^n$ as defined in the graph sampling and for a binary random variable $\varepsilon_i$ such that $\expect{\varepsilon_i | X^{(i)}} = 0$ and $\abs{\varepsilon_i} \leq 1$, we have,
\begin{align}
\prob{\opnorm{\frac1n \sum_{i=1}^n \varepsilon_i X^{(i)} L^{-1/2}} \geq \sqrt{32} R} \leq \frac{1}{(d_1 + d_2)^3},\;\;\text{and},\;\; \expect{\opnorm{\frac1n \sum_{i=1}^n \varepsilon_i X^{(i)}  L^{-1/2}}} \leq 4 R\,.
\end{align}
\end{lemma}
Proof of the lemma is in Section \ref{sec:graph_bernstein_proof}. Now using Lemma \ref{lem:graph_bernstein} we have $\expect{Z_D} \\\leq 16 R \bb \sup_{\Theta \in \calB(D)} \Lnucnorm{\Theta} \leq \frac{D^2}{d_1}$. Then using the bounded differences property and the upper bound on the mean, we get the McDiarmid's concentration,
\begin{align}
\prob{Z_D - D^2/{d_1} \geq t} &\leq \prob{Z_D - \expect{Z_D} \geq t} \nonumber \\
&\leq \exp{\left( - \frac{nt^2}{8\bb^4} \right)}\,
\end{align}
and putting $t=D^2/2d_1$ gives us the theorem.

\subsection{Proof of Lemma \ref{lem:graph_bernstein}}\label{sec:graph_bernstein_proof}

Let $W_i := \frac1n \varepsilon_i X^{(i)} L^{-1/2} = \frac1n \varepsilon_i e_{j(i)}\left(e_{k(i)} - e_{l(i)}\right)^\top L^{-1/2}$ and pseudo-inverse of $L$ be $L^\dagger$, then, $\opnorm{W_i} \leq \sigma_{\min}(L)^{-1/2}\sqrt{2}/{n}$,
\begin{align}
\expect{W_iW_i^\top} &= \expect{\frac1{n^2} e_{j(i)}\left(e_{k(i)} - e_{l(i)}\right)^\top L^{-1/2} L^{-1/2} \left(e_{k(i)} - e_{l(i)}\right)e_{j(i)}^\top} \nonumber \\
&= \expect{\frac1{n^2} e_{j(i)}e_{j(i)}^\top} \expect{\left(e_{k(i)} - e_{l(i)}\right)^\top L^{\dagger} \left(e_{k(i)} - e_{l(i)}\right)} \nonumber \\
&= \frac1{n^2d_1} \id_{d_1 \times d_1} \times 2\left( \expect{e_{k(i)}^\top L^\dagger e_{k(i)}} - \expect{e_{k(i)}^\top L^\dagger e_{l(i)}} \right) \nonumber \\
&= \frac2{n^2d_1} \left( \sum_{u \in [d_1]} P_u L^\dagger_{u,u} -  \sum_{u,v \in [d_1]} P_{u,v} L^\dagger_{u,v} \right) \id_{d_1 \times d_1} \nonumber \\
&= \frac2{n^2 d_1} \Iprod{L}{L^\dagger} \id_{d_1 \times d_1} \nonumber \\ 
&= \frac{2 (d_2-G)}{n^2 d_1} \id_{d_1 \times d_1} ,
\end{align}

\begin{align}
\expect{W_i^\top W_i} &= L^{-1/2} \expect{\frac1{n^2} \left(e_{k(i)} - e_{l(i)}\right) \left(e_{k(i)} - e_{l(i)}\right)^\top} L^{-1/2} \nonumber \\
&= \frac1{n^2} L^{-1/2} \left( \sum_{u,v=1}^{d_2} \left(e_u - e_v \right) \left(e_u - e_v\right)^\top P_{u,v} \right) L^{-1/2}\nonumber \\
&= \frac1{n^2} L^{-1/2} \left( 2L \right) L^{-1/2}\nonumber \\
&= \frac2{n^2} (\id_{d_2 \times d_2} - \sum_{k \in [G]} g_k g_k^T/\lnorm{g_k}{2}^2)\, \text{, and, }
\end{align}
\begin{align}
\max \left\lbrace \opnorm{\expect{\sum_{i=1}^n W_i W_i^\top } },\ \opnorm{\expect{\sum_{i=1}^n W_i^\top W_i} } \right\rbrace &\leq \sum_{i=1}^n \max \left\lbrace \opnorm{\expect{W_i W_i^\top } },\ \opnorm{\expect{W_i^\top W_i} } \right\rbrace \\
&\leq \frac2n \sigma\,.
\end{align}
where $\sigma = \max\left\lbrace \frac{d_2 - G}{d_1} , 1\right\rbrace$. Now by Matrix Bernstein concentration theorem \cite{Tropp15} we have,
\begin{align}
\prob{\opnorm{\frac1n \sum_{i=1}^n \varepsilon_i X^{(i)} L^{-1/2}} \geq t} &\leq \exp{\left(\frac{-n t^2/2}{2\sigma + \sqrt{2} \sigma_{\min}(L)^{-1/2}t/3}\right)},\, \text{and},\\
`\expect{\opnorm{\frac1n \sum_{i=1}^n \varepsilon_i X^{(i)}  L^{-1/2}}} &\leq \sqrt{\frac{4 \sigma \log(d_1 + d_2)}{n}} + \frac{\sqrt{2 \sigma_{\min}(L)^{-1}}}{3n} \log(d_1 + d_2)\,.
\end{align}
Choosing $t=\max \left\lbrace \sqrt{\frac{24 \sigma \log(d_1+d_2)}{n}},\, \frac{16\sqrt{2\sigma_{\min}(L)^{-1}} \log(d_1 + d_2)}{n}\right\rbrace$ produces the desired result.

\subsection{Proof of Lemma \ref{lem:graph_gradient}}
\label{sec:graph_gradient_proof}
The gradient can be written down as,
\begin{align}
\nabla \calL(\Theta^*) = \frac1n \sum_{i=1}^n \left(y_i -  \frac{\exp(\Iprod{\Theta^*}{X^{(i)}})}{1 + \exp(\Iprod{\Theta^*}{X^{(i)}})}\right) X^{(i)}.
\end{align}
Then Lemma \ref{lem:graph_bernstein} directly gives the result because,
\begin{align}
\expect{y_i -  \frac{\exp(\Iprod{\Theta^*}{X^{(i)}})}{1 + \exp(\Iprod{\Theta^*}{X^{(i)}})}  \bigg| X^{(i)}} = 0\;\;\;\; \text{and}\;\;\; \left|y_i - \frac{\exp(\Iprod{\Theta^*}{X^{(i)}})}{1 + \exp(\Iprod{\Theta^*}{X^{(i)}})}\right| \leq 1. \nonumber
\end{align}

\subsection{Proof of Lemma \ref{lmm:graph_deltabound}}
\label{sec:graph_deltabound}
Denote the singular value decomposition of $\Theta^\ast L^{1/2}$ by $\Theta^\ast L^{1/2} = U \Sigma V^\top$, where $U\in \reals^{d_1 \times d_1}$ and $V \in \reals^{d_2 \times d_2}$ are orthogonal matrices.
For a given $r\in[\min\{d_1,d_2-G\}]$, 
let $U_r=[u_1, \ldots, u_r]$  and $V_r=[v_1, \ldots, v_r]$, where $u_i \in \reals^{d_1 \times 1}$ and $v_i \in \reals^{d_2 \times 1}$ are the left and right singular vectors corresponding to the $i$-th largest singular value, respectively.
Define $T$ to be the subspace spanned by all matrices in $\reals^{d_1 \times d_2}$ of the form $U_rA^\top$ or $BV_r^\top$ for any $A\in \mathbb{R}^{d_2\times r}$ or
$B \in \reals^{d_1 \times r}$, respectively.
The orthogonal projection of any matrix $M\in \mathbb{R}^{d_1 \times d_2}$ onto the space $T$ is given by $\mathcal{P}_T(M) = U_rU_r^\top M+MV_rV_r^\top-U_rU_r^\top MV_rV_r^\top $. The projection of $M$ onto the complement space $T^\perp$ is $\mathcal{P}_{T^\perp}(M) = (I-U_rU_r^\top) M (I-V_r V_r^\top )$.
The subspace $T$ and the respective projections onto $T$ and $T^\perp$ play crucial a role in the analysis of nuclear norm minimization, since they define the sub-gradient of the nuclear norm at $\Theta^*$. We refer to \cite{CR09} for more detailed treatment of this topic. 

Let $\Delta'=\calP_{T}(\Delta L^{1/2})$ and $\Delta'' = \calP_{T^\perp} (\Delta L^{1/2})$. Notice that $\calP_{T} (\Theta^\ast L^{1/2}) = U_r \Sigma_r V_r^\top$, where
$\Sigma_r \in \reals^{r \times r}$ is the diagonal matrix formed by the top $r$ singular values. Since $\calP_{T} (\Theta^\ast L^{1/2})$ and $\Delta''$
have row and column spaces that are orthogonal, it follows from Lemma 2.3 in \cite{RFP10} that
\begin{align*}
\nucnorm{ \calP_{T} (\Theta^\ast L^{1/2}) - \Delta''} =\nucnorm{\calP_{T} (\Theta^\ast L^{1/2})} + \nucnorm{\Delta''}\;.
\end{align*}
Hence, in view of the triangle inequality,
\begin{align}
\nucnorm{\hTheta L^{1/2}}&= \nucnorm{ \calP_{T} (\Theta^\ast L^{1/2}) + \calP_{T^\perp} (\Theta^\ast L^{1/2}) - \Delta' - \Delta''} \nonumber \\
& \ge \nucnorm{ \calP_{T} (\Theta^\ast L^{1/2})- \Delta'' }  - \nucnorm{\calP_{T^\perp} (\Theta^\ast L^{1/2}) - \Delta' }\nonumber \\
& =  \nucnorm{ \calP_{T} (\Theta^\ast L^{1/2})} + \nucnorm{\Delta''} -   \nucnorm{\calP_{T^\perp} (\Theta^\ast L^{1/2}) - \Delta' } \nonumber \\
& \ge \nucnorm{ \calP_{T} (\Theta^\ast L^{1/2})} + \nucnorm{\Delta''}  - \nucnorm{\calP_{T^\perp} (\Theta^\ast L^{1/2}) } - \nucnorm{\Delta'} \nonumber \\
& =  \nucnorm{\Theta^\ast L^{1/2}} + \nucnorm{\Delta''}  -2 \nucnorm{\calP_{T^\perp} (\Theta^\ast L^{1/2}) }- \nucnorm{\Delta'}. \label{eq:graph_nucleardecomp}
\end{align}
Because $\widehat{\Theta}$ is an optimal solution, we have
\begin{align}
  \lambda \left( \nucnorm{\widehat{\Theta} L^{1/2}} - \nucnorm{\Theta^\ast L^{1/2}} \right)  &\le 
   - \calL(\Theta^\ast) + \calL(\widehat{\Theta}) \nonumber \\
   &\overset{(a)}{\le}  
  \Iprod{\Delta L^{1/2}}{\nabla \calL(\Theta^\ast)L^{-1/2} } \nonumber \\
  &\overset{(b)}{\le}
  \Lnucnorm{\Delta}  \lnorm{\nabla \calL(\Theta^\ast) L^{-1/2}}{2}  \le \frac{\lambda}{2} \Lnucnorm{\Delta},   \label{eq:graph_optimalitycondition}
\end{align}
where $(a)$ holds due to the convexity of $-\calL$; $(b)$ follows from the Cauchy-Schwarz inequality; the last inequality holds due to
the assumption that $\lambda \geq 2 \lnorm{\nabla\calL (\Theta^*)}{2}$. Combining \eqref{eq:graph_nucleardecomp} and \eqref{eq:graph_optimalitycondition} yields
\begin{align*}
2 \left(  \nucnorm{\Delta''}  -2 \nucnorm{\calP_{T^\perp} (\Theta^\ast L^{1/2}) }- \nucnorm{\Delta'} \right) \le \Lnucnorm{\Delta} \le  \nucnorm{\Delta'} +  \nucnorm{\Delta''}.
\end{align*}
Thus $  \nucnorm{\Delta'' } \le 3  \Lnucnorm{\Delta' } + 4  \nucnorm{ \calP_{T^\perp} (\Theta^\ast L^{1/2}) }$. By triangle inequality, 
\begin{eqnarray*}
  \Lnucnorm{\Delta } \le 4\nucnorm{\Delta' } + 4  \nucnorm{ \calP_{T^\perp} (\Theta^\ast L^{1/2}) }\;.
\end{eqnarray*} 
 Notice that $\Delta'=U_r U_r^\top \Delta L^{1/2} + ( I-U_rU_r^\top)  \Delta L^{1/2} V_rV_r^\top$. Both $U_r U_r^\top \Delta L^{1/2}$ and  $ (I-U_rU_r^\top)  \Delta L^{1/2} V_rV_r^\top$ have rank at most $r$.  Thus  $\Delta'$ has rank at most $2r$. Hence, $\nucnorm{\Delta' }
\le \sqrt{2r} \fnorm{\Delta'} \le \sqrt{2r} \fnorm{\Delta L^{1/2}} \le \sqrt{2r} \Lnorm{\Delta}$. Then the theorem follows because \\$\nucnorm{ \calP_{T^\perp} (\Theta^\ast L^{1/2}) } = \sum^{\min\{d_1,d_2\}}_{j=r+1}\sigma_j(\Theta^* L^{1/2} ).$

\section{Proof of the Information-theoretic Graph Sampling Lower Bound, Theorem \ref{thm:graph_lb}}
\label{sec:graph_lb_proof}

The proof uses Fano Inequality based packing set argument to get an lower bound on the error of any (measurable) estimator. We will construct a packing set in $\Omega_\bb$ with a minimum distance of $\delta$ between any pair of elements in the packing.

Let $\{ \Theta^{(1)}, \Theta^{(2)}, \ldots, \Theta^{(M)} \}$ be a set of $M$ matrices within the set $\Omega_\bb$, satisfying \\$\Lnorm{\Theta^{(\ell_1)} - \Theta^{(\ell_1)} } \geq \delta$ for all $\ell_1, \ell_2 \in [M]$. Now, $\Theta^{(N)}$ is uniformly drawn from this set and then the comparison results (according to MNL model) of $n$ randomly chosen pairs of items, each drawn according to the probability
matrix $P$ and each compared by uniformly chosen user. Let $\hN$ be the best estimator of $N$ from the observations. Then we can show that,
\begin{eqnarray}
  \sup_{\Theta^*\in\Omega_\bb } \prob{ \Lnorm{\hTheta-\Theta^*}^2 \geq \frac{\delta^2}{2} } &\geq& \prob{\hN \neq N}\;,
  \label{eq:graph_lb_fano1}
\end{eqnarray}

Now we have converted the problem of finding the minimum estimation error, into finding the minimum probability error of a $M$-ary hypothesis testing problem. If we can prove that the above RHS is lower bounded by $1/2$, we are done.

The generalized Fano’s inequality along with data processing inequality gives us,
 \begin{eqnarray}
  \prob{\hN\neq N } &\geq& 1-\frac{\E[\, I(\hN;N) \,] + \log 2}{\log M} \\
  &\geq& 1-\frac{ {M \choose 2}^{-1} \sum_{\ell_1,\ell_2\in[M]} D_{\rm KL}(\Theta^{(\ell_1)} \| \Theta^{(\ell_2)}) +\log 2}{\log M} \label{eq:graph_fano2}\;,
 \end{eqnarray}
where $D_{\rm KL}(\Theta^{(\ell_1)} \| \Theta^{(\ell_2)})$ denotes the \textit{expected} Kullback-Leibler divergence between the probability distributions of the comparison results of the observed $nd_1$ pairs, for $N = \ell_1$ and $N = \ell_2$. The expectation is taken over different choices for the selected pairs for comparison.
\begin{align}
  D_{\rm KL}(\Theta^{(\ell_1)} \| \Theta^{(\ell_2)}) \overset{}{=} \frac{n}{d_1} \sum_{\substack{i \in [d_1]\\ \{j,j'\} \subset [d_2]}} 2 P_{j,j'} \Bigg[ &\frac{e^{\Theta_{ij}^{(\ell_1)}}}{e^{\Theta_{ij}^{(\ell_1)}} + e^{\Theta_{ij'}^{(\ell_1)}}} \log\left(\frac{e^{\Theta_{ij}^{(\ell_1)}}\Bigg/\left(e^{\Theta_{ij}^{(\ell_1)}} + e^{\Theta_{ij'}^{(\ell_1)}}\right)}{e^{\Theta_{ij}^{(\ell_2)}}\Bigg/\left(e^{\Theta_{ij}^{(\ell_2)}} + e^{\Theta_{ij'}^{(\ell_2)}}\right)}\right) \\
&+ \frac{e^{\Theta_{ij'}^{(\ell_1)}}}{e^{\Theta_{ij}^{(\ell_1)}} + e^{\Theta_{ij'}^{(\ell_1)}}} \log\left(\frac{e^{\Theta_{ij'}^{(\ell_1)}}\Bigg/\left(e^{\Theta_{ij}^{(\ell_1)}} + e^{\Theta_{ij'}^{(\ell_1)}}\right)}{e^{\Theta_{ij'}^{(\ell_2)}}\Bigg/\left(e^{\Theta_{ij}^{(\ell_2)}} + e^{\Theta_{ij'}^{(\ell_2)}}\right)}\right)  \Bigg]
\end{align}
where $n$ is the number of pairs of items selected and compared by one random user each, $P_{j,j'}$ is half the probability with which item pair $\{j, j'\}$ is selected and the observation probabilities come from the standard MNL model. Let $x_{ijj'} \equiv e^{\Theta_{ij'}^{(\ell_1)}}/(e^{\Theta_{ij}^{(\ell_1)}} + e^{\Theta_{ij'}^{(\ell_1)}})$ and $y_{ijj'} \equiv e^{\Theta_{ij'}^{(\ell_2)}}/(e^{\Theta_{ij}^{(\ell_2)}} + e^{\Theta_{ij'}^{(\ell_2)}})$.

\begin{align}
  D_{\rm KL}(\Theta^{(\ell_1)} \| \Theta^{(\ell_2)}) &\overset{(a)}{=} n \sum_{i \in [d_1]} \frac{1}{d_1}\sum_{\{j,j'\} \subset [d_2]} 2 P_{j,j'} \Bigg[x_{ijj'} \log\frac{x_{ijj'}}{y_{ijj'}} + (1-x_{ijj'}) \log\frac{1-x_{ijj'}}{1-y_{ijj'}}\Bigg] \\
  &\overset{(a)}{\leq} n \sum_{i \in [d_1]} \frac{1}{d_1}\sum_{\{j,j'\} \subset [d_2]} 2 P_{j,j'} \Bigg[x_{ijj'} \frac{x_{ijj'} - y_{ijj'}}{y_{ijj'}} + (1-x_{ijj'}) \frac{y_{ijj'}-x_{ijj'}}{1-y_{ijj'}}\Bigg]\\
    &\overset{}{=} 2n \sum_{i \in [d_1]} \frac{1}{d_1}\sum_{\{j,j'\} \subset [d_2]} \frac{(x_{ijj'} - y_{ijj'})P_{j,j'}(x_{ijj'} - y_{ijj'})}{y_{ijj'}(1-y_{ijj'})}\\
    &\overset{(b)}{\leq} 8ne^{2\bb} \sum_{i \in [d_1]} \frac{1}{d_1}\sum_{\{j,j'\} \subset [d_2]} (x_{ijj'} - y_{ijj'})P_{j,j'}(x_{ijj'} - y_{ijj'}),
\end{align}
where $(a)$ is due to the fact that $\log(x/y) \leq (x-y)/y$ for $x/y \geq 0$ and $(b)$ is true because $|\Theta^{(\ell_2)}_{ij}| \leq \bb$ implies, $y_{ijj'} = e^{\Theta^{(\ell_2)}_{ij}}/(e^{\Theta^{(\ell_2)}_{ij}} + e^{\Theta^{(\ell_2)}_{ij'}}) \geq e^{-2\bb}/2$ which in turn implies, $y_{ijj'}(1-y_{ijj'}) \geq e^{-2\bb}(2-e^{-2\bb})/4 \geq e^{-2\bb}/4$. Let $f(z) = 1/(1+e^{-z})$, a $1$-Lipschitz function, it can be seen that $(x_{ijj'} - y_{ijj'})^2 = (f(\Theta^{(\ell_1)}_{ij} -\Theta^{(\ell_1)}_{ij'}) - f(\Theta^{(\ell_2)}_{ij} -\Theta^{(\ell_2)}_{ij'}))^2 \leq ((\Theta^{(\ell_1)}_{ij} -\Theta^{(\ell_1)}_{ij'}) - (\Theta^{(\ell_2)}_{ij} -\Theta^{(\ell_2)}_{ij'}))^2$. This gives us,
\begin{align}
  D_{\rm KL}(\Theta^{(\ell_1)} \| \Theta^{(\ell_2)}) &\overset{}{\leq} \frac{8ne^{2\bb}}{d_1} \sum_{i \in [d_1]} \sum_{\{j,j'\} \subset [d_2]} P_{u,v} ((\Theta^{(\ell_1)}_{ij} -\Theta^{(\ell_2)}_{ij}) - (\Theta^{(\ell_1)}_{ij'} -\Theta^{(\ell_2)}_{ij'}))^2, \\
  &\overset{(a)}{\leq}  \frac{8ne^{2\bb}}{d_1} \sum_{i \in [d_1]} (\Theta^{(\ell_1)}-\Theta^{(\ell_2)})_i L (\Theta^{(\ell_1)} -\Theta^{(\ell_2)})_i, \\
  &\overset{}{=} \frac{8ne^{2\bb}}{d_1} \sum_{i \in [d_1]} (\Theta^{(\ell_1)}-\Theta^{(\ell_2)})_i L (\Theta^{(\ell_1)} -\Theta^{(\ell_2)})_i, \\
  &\overset{}{=} \frac{8ne^{2\bb}}{d_1} \fnorm{(\Theta^{(\ell_1)}-\Theta^{(\ell_2)}) L^{1/2}}^2\\
  &\overset{}{=} \frac{8ne^{2\bb}}{d_1} \Lnorm{\Theta^{(\ell_1)}-\Theta^{(\ell_2)}}^2\\
\end{align}
where $(a)$ is due to the fact that $L = \diag(P_u) - P$ is the Laplacian of the probability matrix P,
and $\Theta_i$ denotes the $i$-th row of matrix $\Theta$. Combining the above with \eqref{eq:graph_fano2}, we get,
\begin{align}
  \prob{\hN\neq N} &\geq 1- \frac{ {M \choose 2}^{-1}  \sum_{\ell_1,\ell_2\in[M]}  (8 n e^{2\bb} /d_1)\Lnorm{\Theta^{(\ell_1)} - \Theta^{(\ell_2)}}^2 + \log 2}{\log M}\;.
\end{align}
The remainder of the proof relies on the following probabilistic packing.
\begin{lemma}
	\label{lem:graph_packing2} 
	For each $r\in\{1,\ldots,d_1 \}$, and for any positive 
	$\delta>0$ there exists a family of $d_1\times d_2$ dimensional matrices 
	$\{\Theta^{(1)},\ldots,\Theta^{(M(\delta))}\}$ with cardinality 
	$M(\delta) = \lfloor \exp(r d_1 /256)\rfloor$ such that each matrix is rank $r$ and the following bounds hold:
	\begin{eqnarray}
	\Lnorm{\Theta^{(\ell)}}&\leq& \delta \;, \text{ for all }\ell\in[M]\\
	\Lnorm{\Theta^{(\ell_1)} - \Theta^{(\ell_2)}} &\geq& \delta \;, \text{ for all } 
	\ell_1, \ell_2\in[M] \label{eq:graph_packing_lower2}\\
	\Theta^{(\ell)} &\in& \Omega_{\tilde{\bb}} \;, \text{ for all } \ell\in [M]\;, 
	\end{eqnarray}
	with $\tilde{\bb}=   (8\delta/d_2)\sqrt{2\log d} $ for $d=(d_1+d_2)/2$.
\end{lemma}
Now if we assume $\delta \leq \bb d_2/8\sqrt{2 \log d}$, we get $\Theta^{(\ell)} \in \Omega_\bb$ for $\ell \in [M]$. The above lemma also implies that $\fnorm{\Theta^{(\ell_1)} - \Theta^{(\ell_2)}}^2 \leq 4\delta^2$ which implies,
\begin{align}
  \prob{\hN\neq N} &\geq 1- \frac{32 n e^{2\bb} \delta^2 /d_1 + \log 2}{rd_1/256} \geq \frac12\;,
\end{align}
where the last inequality holds when $\delta \leq (e^{-\bb}/128)\sqrt{rd_1^2/n}$. Along with \eqref{eq:graph_lb_fano1}, this proves that,
\begin{align}\label{eq:graph_lb_delta1}
\inf_{\hTheta} \sup_{\Theta^*\in\Omega_\bb} \E\Big[\Lnorm{\hTheta-\Theta^*} \Big] \geq \frac{\delta}{2}\;,
\end{align}
for all $\delta \leq \min\{\bb d_2/8\sqrt{2 \log d}, (e^{-\bb}/128)\sqrt{rd_1^2/n}\}$.

Similarly using the following lemma, Lemma \ref{lem:graph_packing}, we can prove that,
\begin{align}\label{eq:graph_lb_delta2}
\inf_{\hTheta} \sup_{\Theta^*\in\Omega_\bb} \E\Big[\Lnorm{\hTheta-\Theta^*} \Big] \geq \frac{\delta}{2}\;,
\end{align}
for all $\delta \leq \min\{\bb\sqrt{rd_1}/\trace{\sqrt{(L_r)^\dagger}}, (e^{-\bb}/128)\sqrt{rd_1^2/n}\}$.
\begin{lemma}
  \label{lem:graph_packing} 
  For each $r\in\{1,\ldots,d_1 \}$, and for any positive 
  $\delta>0$ there exists a family of $d_1\times d_2$ dimensional matrices 
  $\{\Theta^{(1)},\ldots,\Theta^{(M(\delta))}\}$ with cardinality 
  $M(\delta) = \lfloor \exp(r d_1 /256)\rfloor$ such that each matrix is rank $r$ and the following bounds hold:
  \begin{eqnarray}
    \Lnorm{\Theta^{(\ell)}}&\leq& \delta \;, \text{ for all }\ell\in[M]\\
    \Lnorm{\Theta^{(\ell_1)} - \Theta^{(\ell_2)}} &\geq& \delta \;, \text{ for all } 
    \ell_1, \ell_2\in[M] \label{eq:graph_packing_lower}\\
    \Theta^{(\ell)} &\in& \Omega_{\tilde{\bb}} \;, \text{ for all } \ell\in [M]\;, 
  \end{eqnarray}
  with $\tilde{\bb}= \delta \sqrt{\trace{(L_r)^\dagger}}/\sqrt{rd_1}$, where $L_r$ is the (best) rank $r$ approximation of $L$.
\end{lemma}
Now
combining \eqref{eq:graph_lb_delta1} and \eqref{eq:graph_lb_delta2}, and
maximizing the RHS proves the theorem.
%


\subsection{Proof of Lemma \ref{lem:graph_packing2}}
Following the construction in \cite{NW11}, 
we use probabilistic method to prove the existence of the desired family. 
We will show that the following procedure succeeds in producing the desired family with probability at least half, which proves its  existence. 
Let $d=(d_1+d_2)/2$, and suppose $d_2\geq d_1$ without loss of generality. 
For the choice of  $M'=e^{r d_2 /576}$, and for each $\ell \in [M']$, generate a rank-$r$ matrix $\Theta^{(\ell)}\in\reals^{d_1\times d_2}$ as follows: 
\begin{eqnarray}
  \Theta^{(\ell)} &=& \frac{\delta}{\sqrt{r d_2}} U(V^{(\ell)})^T \Big( \id_{d_2 \times d_2} -\sum_{i \in [G]} \frac{g_i g_i^T}{g_i^T g_i} \Big)\;, 
  \label{eq:defpacking}
\end{eqnarray}
where 
$U\in\reals^{d_1\times r}$ is a random orthogonal basis such that $U^TU=\id_{r\times r}$ and $V^{(\ell)} \in\reals^{d_2\times r}$ is a random matrix with each entry $V^{(\ell)}_{ij} \in\{-1,+1\}$ chosen independently and uniformly at random. By construction, notice that,
\begin{align}\label{eq:enet_single_ub}
\Lnorm{\Theta^{(\ell)}}^2 = \frac{\delta^2}{r\,d_2}\fnorm{(V^{(\ell)})^T L^{1/2}}^2 
= \frac{\delta^2}{r\,d_2} \sum_{i \in [d_2]} (V^{(\ell)}_i)^T L V^{(\ell)}_i \leq  \frac{\delta^2}{r\,d_2} \lnorm{L}{2} \sum_{i \in [d_2]} \fnorm{V^{(\ell)}_i}^2 = \delta^2\,,
\end{align} 
where $V^{(\ell)}_i$ is the $i$-th column of $V^{(\ell)}$, since ${g_i}_{i=1}^G$ span the null space of the Laplacian $L$, $\lnorm{L}{2} \leq 1$, and $\fnorm{V^{(\ell)} }=\sqrt{rd_2}$. Now, consider $\fnorm{\Theta^{(\ell_1)}-\Theta^{(\ell_2)}}^2 = \\
(\delta^2/(r d_2))\Lnorm{(V^{(\ell_1)} - V^{(\ell_2)})^T }^2 \equiv f(V^{(\ell_1)},V^{(\ell_2)})$ 
which is a function over $2 r d_2$ i.i.d. random Rademacher variables $V^{(\ell_1)}$ and $V^{(\ell_2)}$ which define $\Theta^{(\ell_1)}$ and $\Theta^{(\ell_2)}$ respectively. 
Since $f$ is Lipschitz in the following sense, 
we can apply McDiarmid's concentration inequality. 
For all $(V^{(\ell_1)},V^{(\ell_2)})$ and $(\tV^{(\ell_1)},\tV^{(\ell_2)})$ that differ in only one variable, say 
$\tV^{(\ell_1)} = V^{(\ell_1)} + 2e_{ij}$, for some standard basis matrix $e_{ij}$, we have  
\begin{align}
  & \big| f(V^{(\ell_1)},V^{(\ell_2)})-f(\tV^{(\ell_1)},\tV^{(\ell_2)}) \big| \;=\; \nonumber\\
  & \;\; \left| \; \frac{\delta^2}{r\, d_2} \Lnorm{(V^{(\ell_1)} - V^{(\ell_2)})^T  }^2 \, 
  - \, \frac{\delta^2}{r\, d_2} \Lnorm{(V^{(\ell_1)} - V^{(\ell_2)} + 2e_{ij})^T }^2 \; \right| \\
  & \;\; =\; \frac{\delta^2}{r\,d_2}  \left|\;  2e_i^T L (V^{(\ell_1)}_j - V^{(\ell_2)}_j) + 4 e_i^T L e_i \; \right| \\
  & \;\; =\; \frac{\delta^2}{r\,d_2} (2 \lnorm{L_i}{1} \lnorm{V^{(\ell_1)}_j - V^{(\ell_2)}_j}{\infty} + 4 P_i) \; \\
  & \;\; \leq \; \frac{12\,\delta^2}{r\,d_2 } \;,
\end{align}
where we used the fact that $\lnorm{L_i}{1} = 2P_i \leq 2$ and $V^{(\ell_1)} - V^{(\ell_2)}$ is entry-wise bounded by 2. 
The expectation $\E[f(V^{(\ell_1)},V^{(\ell_2)})]$ is 
\begin{eqnarray}
  \frac{\delta^2}{r\,d_2}\E\left[\Lnorm{(V^{(\ell_1)} - V^{(\ell_2)} )^T }^2 \right] =
  \frac{2\delta^2}{r\,d_2} \E\left[\Lnorm{(V^{(\ell_1)})^T}^2 \right] \leq 2 \delta^2 \;,
\end{eqnarray}
where we use equation \eqref{eq:enet_single_ub}. Applying McDiarmid's inequality with bounded difference $12\delta^2/(r d_2)$, we get that 
\begin{eqnarray}
  \prob{\, f(V^{(\ell_1)},V^{(\ell_2)}) \leq 2\delta^2 -t \, } &\leq& \exp\Big\{ - \frac{ t^2\,r\,d_2}{144 \,\delta^4 } \Big\} \;,
\end{eqnarray}
Since there are less than $(M')^2$ pairs of $(\ell_1,\ell_2)$, setting $t=\delta^2$ and applying the union bound gives 
\begin{eqnarray}
  \prob{ \min_{\ell_1,\ell_2\in[M']} \fnorm{\Theta^{(\ell_1)}-\Theta^{(\ell_2)}}^2 \geq \delta^2  } &\geq& 1- \exp\Big\{ - \frac{r\,d_2}{144} + 2\log M'\Big\} \geq \frac78\;, 
  \label{eq:mcdiarmid1}
\end{eqnarray}
where we used $M'=\exp\{rd_2/576\}$ and $d_2\geq607$. 

We are left to prove that $\Theta^{(\ell)}$'s are in 
$\Omega_{(8\delta/d_2)\sqrt{2\log d_2}}$ as defined in \eqref{eq:defgraphomega}. 
Since we removed the mean for each connected component, such that $\Theta^{(\ell)} g_i=0$, $\forall i \in [G]$ by construction, 
we only need to show that the maximum entry is bounded by $(8\delta/d_2)\sqrt{2\log d_2}$. 
We first prove an upper bound in \eqref{eq:lb_maxbound} for a fixed $\ell \in[M']$, 
 and use this to show that there exists a large enough subset of matrices satisfying this bound. 
From \eqref{eq:defpacking}, consider  $(UV^T)_{ij} =  \llangle u_i, v_j\rrangle$, where 
$u_i\in\reals^{r}$ is the first $r$ entries of a random vector drawn uniformly from the $d_2$-dimensional sphere, and 
$v_j\in\reals^r$ is drawn uniformly at random from $\{-1,+1\}^r$ with $\|v_j\| = \sqrt{r}$. 
Using Levy's theorem for concentration on the sphere \cite{Led01}, we have 
\begin{eqnarray}
  \prob{ |\llangle u_i,v_j \rrangle| \geq t } &\leq& 2 \exp \Big\{ -\frac{d_2 \,t^2}{8\,r} \Big\}\;. 
\end{eqnarray}
Notice that by the definition \eqref{eq:defpacking}, $\max_{i,j}|\Theta^{(\ell)}_{ij}| \leq (2\delta/\sqrt{r d_2}) \max_{i,j} |\llangle u_i ,v_j\rrangle|$. Setting $t=\sqrt{(32 r/d_2) \log d_2}$ and taking the union bound over all $d_1d_2$ indices, we get 
\begin{eqnarray} 
  \prob{\, \max_{i,j} |\Theta^{(\ell)}_{ij}|  \leq \frac{2\delta\sqrt{32 \log d_2} }{d_2}\,} &\geq& 1-2d_1d_2 \exp \Big\{ -4 \log d_2 \Big\} \;\geq \;\frac12 \;,
  \label{eq:lb_maxbound}
\end{eqnarray}
for a fixed $\ell\in[M']$. 
Consider the event that there exists a subset $S\subset[M']$ of cardinality $M=(1/4)M'$ with the same bound on maximum entry, then 
from \eqref{eq:lb_maxbound} we get 
\begin{align}
  \prob{ \exists S\subset [M']\text{ such that } \lnorm{\Theta^{(\ell)}}{\infty}\leq \frac{2\delta\sqrt{32\log d_2}}{d_2} \text{ for all }\ell\in S} &\geq& \sum_{m=M}^{M'} {M' \choose m} \Big(\frac12\Big)^m\;,
  \label{eq:unioncardinality}
\end{align}
which is larger than half for our choice of $M < M'/2$. 



\subsection{Proof of Lemma \ref{lem:graph_packing}}
Inspired from the construction in \cite{NW11}, we furnish the following probabilistic argument for the existence of the desired family. For the choice of $M = \lfloor e^{rd_1/256}\rfloor$, and for each $\ell \in [M]$, generate a rank-$r$ matrix $\Theta^{(\ell)} \in \reals^{d_1 \times d_2}$ as follows:
\begin{eqnarray}
  \Theta^{(\ell)} &=& \frac{\delta}{\sqrt{r d_1}} V^{(\ell)} \sqrt{\Lambda_r^\dagger} U_r^T\;, 
  \label{eq:graph_defpacking}
\end{eqnarray}
where the columns of $U_r\in\reals^{d_2\times r}$ are the top $r$ singular vectors of $L = U\Lambda U^\top$, $\Lambda_r$ is a diagonal matrix in $\reals^{r\times r}$ and its diagonal elements are the top $r$ singular values of $L$ corresponding to columns of $U_r$, $\dagger$ represents the Moore-Penrose pseudo inverse, and $V^{(\ell)}$ is a random matrix with each entry $V^{(\ell)}_{ij} \in\{-1,+1\}$ chosen independently and uniformly at random. First by definition, $\Lnorm{\Theta^{(\ell)}} = (\delta/ \sqrt{rd_1})\fnorm{V^{(\ell)}}  \leq \delta$, since $\fnorm{V^{(\ell)}}=\sqrt{rd_1}$. 

Define $f$ as $f(V^{(\ell_1)},V^{(\ell_2)}) \equiv\Lnorm{\Theta^{(\ell_1)}-\Theta^{(\ell_2)}}^2 = 
(\delta^2/(r d_1))\fnorm{V^{(\ell_1)} - V^{(\ell_2)} }^2$ 
which is a function of $2 r d_1$ i.i.d. random Rademacher variables.
Now we can apply McDiarmid's concentration inequality since $f$ is Lipschitz as folows. 
For all $(V^{(\ell_1)},V^{(\ell_2)})$ and $(\tV^{(\ell_1)},\tV^{(\ell_2)})$ that differ in only one variable, say 
$\tV^{(\ell_1)} = V^{(\ell_1)} + 2e_{ij}$, for some standard basis matrix $e_{ij}$, we have  
\begin{align}
  &\big|f(V^{(\ell_1)},V^{(\ell_2)})-f(\tV^{(\ell_1)},\tV^{(\ell_2)}) \big| \nonumber \\&=\;   \left| \; \frac{\delta^2}{r\, d_2} \fnorm{ V^{(\ell_1)} - V^{(\ell_2)} }^2 \,- \, \frac{\delta^2}{r\, d_2} \fnorm{ V^{(\ell_1)} - V^{(\ell_2)} + 2e_{ij} }^2 \; \right| \nonumber \\
  &=\; \left|\; \frac{\delta^2}{r\,d_2} \fnorm{2e_{ij}}^2 + \frac{\delta^2}{r\,d_2} \llangle  (V^{(\ell_1)} - V^{(\ell_2)})  , 2e_{ij} \rrangle \; \right| \nonumber \\
  &\leq \; \frac{4\,\delta^2}{r\,d_1 }  \;+\; \frac{\delta^2}{r\,d_1}   \lnorm{V^{(\ell_1)} - V^{(\ell_2)}}{\infty} \, \lnorm{2e_{ij} }{1} \nonumber \\ 
  &\leq \; \frac{8\,\delta^2}{r\,d_1 } \;,
\end{align}
where the penultimate step is true since $(V^{(\ell_1)} - V^{(\ell_2)})$ is entry-wise bounded by 2. 
The expectation $\E[f(V^{(\ell_1)},V^{(\ell_2)})]$ is 
\begin{eqnarray}
  \frac{\delta^2}{r\,d_1}\E\left[\fnorm{(V^{(\ell_1)} - V^{(\ell_2)} ) }^2 \right] &=& 
  \frac{2\delta^2}{r\,d_1} \E\left[\fnorm{V^{(\ell_1)}  }^2 \right]\nonumber \\ 
  &=& 2\,\delta^2\;. 
\end{eqnarray}
Now applying McDiarmid's inequality on the function $f$, we get that 
\begin{eqnarray}
  \prob{\, f(V^{(\ell_1)},V^{(\ell_2)}) \leq 2\delta^2 -t \, } &\leq& \exp\Big\{ - \frac{ t^2\,r\,d_1}{64 \,\delta^4 } \Big\} \;,
\end{eqnarray}
Setting $t=\delta^2$ and applying the union bound gives us,
\begin{eqnarray}
  \prob{ \min_{\ell_1,\ell_2\in[M]} \fnorm{\Theta^{(\ell_1)}-\Theta^{(\ell_2)}}^2 \geq \delta^2  } &\geq& 1- \exp\Big\{ - \frac{r\,d_1}{64} + 2\log M\Big\} > 0\;.
  \label{eq:graph_mcdiarmid1}
\end{eqnarray}
In the last step, we used $M=\lfloor \exp\{rd_1/\;256\} \rfloor$. At last we prove that $\Theta^{(\ell)}$'s are in $\Omega_{\delta\sqrt{ \trace{(L_r)^\dagger}}/{r d_1}}$ as defined in \eqref{eq:defgraphomega}. 
Since we know that $g_i$ belongs to the kernel of $L$ for all $i \in [G]$, $\Theta^{(\ell)} g=0$ by construction \eqref{eq: graph_def_group}. From \eqref{eq:graph_defpacking}, consider  $(V \sqrt{\Lambda_r^\dagger} U^\top_r)_{ij} =  \llangle v_i, \sqrt{\Lambda_r^\dagger}{(u_r)}_j\rrangle$, where ${(u_r)}_j\in\reals^{r}$ is the vector of $i$-th entries of the top $r$ singular vectors of $L$, and $v_i\in\reals^r$ is drawn uniformly at random from $\{-1,+1\}^r$. 

\begin{eqnarray}
  \Big|\llangle v_i, \sqrt{\Lambda_r^\dagger}{(u_r)}_j \rrangle \Big| \leq \lnorm{v_i}{\infty}\lnorm{ \sqrt{\Lambda_r^\dagger}{(u_r)}_j}{1} \leq  \sqrt{\trace{\Lambda_r^\dagger}} = \sqrt{\trace{(L_r)^\dagger}}\;. 
\end{eqnarray}
 The above inequality proves that $\lnorm{\Theta^{(\ell)}}{\infty}$ is upper bounded as desired.

\section{Proof of Theorem \ref{thm:kwise_ub}} 
\label{sec:kwise_ub_proof}

We first introduce some additional notations used in the proof. Recall that $\cL(\Theta)$
is the log likelihood function. Let $\nabla \cL(\Theta) \in \reals^{d_1 \times d_2}$
denote its gradient such that $ \nabla_{ij} \cL(\Theta) = \frac{\partial \cL(\Theta)}{\partial \Theta_{ij} }$.
Let $\nabla^2\cL(\Theta) \in \reals^{d_1d_2 \times d_1 d_2}$ denote its Hessian matrix such that
$ \nabla^2_{ij, i'j'} \cL(\Theta) = \frac{\partial^2 \cL(\Theta)}{\partial \Theta_{ij} \partial \Theta_{i'j'} }$.
By the definition of $\cL(\Theta)$ in \eqref{eq:defkwiseL}, we have
\begin{eqnarray}
  \nabla \cL(\Theta^\ast) &=& -\frac{1}{k\,d_1} \sum_{i=1}^{d_1}  \sum_{\ell=1}^{k} e_{i} (e_{v_{i,\ell}}  - p_{i,\ell } )^\top\;,
\end{eqnarray}
where $p_{i,\ell}$ denotes the conditional choice probability at $\ell$-th position. 
Precisely, $p_{i,\ell} = \sum_{j \in S_{i,\ell} }  p_{j|(i,\ell)} e_j$ where 
$p_{j|(i,\ell)}$ is the probability that item $j$ is chosen at $\ell$-th position from the top by the user $i$ 
conditioned on the top $\ell-1$ choices such that   
$p_{j|(i,\ell)} \equiv \prob{v_{i,\ell}  = j | v_{i,1},\ldots, v_{i,\ell-1},S_{i} } = 
e^{\Theta^*_{ij}}/(\sum_{j'\in S_{i,\ell}} e^{\Theta_{ij'}})$ and 
$S_{i,\ell} \equiv S_i\setminus\{v_{i,1},\ldots,v_{i,\ell-1}\}$, 
where $S_i$ is the set of alternatives presented to the $i$-th user and 
 $v_{i,\ell}$ is the item ranked at the $\ell$-th position by the user $i$.
Notice that for $i \neq i'$, $ \frac{\partial^2 \cL(\Theta)}{\partial \Theta_{ij} \partial \Theta_{i'j'} } =0$
and the Hessian is 
\begin{eqnarray}
\frac{\partial^2 \cL(\Theta)}{\partial \Theta_{ij} \partial \Theta_{ij'} }  &= &
     \frac{1}{k\,d_1} \sum_{\ell=1}^{k} \ind\big(j \in S_{i, \ell} \big) \frac{\partial p_{j|(i,\ell)} }{\partial \Theta_{ij'}} \nonumber\\
&=&  \frac{1}{k\,d_1} \sum_{\ell=1}^{k} \ind\big(j, j' \in S_{i, \ell} \big) \left( p_{j|(i,\ell)} \ind(j=j')  - p_{j|(i,\ell)} p_{j'|(i,\ell)}  \right).
\end{eqnarray}
This Hessian matrix is a block-diagonal matrix $\nabla^2\cL(\Theta)= \diag(H^{(1)}(\Theta), \ldots, H^{(d_1)}(\Theta) )$ with
\begin{align}
H^{(i)}(\Theta) = \frac{1}{k\,d_1} \sum_{\ell=1}^k
 \big(\diag(p_{i,\ell} ) - p_{i,\ell }  p^\top_{i, \ell} \big) \; . \label{eq:DefHessian}
\end{align}

Let $\Delta=\Theta^*-\widehat{\Theta}$ where $\widehat{\Theta}$ is the optimal solution of the convex program in \eqref{eq:kwiseopt}.
We first introduce three key technical lemmas. 
The first lemma follows from Lemma 1 of \cite{NW11},  and
shows that $\Delta$ is approximately low-rank. 
\begin{lemma}
\label{lmm:kwise_deltabound}
If $\lambda\geq 2 \lnorm{\nabla\calL (\Theta^*)}{2}$, then we have
\begin{eqnarray}
  \nucnorm{\Delta} &\le& 4\sqrt{2 r} \fnorm{\Delta} + 4 \sum^{\min\{d_1,d_2\}}_{j = r+1}\sigma_j(\Theta^*) \;, \label{eq:kwise_deltabound}
\end{eqnarray}
for all $r \in[\min\{d_1,d_2\}]$.
\end{lemma}
Proof of the above lemma is omitted because of its similarity to that of Lemma \ref{lmm:graph_deltabound}.
The following lemma provides a bound on the gradient
using the concentration in measure of sum of independent random matrices \cite{Jo11}. 
\begin{lemma}\label{lmm:kwise_gradient2}
  For any positive constant $c \ge 1$ and $k \leq (1/e)\; d_2 (4\log d_2 + \log d_1)$, with probability at least $1-2 d^{-c} - d_2^{-3}$,
  \begin{align}
    &\lnorm{\nabla \calL(\Theta^\ast)}{2} \nonumber \\ &\leq
    \sqrt{\frac{ 4(1+c)  \, \log d} { k\,d_1^2  }} \max \left\{ \sqrt{d_1/d_2},\; e^{2\bb} \sqrt{4(1+c) \log (d)}(8\log d_2 + 2 \log d_1)\log k \right\}
    \;.
  \end{align}
\end{lemma}
Since we are typically interested in the regime where 
the number of samples is much smaller than the dimension $d_1\times d_2$ of the problem, 
the Hessian is typically not positive definite. However, when we restrict our attention to the vectorized $\Delta$ with 
relatively small nuclear norm, then we can prove restricted strong convexity, which gives the following bound. 
\begin{lemma}[{\bf Restricted Strong Convexity for collaborative ranking}]
\label{lmm:kwise_hessian2}
Fix any $\Theta \in \Omega _\bb$ and assume $ 24 \,\leq k \leq \min\{ d_1^2, (d_1^2+d_2^2)/(2d_1)\} \log d $.
  Under the random sampling model of the alternatives $\{j_{i\ell}\}_{i\in[d_1],\ell\in[k]}$ and
  the random outcome of the comparisons described in section \ref{sec:intro},
  with probability larger than $1-2d^{-2^{18}}$,
  \begin{eqnarray}
    {\rm Vec}(\Delta)^\top \,\nabla^2\calL(\Theta)\, {\rm Vec}(\Delta)  &\geq& \frac{e^{-4\bb}}{24\,d_1d_2} \fnorm{\Delta}^2\;,
    \label{eq:kwise_hessian2}
  \end{eqnarray}
  for all $\Delta$ in $\calA$ where
  \begin{align}
    \calA = \Big\{ \Delta \in \reals^{d_1\times d_2} \,\big|\, \lnorm{\Delta}{\infty} \leq 2\bb\,, \, \sum_{j\in[d_2]}\Delta_{ij} = 0\text{ for all }i\in[d_1] \text{ and } \fnorm{\Delta}^2 \geq \mu \nucnorm{\Delta}
    \Big\} \;. \label{eq:kwise_defA}
  \end{align}
  with 
  \begin{eqnarray}
    \mu &\equiv& 2^{10}\,e^{2\bb}\, \bb\, d_2 \sqrt{\frac{d_1 \, \log d}{k\,\min\{d_1,d_2\} }} \;.
  \label{eq:defmu}
  \end{eqnarray}
\end{lemma}

Building on these lemmas, the proof of Theorem \ref{thm:kwise_ub} is divided into the following two cases.
In both cases, we will show that 
\begin{align}
  \fnorm{\Delta}^2  \;\leq\;  72 \, e^{4\bb}  c_0 \lambda_0  \,d_1d_2\,  \nucnorm{\Delta} \;, 
  \label{eq:errorfrobeniusbound}
\end{align}
 with high probability. Applying Lemma \ref{lmm:kwise_deltabound} 
proves the desired theorem. We are left to show Eq. \eqref{eq:errorfrobeniusbound} holds.

\bigskip
\noindent{\bf Case 1: Suppose $\fnorm{\Delta}^2 \geq  \mu \,\nucnorm{\Delta} $.}
With $\Delta=\Theta^* - \hTheta$, the Taylor expansion yields
\begin{align}
\calL(\widehat{\Theta})=\calL(\Theta^\ast) -  {\llangle \nabla \calL(\Theta^\ast), \Delta \rrangle} + \frac{1}{2} {\rm Vec}(\Delta) \nabla^2\calL(\Theta) {\rm Vec}^\top (\Delta),
\label{eq:LikelihoodTaylor}
\end{align}
where $\Theta=a \widehat{\Theta} + (1-a) \Theta^\ast$ for some $a \in [0,1]$.
It follows from Lemma \ref{lmm:kwise_hessian2} that with probability at least $1-2d^{-2^{18}}$,
\begin{align*}
   \calL(\widehat{\Theta}) -\calL(\Theta^\ast) & \;\ge\;  -\llangle \nabla \calL(\Theta^\ast), \Delta \rrangle
  + \frac{ e^{-4\bb}}{48\,d_1\,d_2} \fnorm{\Delta}^2\\
  &\; \ge\; - \lnorm{\nabla\calL(\Theta^\ast)}{2} \nucnorm{\Delta}+ \frac{  e^{-4\bb}}{48\,d_1\,d_2} \fnorm{\Delta}^2\;.
\end{align*}
From the definition of $\widehat{\Theta}$ as an optimal solution of the minimization, we have
\begin{align*}
\calL(\widehat{\Theta})-  \calL(\Theta^\ast)  \;\le \; \lambda \left(  \nucnorm{\Theta^\ast} - \nucnorm{\widehat{\Theta}} \right) \;\le\; \lambda \nucnorm{\Delta}\;.
\end{align*}
By the assumption, we choose $\lambda\geq 480\lambda_0$. 
 In view of Lemma \ref{lmm:kwise_gradient2}, this implies that 
  $\lambda \geq 2\lnorm{\nabla\calL (\Theta^*)}{2} $ 
  with probability at least $1-2d^{-3}$. 
   It follows that with probability at least $1-2d^{-3}- 2d^{-2^{18}}$,
\begin{align*}
\frac{ e^{-4\bb}}{48d_1 d_2 } \fnorm{\Delta}^2 \;  \leq \;  \big(\lambda + \lnorm{\nabla\calL(\Theta^*)}{2}\big)\,  \nucnorm{\Delta} \;\leq\; \frac{3 \lambda}{2}  \nucnorm{\Delta} \;.
\end{align*}
By our assumption on $\lambda \leq c_0\lambda_0$, this proves the desired bound in Eq. \eqref{eq:errorfrobeniusbound}

\noindent{\bf Case 2:  Suppose $\fnorm{\Delta}^2 \leq  \mu \,\nucnorm{\Delta} $.}  
By the definition of $\mu$ and the fact that $c_0 \ge {480}$, it follows that
$\mu \le 72 \, e^{4\bb}  c_0 \lambda_0  \,d_1d_2$, and we get the same bound as in Eq. \eqref{eq:errorfrobeniusbound}.

\subsection{Proof of Lemma \ref{lmm:kwise_gradient2}}
Define $X_i=  -e_{i} \sum_{\ell=1}^{k} (e_{v_{i,\ell} } - p_{i,\ell } )^\top $ such that $\nabla \calL(\Theta^\ast)= \frac{1}{k\, d_1} \sum_{i=1}^{d_1} X_i$, which is a sum of $d_1$ independent random matrices. 
%
Although $\lnorm{X_i}{2}$ can be as large as  $O(k)$, 
this occurs with very low probability. 
We make this precise in the following lemma and 
focus on the case where $\lnorm{X_i}{2}= O(\sqrt{k})$ for all $i\in[d_1]$. 
\begin{lemma} \label{lem:balls_bins}
  For a fixed $i\in[d_1]$ and $j\in[d_2]$, 
  if $ k \leq (1/e) \, d_2 \, (4\log d_2+\log d_1)$, then the number of times the item $j$ is observed by the user $i$ is at most $8 (\log d_2) +2(\log d_1)$ with probability 
  larger than $1-1/(d_2^4d_1)$. 
\end{lemma} 
Proof is given in the end of this Section. 
Applying union bound over the $d_1$ items and $d_2$ users, we have the multiplicity in sampling for any item for all users is bounded by 
$8 (\log d_2) +2(\log d_1)$ with probability at least $1-d_2^{-3}$. 
We denote this event by $\cA$ and let $\indc{\cA}$ be the indicator function that all the multiplicities in sampling are bounded.  
We first upper bound $\lnorm{\left(\sum_iX_i\right) \indc{\cA}}{2}$ using the Matrix Bernstein inequality \cite{Jo11}. 
\begin{align}
  \lnorm{X_i\indc{\cA}}{2}  &= \Big\|\indc{\cA} \sum_{\ell=1}^{k} \big(e_{v_{i,\ell} } - p_{i,\ell } \big) \Big\| \nonumber \\
  &\overset{(a)}{\leq}   \Big\| \indc{\cA} \sum_{\ell=1}^k e_{v_{i,l}}  \Big\|  + \Big\| \indc{\cA} \sum_{\ell=1}^{k} p_{i,\ell }  \Big\| \nonumber \\
  &\overset{(b)}{\leq} (8 (\log d_2) +2(\log d_1)) \sqrt{\min\{k,d_2\}}  \bigg( 1 +  \; \left( \sum_{\ell = 1}^k \frac{e^{2\bb}}{\ell}\right) \bigg) \nonumber \\
  &\overset{(c)}{\leq} \sqrt{k} (8 (\log d_2) +2(\log d_1)) \big( 1 + 2e^{2\bb} \; \log k\big) \nonumber\\
  &\leq 3\sqrt{k} (8 (\log d_2) +2(\log d_1)) e^{2\bb} \; \log k \;,
\end{align}
where $(a)$ is by triangle inequality, $(b)$ is because under the given event $\cA$ each term in $\sum_\ell e_{v_{i,\ell}}$ and $\sum_l p_{i,\ell}$ are upper bounded by $\log d_2$ and $\left(\sum_{\ell = 1}^k \frac{e^{2\bb}}{\ell} \right) \log d_2$ respectively and because there can be at most $\min\{d_2,k\}$ non-zero entries in the two vectors $\sum_{\ell} e_{v_{i, \ell}}$ and $\sum_{\ell} p_{i,\ell}$ and, $(c)$ is due to the fact that $k$-th harmonic number $\sum_{\ell=1}^k \frac{1}{\ell}$ is upper bounded by $\log k$.
We also have,
\begin{align}
\lnorm{\sum_i \expect{X_i X_i^\top \indc{\cA} }}{2} & \leq \lnorm{\sum_i \expect{X_i X_i^\top}}{2} \nonumber \\&\leq \lnorm{\sum_{i=1}^{d_1} e_{i} e_{i}^\top \expect{\sum_{\ell, \ell'=1}^k\left(e_{v_{i,\ell}} - p_{i,\ell}\right)^\top\left(e_{v_{i,\ell'}} - p_{i,\ell'}\right)}}{2} \nonumber \\
& = \lnorm{\sum_{i=1}^{d_1} e_{i} e_{i}^\top \expect{\sum_{\ell = 1}^k\left(e_{v_{i,\ell}} - p_{i,\ell}\right)^\top\left(e_{v_{i,\ell}} - p_{i,\ell}\right)}}{2} \nonumber \\
& = \lnorm{\sum_{i=1}^{d_1} e_{i} e_{i}^\top \expect{\sum_{\ell=1}^k e_{v_{i,\ell}}^\top e_{v_{i,\ell}} - p_{i,\ell}^\top p_{i,\ell}}}{2} \nonumber \\
& \leq \lnorm{\sum_{i=1}^{d_1} e_{i} e_{i}^\top \expect{\sum_{\ell=1}^k e_{v_{i,\ell}}^\top e_{v_{i,\ell}}}}{2} \nonumber \\
& = \,k \lnorm{{\mathbf I}_{d_1\times d_1}}{2} = k,
\end{align}
and
\begin{align}
\lnorm{\sum_{i=1}^{d_1} \expect{X_i^\top X_i \indc{\cA} }}{2} &\leq \lnorm{\sum_{i=1}^{d_1} \expect{X_i^\top X_i}}{2}\nonumber \nonumber \\
&\leq \lnorm{\sum_{i=1}^{d_1} \expect{ \sum_{\ell,\ell'=1}^{k} (e_{v_{i,\ell}} - p_{i,\ell }) (e_{v_{i,\ell'} } - p_{i, \ell' })^\top  }}{2} \nonumber \\
& = \lnorm{\sum_{i=1}^{d_1} \expect{ \sum_{\ell=1}^{k} (e_{v_{i,\ell}} - p_{i,\ell }) (e_{v_{i,\ell} } - p_{i, \ell })^\top  }}{2} \\
& = \lnorm{\sum_{i=1}^{d_1} \expect{ \sum_{\ell=1}^{k} e_{v_{i,\ell} }e_{v_{i,\ell} }^\top - p_{i,\ell } p^\top_{i,\ell} }}{2} \nonumber \\
& \leq \lnorm{\sum_{i=1}^{d_1} \expect{  \sum_{\ell=1}^{k} e_{v_{i,\ell} } e_{v_{i,\ell} }^\top }}{2} \nonumber \\
& = \lnorm{\sum_{i=1}^{d_1} \frac{k}{d_2} {\mathbf I}_{d_2 \times d_2}}{2} = \frac{kd_1}{d_2}\; .
\end{align}
By matrix Bernstein inequality \cite{Jo11},
\begin{align*}
  &{\mathbb P} \Big( \lnorm{\nabla \calL(\Theta^\ast) \indc{\cA}}{2} > t \Big) \nonumber\\ &\leq (d_1+d_2)
  \exp \Big( \frac{- k^2\,d_1^2 \, t^2/2}{ (d_1k/\min\{d_2,d_1\} )+ (3e^{2\bb} k^{3/2} d_1 (8 (\log d_2) +2(\log d_1))\log k\; t/3)} \Big)\;,
\end{align*}
which gives the tail probability of $2d^{-c}$ for the choice of
\begin{eqnarray*}
  t&=& \max \left\{  \sqrt{\frac{4(1+c) \, \log d}{ k\,d_1\, \min\{d_2,d_1\} }} \,,\,  \frac{4(1+c)e^{2\bb} \log (d)\; (8 (\log d_2) +2(\log d_1))\log k}{ k^{1/2} \,d_1}\right\} \\
& =& \frac{  \sqrt{ 4(1+c)  \, \log d}  } { k^{1/2}\, d_1  } \max \left\{ \sqrt{d_1/d_2}   \,,\,\;e^{2\bb} \sqrt{4(1+c) \log(d)}\; (8 (\log d_2) +2(\log d_1)) \log k \right\}.
\end{eqnarray*}
Now with a high probability of $1 - \frac{2}{d^c} - \frac{1}{d_2^{3}}$ the desired bound is true.

\subsection{Proof of Lemma \ref{lem:balls_bins}}
In a classical balls-in-bins setting, we consider $k$ as the number of balls and $d_2$ as the number of bins. 
We can consider  the number of balls in a particular bin as the number of times the user $i$ observes item $j$.
 Let the event that this number is at least $\delta$ be denoted by the event $A^j_{\delta}$.
Then, $
\prob{A^j_\delta} \leq  {k \choose \delta} \frac{1}{d_2^\delta} \;
 \overset{}{\leq } \left( \frac{k e}{d_2 \delta}\right)^\delta.
$ 
Using the fact that $(1/x)^x\leq a$ for any $x\geq (2 \log (1/a))/(\log \log (1/a))$, 
we let $x=d_2 \delta/(ke)$ to get   
\begin{eqnarray*}
  \left( \frac{k e}{d_2 \delta}\right)^\delta \leq a^{\frac{ke}{d_2}}\;,
\end{eqnarray*}
for $ \delta \geq (ke/d_2)(2 \log (1/a))/(\log \log (1/a))$. 
Choosing $a=(1/d_2^4d_1)^{d_2/ke}$, we have 
$\prob{A^j_\delta}\leq 1/(d_1d_2^4)$, for 
a choice of  $\delta = 2\; \log (d_2^4d_1 ) \geq 2 \log(d_2^4d_1) / (\log (  (d_2/ke)\log(d_2^4 d_1) ))$. 



\subsection{Proof of Lemma \ref{lmm:kwise_hessian2}}

Recall that the Hessian matrix is a block-diagonal matrix with the $i$-th block
$H^{(i)}(\Theta)$ given by \eqref{eq:DefHessian}. 
We use the following remark from \cite{HOX14} to bound the Hessian. 
\begin{remark}   
   \cite[Claim 1]{HOX14}
  Given $\theta \in   \reals^r,$  let 
  $p$ be the column probability vector with $p_i=e^{\theta_i}/(e^{\theta_1}+\cdots + e^{\theta_\rho})$ for each $i\in[\rho]$  
  and for any positive integer $\rho$. 
  If  $| \theta_i | \leq \bb,$ for all $i \in [\rho]$,  then 
  $$e^{2\bb} \Big(\diag{(p)} - pp^T\Big) \;\succeq\;  \frac{1}{\rho} \diag(\ones) - \frac{1}{\rho^2}\ones\ones^\top  \;.$$
  \label{rem:hess}
\end{remark}
By letting $\ones_{S_{i,\ell}} = \sum_{j \in S_{i,\ell} } e_j$ and applying the above claim, we have
\begin{align*}
  e^{2\alpha} H^{(i)}(\Theta) & \succeq \frac{1}{k\, d_1} \sum_{\ell=1}^k  \left(  \frac{1}{k-\ell+1 }\diag(\ones_{S_{i,\ell}}) - \frac{1}{(k-\ell+1)^2 }  \ones_{S_{i,\ell}}\ones^\top_{S_{i,\ell}} \right) \\
 &= \frac{1}{2\,k\,d_1}  \sum_{\ell=1}^k \frac{1}{(k-\ell+1)^2}\sum_{j, j' \in S_{i,\ell}} (e_j-e_{j'})(e_j-e_{j'})^\top \\
 & \succeq \frac{1}{2\,k^3\, d_1} \sum_{\ell=1}^k \sum_{j, j' \in S_{i,\ell}} (e_j-e_{j'})(e_j-e_{j'})^\top.
\end{align*}
Hence,
\begin{align*}
 {\rm Vec}(\Delta) \nabla^2\calL(\Theta) {\rm Vec}^\top (\Delta)  &=  \sum_{i=1}^{d_1} (\Delta^\top e_i)^\top H^{(i)} (\Theta) (\Delta^\top e_i ) \\
 & \ge   \frac{e^{-2\alpha}}{2\,k^3\,d_1 } \sum_{i=1}^{d_1} \sum_{\ell=1}^k \sum_{j,j' \in S_{i,\ell} } \lnorm{e_i^\top \Delta (e_j - e_{j'} )}{2}^2.
\end{align*}
By changing the order of the summation, we get that
\begin{align*}
\sum_{\ell=1}^k \sum_{j,j' \in S_{i,\ell} } \lnorm{e_i^\top \Delta (e_j - e_{j'} )}{2}^2 &= \sum_{ \ell, \ell'=1}^{k} \Iprod{\Delta}{e_{i, j_{i,\ell}} -e_{i, j_{i,\ell'} } }^2 \sum_{\ell''=1}^k \ind\big(\,\sigma_i(j_{i,\ell^{''}}) \nonumber \\ &\le \min \{\sigma_i(j_{i,\ell}), \sigma_i(j_{i,\ell'} )   \}\,\big).
\end{align*}
Define 
\begin{eqnarray}
  \label{eq:defchi}
  \chi_{i,\ell,\ell',\ell^{''}} &\equiv& \ind\big(\,\sigma_i(j_{i,\ell^{''}}) \le \min \{\sigma_i(j_{i,\ell}), \sigma_i(j_{i,\ell'} )   \}\,\big)\;,
\end{eqnarray}
 and let 
\begin{align*}
  H(\Delta ) \; \equiv \;  \frac{e^{-2\alpha}}{2\,k^3\,d_1 } \sum_{i=1}^{d_1} \sum_{ \ell, \ell'=1}^{k} \Iprod{\Delta}{e_{i, j_{i,\ell}} -e_{i, j_{i,\ell'} } }^2 \sum_{\ell''=1}^k \chi_{i,\ell,\ell',\ell^{''}}  .
\end{align*}
Then we have ${\rm Vec}^\top(\Delta) \nabla^2\calL(\Theta) {\rm Vec} (\Delta) \ge H(\Delta)$. To prove the theorem, it suffices to bound $H(\Delta)$ from the below. 
First, we prove a lower bound on  the expectation $\E[H(\Delta)]$.
Notice that for $\ell\neq \ell'$, the conditional expectation of $\chi_{i,\ell,\ell',\ell''}$'s, given the set of alternatives presented to user $i$ is 
\begin{align*}
  {\mathbb E} \Big[ \sum_{\ell''=1}^k \chi_{i,\ell,\ell',\ell^{''}} \,\big| \, j_{i,1}, \ldots, j_{i,k}\Big]  &= 1+ \sum_{\ell'' \neq \ell, \ell'} \frac{\exp( \theta_{i,j_{i,\ell''}} ) } { \exp( \theta_{i,j_{i,\ell''}} ) +\exp( \theta_{i,j_{i,\ell'}} ) +\exp( \theta_{i,j_{i,\ell}} )}  \\
&\ge 1+ \frac{k-2}{1+2e^{2\bb}} \ge \frac{k}{3 e^{2\bb}}.
\end{align*}
Then,
\begin{eqnarray}
  \E[H(\Delta)] &=&
  \frac{e^{-2\bb}}{2\,k^3\,d_1}  \sum_{i,\ell,\ell'}
{\mathbb E}\Big[\llangle \Delta, e_{i,j_{i,\ell}}-e_{i,j_{i,\ell'}}\rrangle^2  {\mathbb E} \big[ \sum_{\ell''=1}^{k} \chi_{i,\ell,\ell',\ell^{''}} \,\big| \, j_{i,1}, \ldots, j_{i,k} \big] \Big]\nonumber\\
&\geq& \frac{e^{-4\bb}}{6\,k^2\,d_1 }  \sum_{i=1}^{d_1} \sum_{\ell,\ell' \in [k]} \expect{
\llangle \Delta, e_{i,j_{i,\ell}}-e_{i,j_{i,\ell'}}\rrangle^2}   \nonumber\\
&=&  \frac{e^{-4\bb}}{6\,k^2\,d_1 }  \sum_{i=1}^{d_1}
  \sum_{\ell \neq \ell'\in [k]}  \left( \frac{2}{d_2}\sum_{j=1}^{d_2}\Delta_{ij}^2 - \frac{2}{d_2^2}\sum_{j,j'=1}^{d_2} \Delta_{ij}\Delta_{ij'}\right)\nonumber\\
&=& \frac{e^{-4\bb}(k-1)}{3\,k\,d_1\,d_2} \fnorm{\Delta}^2\;,
\label{eq:kwise_hessexp}
\end{eqnarray}
where the last equality holds because  $\sum_{j\in[d_2]}\Delta_{ij}=0$
for $\Delta \in \Omega_{2\bb}$ and for all $i\in[d_1]$.

We are left to prove that $H(\Delta)$ cannot deviate from its mean too much.  
Suppose there exists a
$\Delta \in \calA$ such that Eq. \eqref{eq:kwise_hessian2} is violated, i.e.
$H(\Delta) < (e^{-4 \bb}/(24 \,d_1d_2)) \fnorm{\Delta}^2$. 
We will show this happens with a small probability. 
From Eq. \eqref{eq:kwise_hessexp}, we  get that  for $k\geq 24$, 
\begin{eqnarray}
  \E[H(\Delta)] - H(\Delta) &\geq& \frac{(7k-8)}{24k} \frac{ e^{-4\bb} }{d_1\, d_2} \fnorm{\Delta}^2 \nonumber\\
    &\geq&\frac{ (20/3)\, e^{-4\bb}  }{24 \,d_1 d_2} \fnorm{\Delta}^2 \;. \label{eq:kwise_peeling1}
\end{eqnarray}
We use a peeling argument as in \cite[Lemma 3]{NW11}, \cite{Van00}  
to upper bound the probability that Eq. \eqref{eq:kwise_peeling1} is true. 
We first construct the following family of subsets to cover $\calA$ such that 
$\calA \subseteq \bigcup_{\ell=1}^\infty \calS_\ell$. 
Recall \\ $\mu=2^{10}e^{2\bb} \bb d_2 \sqrt{(d_1\log d)/(k\min\{d_1,d_2\})}$, define in \eqref{eq:defmu}. 
Notice that since for any $\Delta\in \calA$, $\fnorm{\Delta}^2 \ge \mu  \nucnorm{\Delta} \ge \mu \fnorm{\Delta}$, it follows that
$\fnorm{\Delta} \ge \mu$.
Then, we can cover $\calA$ with the family of sets 
\begin{align*}
  \calS_\ell =\Big\{ \Delta\in\reals^{d_1\times d_2} \,\Big|\, &\lnorm{\Delta}{\infty} \leq 2\bb \,,\, \beta^{\ell-1}\mu \leq \fnorm{\Delta} \leq \beta^\ell \mu \,,\,  \nonumber\\ &\sum_{j\in[d_2]} \Delta_{ij}=0 \text{ for all }i\in[d_1], \text{ and }  \nucnorm{\Delta} \leq  \beta^{2\ell}\mu   \Big\} \;,
\end{align*}
where $\beta=\sqrt{10/9}$ and for $\ell\in\{1,2,3,\ldots \}$.
This implies that 
when there exists a $\Delta\in\calA$ such that 
\eqref{eq:kwise_peeling1} holds, then there exists an $\ell\in\Z_+$ such that $\Delta\in \calS_\ell$ and 
\begin{eqnarray}
    \E[H(\Delta)] - H(\Delta) &\geq & \frac{ (20/3)\, e^{-4\bb}  }{24\, d_1 d_2} \beta^{2(\ell-1)} \mu^2 \nonumber\\
    &\geq & \frac{e^{-4\bb} }{4 \,d_1 d_2} \beta^{2 \ell} \mu^2 \;.
    \label{eq:kwise_peeling2} 
\end{eqnarray}

Applying the union bound over $\ell\in\Z_+$, we get from \eqref{eq:kwise_peeling1} and \eqref{eq:kwise_peeling2} that 
\begin{align}
  &\prob{\exists \Delta \in \calA \;,\; H(\Delta) < \frac{ e^{-4\bb} } { 24 \,d_1 d_2 } \fnorm{\Delta}^2 }\; \nonumber \\ &\leq\;
    \sum_{\ell=1}^\infty \prob{ \sup_{\Delta\in\calS_\ell} \big(\; \E[ H(\Delta)] - H(\Delta) \;\big) > \frac{e^{-4 \bb} }{ 4 \,d_1 d_2}(\beta^\ell \mu)^2 } \nonumber\\
    &\leq\;\; \sum_{\ell=1}^\infty \prob{ \sup_{\Delta\in\calB(\beta^\ell\mu)} \big(\; \E[ H(\Delta)] - H(\Delta) \;\big) > \frac{e^{-4 \bb}  }{ 4 \,d_1 d_2}(\beta^\ell \mu)^2 }
    \;, \label{eq:kwise_peeling3}
\end{align}
where we define a new set $\calB(D)$ such that $\calS_\ell \subseteq \calB(\beta^\ell \mu)$: 
\begin{align}
  \calB(D) = \big\{\, \Delta \in \reals^{d_1\times d_2} \,\big|\, &\|\Delta\|_\infty \leq 2\bb, \fnorm{\Delta}\leq D, \nonumber\\ &\sum_{j\in[d_2]} \Delta_{ij}=0 \text{ for all }i\in[d_1],  \mu \nucnorm{\Delta} \leq D^2 \,
      \Big\} \;. \label{eq:defcB}
\end{align}
The following key lemma provides the upper bound on this probability.  
\begin{lemma}
For $(16\min\{d_1,d_2\} \log d)/(3d_1)\leq k\leq d_1^2 \log d$, 
\begin{eqnarray}
  \prob{\sup_{\Delta\in\calB(D)} \Big(\; \E[ H(\Delta)] - H(\Delta) \;\Big) \geq \frac{e^{-4\bb}}{4d_1d_2} D^2 } &\leq& \exp\Big\{ - \frac{e^{-4\bb} \,k\,D^4}{2^{19} \bb^4 d_1 d_2^2 } \Big\}\;.
  \label{eq:kwise_hessian3}
\end{eqnarray}
\label{lmm:kwise_hessian3}
\end{lemma}

Let $\eta= \exp\left(-\frac{e^{-4\bb}4k(\beta-1.002) \mu^4}{2^{19} \bb^4d_1 d_2^2} \right)$. 
Applying the tail bound to \eqref{eq:kwise_peeling3}, we get 
\begin{eqnarray*}
  \prob{\exists \Delta \in \calA \;,\; H(\Delta) < \frac{ e^{-4\bb} } { 24 \,d_1 d_2 } \fnorm{\Delta}^2 } 
  &\leq & \sum_{\ell=1}^\infty \exp \Big\{-\frac{e^{-4\bb}k(\beta^\ell \mu)^4 }{2^{19} \bb^4 d_1 d_2^2}\Big\} \\
  & \overset{(a)}{\leq} &\sum_{\ell=1}^\infty \exp\Big\{-\frac{e^{-4 \bb}4k\ell (\beta-1.002) \mu^4 }{2^{19} \bb^4d_1 d_2^2}\Big\} \\
  &\leq & \frac{\eta}{1-\eta},
\end{eqnarray*}
where $(a)$ holds because $\beta^{x} \geq x \log\beta \ge x(\beta-1.002)$ for the choice of $\beta=\sqrt{10/9}$.
By the definition of $\mu$,
\begin{align*}
\eta \; = \; \exp\Big\{ - \frac{  2^{23}\,e^{4 \bb} d_2^2 d_1 (\log d)^2 (\beta-1.002) }{k (\min\{d_1,d_2\})^2}  \Big\} \;  \le\;   \exp \{ -\,2^{18}\,\log d\} \;,
\end{align*}
where the last inequality follows from the assumption that 
$k\leq \max\{d_1,d_2^2/d_1\}\log d =  (d_2^2 d_1 \log d) / (\min\{d_1,d_2\})^2 $, 
and $\beta - 1.002\geq 2^{-5}$. 
Since for $d \ge 2$, $\exp\{-2^{18}\log d\} \leq 1/2$ and thus $\eta \le 1/2$, the lemma follows by assembling the last two
displayed inequalities.

\subsection{Proof of Lemma \ref{lmm:kwise_hessian3}}
\label{sec:kwise_hessian3_proof}

Recall that
\begin{align*}
H(\Delta ) = \frac{e^{-2\alpha}}{2\,k^3\,d_1 } \sum_{i=1}^{d_1} \sum_{ \ell, \ell'=1}^{k} \Iprod{\Delta}{e_{i, j_{i,\ell}} -e_{i, j_{i,\ell'} } }^2 \sum_{\ell''=1}^k \chi_{i,\ell,\ell',\ell^{''}}  \;,
\end{align*}
with $\chi_{i,\ell,\ell',\ell^{''}} = \indc{\sigma_i(j_{i,\ell^{''}}) \le \min \{\sigma_i(j_{i,\ell}), \sigma_i(j_{i,\ell'} )   \}}$.
Let $Z = \sup_{\Delta\in\calB(D)}  \E[H(\Delta)] - H(\Delta)$ be the worst-case random deviation of $H(\Delta)$ form its mean.
We prove an upper bound on $Z$ by showing that
$Z - \E[Z] \leq e^{-4 \bb} D^2/(64 d_1d_2) $ with high probability, and $\E[Z] \leq 9 e^{-4 \bb} D^2/(40d_1d_2)$.
This proves the desired claim in Lemma \ref{lmm:kwise_hessian3}.

To prove the concentration of $Z$,
we utilize the random utility model (RUM) theoretic  interpretation of the MNL model.
The random variable $Z$ depends on the random choice of alternatives $\{j_{i,\ell}\}_{i\in[d_1],\ell\in[k]}$
and the random $k$-wise ranking outcomes $\{\sigma_i\}_{i\in[d_1]}$.
The random utility theory, pioneered by  \cite{Thu27,Mar60,Luce59}, 
tells us that the $k$-wise ranking from the MNL model
has the same distribution as first drawing independent
(unobserved) utilities $u_{i, \ell}$'s of the item $j_{i,\ell}$ for user $i$
according to the standard Gumbel Cumulative Distribution Function (CDF) $F(c-\Theta_{i, j_{i,\ell}})$ with $F(c)=e^{-e^{-c} } $,
and then ranking the $k$ items for user $i$ according to their respective utilities.
Given this definition of the MNL model, we have $\chi_{i,\ell,\ell',\ell^{''}} = \indc{u_{i,\ell^{''} }  \ge \max \{ u_{i,\ell } , u_{i,\ell'}    \} }$.
Thus $Z$ is a function of independent choices of the items and their (unobserved) utilities, i.e. $Z=f(\{(j_{i,\ell},u_{i,\ell} ) \}_{i\in[d_1],\ell\in[k]})$.
Let $x_{i,\ell} = (j_{i,\ell}, u_{i,\ell} )$ and write $H(\Delta)$ as $H(\Delta, \{ x_{i,\ell} \}_{i\in[d_1],\ell\in[k]} )$.
This allows us to bound the difference and apply McDiarmid's tail bound. 
Note that for any  $i\in[d_1]$, $\ell\in[k]$,
$x_{1,1} ,\ldots, x_{d_1, k} $, and $x'_{i, \ell} $,
\begin{align*}
& \big|\,f \big( \,x_{1,1}, \ldots,x_{i,\ell},\ldots,x_{d_1, k} \,\big) - f\big(\,x_{1,1} ,\ldots, x'_{i,\ell} ,\ldots, x_{d_1, k} \,\big)\,\big|\\
& =\big| \sup_{\Delta \in \calB(D) } \left( \expect{H(\Delta)} - H(\Delta, x_{1,1}, \ldots,x_{i,\ell},\ldots,x_{d_1, k} ) \right) -\\ &\sup_{\Delta \in \calB(D) } \left( \expect{H(\Delta) } - H(\Delta, x_{1,1} ,\ldots, x'_{i,\ell} ,\ldots, x_{d_1, k}) \right) \big|  \\
& \le \sup_{\Delta\in \calB(D)  } \big| H(\Delta, x_{1,1}, \ldots,x_{i,\ell},\ldots,x_{d_1, k} )  -H(\Delta, x_{1,1} ,\ldots, x'_{i,\ell} ,\ldots, x_{d_1, k})   \big| \\
& \overset{(a)}{\le} \frac{e^{-2\bb} }{2\,k^3\, d_1 }  \sup_{\Delta \in \calB(D) }  \Big\{ 2  \sum_{\ell'\in[k]}
   \llangle \Delta, e_{i,j_{i,\ell}}-e_{i,j_{i,\ell'}}\rrangle^2  \sum_{\ell''=1}^k \chi_{i,\ell,\ell',\ell^{''}} +\\ &\sum_{\ell',\ell''\in[k]}
  \llangle \Delta, e_{i,j_{i,\ell'}}-e_{i,j_{i,\ell''}}\rrangle^2  \chi_{i,\ell',\ell'',\ell } \Big\} \\
& \overset{(b)}{\le} \frac{8 \bb^2 e^{-2\bb} }{  k^3\,d_1 }   \Big\{ 2  \sum_{\ell'\in[k]\backslash\{\ell\} }
    \sum_{\ell''=1}^k \chi_{i,\ell,\ell',\ell^{''} } + \sum_{\ell',\ell''\in[k], \ell' \neq \ell'', }
  \chi_{i,\ell',\ell'',\ell } \Big\} \\
& \le \frac{ 16 \bb^2 e^{-2\bb} }{ k\,d_1 } \;,
\end{align*}
where $(a)$ follows because for a fixed  $i$ and $\ell$, the random variable $x_{i,\ell}=(j_{i,\ell},u_{i,\ell})$ can appear 
in three terms, i.e. $\sum_{\ell',\ell''} \llangle \Delta,e_{i,j_{i,\ell}}-e_{i,j_{i,\ell'}}\rrangle^2\chi_{i,\ell,\ell',\ell''}
+ \sum_{\ell',\ell''} \llangle \Delta,e_{i,j_{i,\ell'}}-e_{i,j_{i,\ell}}\rrangle^2\chi_{i,\ell',\ell,\ell''} + \sum_{\ell',\ell''} \llangle \Delta,e_{i,j_{i,\ell'}}-e_{i,j_{i,\ell''}}\rrangle^2\chi_{i,\ell',\ell'',\ell}$, 
and $(b)$ follows because $|\Delta_{ij}|\leq 2\bb$ for all $i$, $j$ since $\Delta\in \calB(D)$. 
The last inequality follows because
in the worst case, 
$\sum_{\ell'\in[k]\backslash\{\ell\}}   \sum_{\ell''=1}^k \chi_{i,\ell,\ell',\ell^{''} }  \leq k(k-1)/2$ 
and $\sum_{\ell',\ell''\in[k] , \ell' \neq \ell''} \chi_{i,\ell',\ell'',\ell }  \leq  k(k-1)$. 
This holds with equality if $\sigma_i(j_{i,\ell})=k$ and $\sigma_i(j_{i,\ell})=1$, respectively. 
By bounded differences inequality, we have
\begin{align*}
\prob{Z - \expect{Z} \ge t } \,\le\, \exp\left( - \frac{k^2\,d_1^2\, t^2}{ 2^7\, \bb^4e^{-4\bb}  d_1 k}\right),
\end{align*}
It follows that for the choice of $t=e^{-4\bb} D^2/(64 d_1d_2)$,
\begin{align*}
\prob{Z - \expect{Z} \ge \frac{e^{-4\bb} D^2}{ 64 d_1d_2} } \le \exp\Big( - \frac{e^{-4\bb} k D^4 }{2^{19} \bb^4  d_1 d_2^2 } \Big) \;.
\end{align*}

We are left to prove the upper bound on $\E[Z]$ using symmetrization and contraction.
Define random variables 
\begin{eqnarray}
  Y_{i,\ell,\ell',\ell''}(\Delta) &\equiv& (\Delta_{i,j_{i,\ell}} - \Delta_{i,j_{i,\ell'}})^2  \chi_{i,\ell,\ell',\ell''} \;,
  \label{eq:defY}
\end{eqnarray}
 where the randomness is in the choice of
alternatives $j_{i,\ell},j_{i,\ell'},$ and $j_{i,\ell''}$, and the outcome of the comparisons of those three alternatives.

The main challenge in applying the symmetrization to 
$\sum_{\ell,\ell',\ell''\in[k]} Y_{i,\ell,\ell',\ell''}(\Delta) $ is that 
we need to 
partition the summation over the set $[k]\times[k]\times[k]$ into 
subsets of independent random variables, such that we can apply the standard symmetrization argument. 
To this end, we prove in the following lemma a
generalization of the well-known problem of 
scheduling a round robin tournament 
to a tournament of matches involving three teams each.   
No teams are present in more than one triple in a single round, and we want to minimize the number of rounds to cover all combination of triples are matched. 
For example, 
when there are $k=6$ teams, there is a simple construction of such a tournament: 
$T_1=\{(1,2,3),(4,5,6)\}$, $T_2=\{1,2,4),(3,5,6)\}$, $T_3=\{(1,2,5),(3,4,6)\}$, 
$T_4=\{(1,2,6),(3,4,5)\}$, $T_5=\{(1,3,4),(2,5,6)\}$, $T_6=\{(1,3,5),(2,4,6)\}$, 
$T_7=\{(1,3,6),(2,4,5)\}$, $T_8=\{(1,4,5),(2,3,6)\}$, $T_9=\{(1,4,6),(2,3,5)\}$, $T_{10}=\{(1,5,6),(2,3,4)\}$. 
This is a perfect scheduling of a tournament with three teams in each match. For a general $k$, the following lemma provides a construction 
with  $O(k^2)$ rounds. 
\begin{lemma}
  There exists a partition $(T_1,\ldots,T_N)$ of $[k]\times[k]\times[k]$ for some $N\leq 24k^2$  such that
  $T_a$'s are disjoint subsets of  $[k]\times[k]\times[k]$,
  $\bigcup_{a\in[N]} T_a = [k]\times[k]\times[k]$,  $|T_a| \leq \lfloor k/3\rfloor  $ and
  for any $a\in[N]$ the set of random variables in $T_a$ satisfy
  \begin{eqnarray*}
    \{Y_{i,\ell,\ell',\ell''}\}_{i\in[d_1],(\ell,\ell',\ell'')\in T_a} \text{ are mutually independent }\;.
  \end{eqnarray*}
  \label{lmm:kwise_partition}
\end{lemma}
Now, we are ready to partition the summation. 
\begin{eqnarray}
  {\mathbb E}\big[ Z \big] &=& \frac{e^{-2\bb}}{2\,k^3\,d_1} {\mathbb E} \Big[  \sup_{\Delta\in\calB(D)} \sum_{i\in[d_1]} \sum_{\ell,\ell',\ell''\in[k]}
    \big\{\E[Y_{i,\ell,\ell',\ell''}(\Delta) ] - Y_{i,\ell,\ell',\ell''}(\Delta)\big\}  \Big] \nonumber \\
    &=&  \frac{e^{-2\bb}}{2\,k^3\,d_1}  {\mathbb E} \Big[ \sup_{\Delta\in\calB(D)} \sum_{i\in[d_1]}  \sum_{a\in[N]} \sum_{(\ell,\ell',\ell'') \in T_a }  \big\{\E[Y_{i,\ell,\ell',\ell''}(\Delta) ] - Y_{i,\ell,\ell',\ell''}(\Delta)\big\} \Big] \nonumber\\
    &\leq& \frac{e^{-2\bb}}{2\,k^3\,d_1}  \sum_{a\in[N]}   {\mathbb E} \Big[ \sup_{\Delta\in\calB(D)} \sum_{i\in[d_1]} \sum_{(\ell,\ell',\ell'') \in T_a }  \big\{\E[Y_{i,\ell,\ell',\ell''}(\Delta) ] - Y_{i,\ell,\ell',\ell''}(\Delta)\big\} \Big] \nonumber\\
    &\leq& \frac{e^{-2\bb}}{k^3\,d_1}  \sum_{a\in[N]}   {\mathbb E} \Big[ \sup_{\Delta\in\calB(D)} \sum_{i\in[d_1]} \sum_{(\ell,\ell',\ell'') \in T_a }   \xi_{i,\ell,\ell',\ell''} Y_{i,\ell,\ell',\ell''}(\Delta) \Big] \nonumber\\
    &=& \frac{e^{-2\bb}}{k^3\,d_1}  \sum_{a\in[N]}  {\mathbb E} \Big[ \sup_{\Delta\in\calB(D)} \sum_{i\in[d_1]} \sum_{(\ell,\ell',\ell'') \in T_a }   \xi_{i,\ell,\ell',\ell''} (\Delta_{i,j_{i,\ell}} - \Delta_{i,j_{i,\ell'}})^2 \chi_{i,\ell,\ell',\ell''}  \Big]\;,
    \label{eq:kwise_symbound3}
\end{eqnarray}
where the first inequality follows from the fact that sum of the supremum is no less than the supremum of the sum,
and the second inequality follows from standard symmetrization argument applied to independent random variables $\{Y_{i,\ell,\ell',\ell''}(\Delta) \}_{i\in[d_1],(\ell,\ell',\ell'')\in T_a}$  with i.i.d.\ Rademacher random variables $\xi_{i,\ell,\ell',\ell''}$'s.
Since $ (\Delta_{i,j_{i,\ell}} - \Delta_{i,j_{i,\ell'}})^2 \chi_{i,\ell,\ell',\ell''} \leq 4\bb  |\Delta_{i,j_{i,\ell}} - \Delta_{i,j_{i,\ell'}}| \chi_{i,\ell,\ell',\ell''}$,
we have by the Ledoux-Talagrand contraction inequality \cite{ledoux2013probability} that
\begin{align}
  &{\mathbb E} \Big[ \sup_{\Delta\in\calB(D)} \sum_{i\in[d_1]} \sum_{(\ell,\ell',\ell'') \in T_a }   \xi_{i,\ell,\ell',\ell''} (\Delta_{i,j_{i,\ell}} - \Delta_{i,j_{i,\ell'}})^2  \chi_{i,\ell,\ell',\ell''} \Big] \nonumber\\
  & \leq \; 8\bb {\mathbb E} \Big[ \sup_{\Delta\in\calB(D)} \sum_{i\in[d_1]} \sum_{(\ell,\ell',\ell'') \in T_a }   \xi_{i,\ell,\ell',\ell''} \, \chi_{i,\ell,\ell',\ell''}\, \llangle \Delta, e_i(e_{j_{i,\ell}} - e_{j_{i,\ell'}})^T \rrangle  \Big] \label{eq:kwise_symbound2}
\end{align}
Applying H\"older's inequality, we get that
\begin{align}
  &\Big|  \sum_{i\in[d_1]} \sum_{(\ell,\ell',\ell'') \in T_a }   \xi_{i,\ell,\ell',\ell''} \,\chi_{i,\ell,\ell',\ell''} \, \llangle \Delta, e_i(e_{j_{i,\ell}} - e_{j_{i,\ell'}})^T \rrangle     \Big| \nonumber \\
  & \leq \; \nucnorm{\Delta}  \lnorm{ \sum_{i\in[d_1]} \sum_{(\ell,\ell',\ell'') \in T_a }   \xi_{i,\ell,\ell',\ell''} \,  \chi_{i,\ell,\ell',\ell''} \, \big( e_i(e_{j_{i,\ell}} - e_{j_{i,\ell'}})^T \big)   }{2} \;.   \label{eq:kwise_symbound}
\end{align}

We are left to prove that the expected value of the right-hand side of the above inequality is bounded by
$C \nucnorm{\Delta}  \sqrt{k d_1 \log d / \min\{d_1,d_2\}} $ for some numerical constant $C$. 
For $i\in[d_1]$ and $(\ell,\ell',\ell'')\in T_a$, let $W_{i,\ell,\ell',\ell''} =\xi_{i,\ell,\ell',\ell''} \, \chi_{i,\ell,\ell',\ell''}\, \big( e_i(e_{j_{i,\ell}} - e_{j_{i,\ell'}})^T \big) $ 
be independent zero-mean random matrices, such that
\begin{eqnarray*}
  \lnorm{W_{i,\ell,\ell',\ell''}}{2} = \lnorm{\xi_{i,\ell,\ell',\ell''} \, \chi_{i,\ell,\ell',\ell''}\, \big( e_i(e_{j_{i,\ell}} - e_{j_{i,\ell'}})^T \big) }{2} \leq \sqrt{2} \;,
\end{eqnarray*}
almost surely, and
\begin{eqnarray*}
  \E[W_{i,\ell,\ell',\ell''}W_{i,\ell,\ell',\ell''}^T] &=& \E[\big( e_i(e_{j_{i,\ell}} - e_{j_{i,\ell'}})^T (e_{j_{i,\ell}} - e_{j_{i,\ell'}}) e_i^T\big)  \chi_{i,\ell,\ell',\ell''} ] \\
  &=& 2  \expect{ \chi_{i,\ell,\ell',\ell''} } e_ie_i^T \\
  &\preceq& 2 e_ie_i^T\;,
\end{eqnarray*}
and
\begin{eqnarray*}
  \E[W_{i,\ell,\ell',\ell''}^T W_{i,\ell,\ell',\ell''}]&=& \E[\big(  (e_{j_{i,\ell}} - e_{j_{i,\ell'}}) e_i^Te_i(e_{j_{i,\ell}} - e_{j_{i,\ell'}})^T\big)  \chi_{i,\ell,\ell',\ell''} ] \\
  &\preceq& \E[ (e_{j_{i,\ell}} - e_{j_{i,\ell'}}) e_i^Te_i(e_{j_{i,\ell}} - e_{j_{i,\ell'}})^T ] \\
  &=& \frac{2}{d_2} {\mathbf I}_{d_2 \times d_2} - \frac{2}{d_2^2} \ones\ones^\top
  \;.
\end{eqnarray*}
This gives
\begin{eqnarray*}
  \sigma^2 &=& \max\left\{ \lnorm{ \sum_{\substack{i\in[d_1] \\ (\ell,\ell',\ell'')\in T_a}} \E[W_{i,\ell,\ell',\ell''}W_{i,\ell,\ell',\ell''}^T]}{2} , \lnorm{ \sum_{\substack{i\in[d_1] \\ (\ell,\ell',\ell'')\in T_a}} \E[W_{i,\ell,\ell',\ell''}^TW_{i,\ell,\ell',\ell''}]}{2} \right\} \\
  &\leq& \max\left\{ 2 |T_a|\,,\, \frac{2d_1 |T_a| }{d_2} \right\}   \,=\, \frac{2d_1|T_a|}{\min\{d_1,d_2\}} \leq \frac{2d_1 k }{3 \min\{d_1,d_2\}} \;,
\end{eqnarray*}
since we have designed $T_a$'s such that $|T_a|\leq k/3$.
Applying matrix Bernstein inequality \cite{Jo11} yields the tail bound
\begin{eqnarray*}
  \prob{\lnorm{\sum_{i\in[d_1]} \sum_{(\ell,\ell',\ell'')\in T_a} W_{i,\ell,\ell',\ell''} }{2} \geq t } &\leq& (d_1+d_2) \exp\Big( \frac{-t^2/2}{\sigma^2 + \sqrt{2} t/3} \Big) \;.
\end{eqnarray*}
Choosing $t= \max \big\{\, \sqrt{32 k d_1 \log d /(3\min\{d_1,d_2\})} , (16\sqrt{2}/3) \log d \,\big\}$,
we obtain with probability at least  $1-2d^{-3}$,
\begin{eqnarray*}
  \lnorm{\sum_{i\in[d_1]} \sum_{(\ell,\ell',\ell'')\in T_a} W_{i,\ell,\ell',\ell''} }{2} &\leq&
    \max \left\{ \sqrt{\frac{32 k d_1 \log d }{3 \min\{d_1 , d_2\}}} \,,\, \frac{16 \sqrt{2} \log d}{3} \right\}\;.
\end{eqnarray*}
It follows from the fact
$\lnorm{\sum_{i\in[d_1]}\sum_{(\ell,\ell',\ell'')\in T_a} W_{i,\ell,\ell',\ell''}}{2} \leq
\sum_{i,(\ell,\ell',\ell'')} \lnorm{W_{i,\ell,\ell',\ell''}}{2} \leq \frac{\sqrt{2} d_1 k }3$ that
\begin{eqnarray*}
  \E\left[\lnorm{\sum_{i\in[d_1]}\sum_{(\ell,\ell',\ell'')\in T_a} W_{i,\ell,\ell',\ell''}}{2}\right] &\leq&
    \max\left\{ \sqrt{\frac{32 k d_1 \log d }{3 \min\{d_1 , d_2\}}},  \frac{16 \sqrt{2} \log d}{3} \right\} \,+\,  \frac{ 2 \sqrt{2} d_1 k}{3d^3}\\
    &\leq & 2\sqrt{\frac{32 k d_1 \log d }{3 \min\{d_1 , d_2\}}} \;,
\end{eqnarray*}
where the last inequality follows from the assumption that $(16\min\{d_1,d_2\} \log d)/(3d_1)  \leq k\leq d_1^2 \log d $.
Substituting this in the RHS of Eq. \eqref{eq:kwise_symbound},
and then together with Eqs. \eqref{eq:kwise_symbound2} and \eqref{eq:kwise_symbound3},
this gives the following desired bound:
\begin{eqnarray*}
  \E[Z] &\leq& \sum_{a\in[N]} \sup_{\Delta\in\calB(D)} \frac{16 \bb e^{-2\bb} }{k^3\,d_1  } \sqrt{\frac{32 k d_1 \log d }{3 \min\{d_1,d_2\}}} \nucnorm{\Delta}\\
    &\leq& \sum_{a\in[N]}   \frac{e^{-4\bb}\sqrt{2}}{16\sqrt{3} k^2\,d_1\,d_2}  \underbrace{\Big(2^{10} e^{2\bb} \bb d_2 \sqrt{\frac{d_1 \log d}{k \min\{d_1,d_2\} } } \Big)}_{=\mu} \nucnorm{\Delta} \\
    &\leq & \frac{9 e^{-4 \bb} D^2}{40 d_1d_2}\;,
\end{eqnarray*}
where the last inequality holds because $N\leq 4 k^2$ and  $ \mu \nucnorm{\Delta} \leq D^2$.

\subsection{Proof of Lemma \ref{lmm:kwise_partition}}
Recall that $Y_{i,\ell,\ell',\ell''}(\Delta) = (\Delta_{i,j_{i,\ell}} - \Delta_{i,j_{i,\ell'}})^2  \chi_{i,\ell,\ell',\ell''}$, as defined in \eqref{eq:defY}. 
From the random utility model (RUM) interpretation of the MNL model presented in Section~\ref{sec:intro}, 
it is not difficult to show that $Y_{i,\ell,\ell',\ell''}$ and $Y_{i,\tell,\tell',\tell''}$ are 
mutually independent if the two triples $(\ell,\ell',\ell'')$ and $(\tell,\tell',\tell'')$ do not overlap, i.e., no index is present in both triples.   

Now, borrowing the terminologies from round robin tournaments, 
we construct a schedule for a tournament with $k$ teams where each match involve three teams. 
Let $T_{a,b}$ denote a set of triples playing at the same round, indexed by two integers $a\in\{3,\ldots,2k-3\}$ and $b\in\{5,\ldots,2k-1\}$. 
Hence, there are total $N=(2k-5)^2$ rounds. 

Each round $(a,b)$ consists of disjoint triples and is defined as 
\begin{eqnarray*}
  T_{a,b} &\equiv& \big\{(\ell,\ell',\ell'') \in [k]\times[k]\times[k] \,|\, \ell<\ell'<\ell'', \ell+\ell'=a, \text{ and } \ell'+\ell''=b \big\}\;.
\end{eqnarray*}

We need to prove that $(a)$ there is no missing triple; and $(b)$ no team plays twice in a single round. 
First, for any ordered triple $(\ell,\ell',\ell'')$, there exists $a\in\{3,\ldots,2k-3\}$ and $b\in\{5,\ldots,2k-1\}$ such that 
$\ell+\ell'=a$ and $\ell'+\ell''=b$. This proves that all ordered triples are covered by the above construction. 
Next, given a pair $(a,b)$, no two triples in $T_{a,b}$ can share the same team. 
Suppose there exists two distinct ordered triples $(\ell,\ell',\ell'')$ and $(\tell,\tell',\tell'')$ both in $T_{a,b}$, and one of the triples are shared.  
Then, from the two equations $\ell+\ell'=\tell+\tell'=a$ and $\ell'+\ell''=\tell'+\tell''=b $, it follows that all three indices must be the same, which 
is a contradiction. 
This proves the desired claim for ordered triples. 

One caveat is that we wanted to cover the whole $[k]\times[k]\times[k]$, and not just the ordered triples. 
In the above construction, for example, a triple $(3,2,1)$ does not appear. 
This can be resolved by simply taking all $T_{a,b}$'s from the above construction, and 
make 6 copies of each round, and permuting all the triples in each copy according to 
the same permutation over $\{1,2,3\}$. This increases the total rounds to $N=6(2k-5)^2\leq 24k^2$. 
Note that $|T_{a,b}|\leq\lfloor k/3\rfloor$ since no item can be in more than one triple. 

\section{Proof of Estimating Approximate Low-rank Matrices in Corollary \ref{cor:kwise_appxlowrank}}
\label{sec:kwise_cor_proof}
We follow closely the proof of a similar corollary in \cite{NW11}. 
First fix a threshold $\tau>0$, and set $r=\max\{j|\sigma_j(\Theta^*)>\tau\}$. With this choice of $r$, we have 
\begin{eqnarray*}
  \sum_{j=r+1}^{\min\{d_1,d_2\}} \sigma_j(\Theta^*) \;=\; \tau \sum_{j=r+1}^{\min\{d_1,d_2\}} \frac{\sigma_j(\Theta^*)}{\tau} \;\leq\; \tau \sum_{j=r+1}^{\min\{d_1,d_2\}} \Big(\frac{\sigma_j(\Theta^*)}{\tau}\Big)^q \;\leq\; \tau^{1-q} \rho_q \;.
\end{eqnarray*}
Also, since $r\tau^q \leq \sum_{j=1}^r \sigma_j(\Theta^*)^q \leq \rho_q$, it follows that 
$\sqrt{r} \leq \sqrt{\rho_q} \tau^{-q/2}$. Using these bounds, Eq. \eqref{eq:kwise_ub} is now 
\begin{eqnarray*}
  \fnorm{\hTheta-\Theta}^2\;\leq\;  \underbrace{288\sqrt{2}c_0 e^{4\bb}d_1d_2\lambda_0}_{ = A} \,\big( \sqrt{\rho_q} \tau^{-q/2} \fnorm{\hTheta-\Theta} + \tau^{1-q}\rho_q \,\big) \;.
\end{eqnarray*}
With the choice of  $\tau = A$ and due to the fact that $x^2 \leq b x + c$ implies $x \leq (b + \sqrt{b^2 + 4c})/2$ we have,
\begin{eqnarray*}
  \fnorm{\hTheta-\Theta} \;\leq\; 2 \sqrt{ \rho_q } A^{(2-q)/2} \;.
\end{eqnarray*}

\section{Proof of the Information-theoretic Lower Bound in Theorem \ref{thm:kwise_lb}}
\label{sec:kwise_lb_proof}

The proof uses information-theoretic methods 
which reduces the estimation problem to a multiway hypothesis testing problem. 
To prove a lower bound on the expected error, it suffices to prove that,
\begin{eqnarray}
  \sup_{\Theta^*\in\Omega_\bb } \prob{ \fnorm{\hTheta-\Theta^*}^2 \geq \frac{\delta^2}{4} } &\geq& \frac{1}{2}
  \;.\label{eq:lb_fano1}
\end{eqnarray}
To prove the above claim, we follow the standard recipe of constructing a packing in $\Omega_\bb$. Consider a family $\{\Theta^{(1)},\ldots,\Theta^{(M(\delta)}\}$ of 
$d_1\times d_2$ dimensional matrices contained in $\Omega_\bb$ satisfying 
$\fnorm{\Theta^{(\ell_1)}-\Theta^{(\ell_2)}}\geq\delta$ for all $\ell_1,\ell_2,\in[M(\delta)]$. 
We will use $M$ to refer to $M(\delta)$ for simplify the notation. 
 Suppose we draw an index $L\in[M(\delta)]$ uniformly at random, and 
 we are given direct observations 
 $\sigma_i$ as per MNL model with $\Theta^*=\Theta^{(L)}$ on a randomly chosen set of $k$ items $S_i$ 
 for each user $i\in[d_1]$. 
It follows from triangular inequality that 
\begin{eqnarray}
  \sup_{\Theta^*\in\Omega_\bb } \prob{ \fnorm{\hTheta-\Theta^*}^2 \geq \frac{\delta^2}{4} } &\geq& \prob{\hL \neq L}\;,
  \label{eq:lb_fano2}
\end{eqnarray}
where $\hL$ is the resulting best estimate of the multiway hypothesis testing on $L$. 
 The generalized Fano's inequality gives 
 \begin{eqnarray}
  \prob{\hL\neq L | S(1),\ldots,S(d_1)} &\geq& 1-\frac{I(\hL;L) + \log 2}{\log M} \\
  &\geq& 1-\frac{ {M \choose 2}^{-1} \sum_{\ell_1,\ell_2\in[M]} D_{\rm KL}(\Theta^{(\ell_1)} \| \Theta^{(\ell_2)}) +\log 2}{\log M} \label{eq:kwise_fano}\;,
 \end{eqnarray}
where $D_{\rm KL}(\Theta^{(\ell_1)}\|\Theta^{(\ell_2)})$ denotes the Kullback-Leibler divergence between the distributions of the partial rankings 
$\prob{\sigma_1,\ldots,\sigma_{d_1}|\Theta^{(\ell_1)},S(1),\ldots,S(d_1)}$ and \\
$\prob{\sigma_1,\ldots,\sigma_{d_1}|\Theta^{(\ell_2)},S(1),\ldots,S(d_1)}$. 
The second inequality follows from a standard technique, which we repeat here for completeness. 
Let $\Sigma=\{\sigma_1,\ldots,\sigma_{d_1}\}$ denote the observed outcome of comparisons. 
Since $L\text{--} \Theta^{(L)}\text{--}\Sigma\text{--}\hL$ form a Markov chain, the data processing inequality gives 
$I(\hL;L) \leq I(\Sigma;L)$. For simplicity, we drop the conditioning on the set of alternatives $\{S(1),\ldots,S(d_1)\}$, and 
and let $p(\cdot)$ denotes joint, marginal, and conditional distribution of respective random variables. 
It follows that 
\begin{eqnarray} 
  I(\Sigma;L) &=& \sum_{\ell\in[M],\Sigma}  p(\Sigma|\ell)\frac{1}{M} \log \frac{p(\ell,\Sigma)}{p(\ell)p(\Sigma)} \nonumber\\
    &=& \frac{1}{M} \sum_{\ell\in[M]} \sum_{\Sigma}  p(\Sigma|\ell) \log \frac{p(\Sigma|\ell)}{\frac{1}{M}\sum_{\ell'}p(\Sigma|\ell')}    \nonumber\\
    &\leq& \frac{1}{M^2} \sum_{\ell,\ell'\in[M]} \sum_{\Sigma}  p(\Sigma|\ell) \log \frac{p(\Sigma|\ell)}{p(\Sigma|\ell')} \nonumber\\
    &=& \frac{1}{M^2} \sum_{\ell,\ell'\in[M]} D_{\rm KL} (\Theta^{(\ell_1)} \| \Theta^{(\ell_2)} )\;,
\end{eqnarray}
where the first inequality follows from Jensen's inequality.
To compute the KL-divergence, recall that from the RUM interpretation of the MNL model (see Section \ref{sec:intro}),  
one can generate sample rankings $\Sigma$  by drawing random variables with exponential distributions with mean $e^{\Theta^*_{ij}}$'s. Precisely, let  
$X^{(\ell)} = [X^{(\ell)}_{ij}]_{i\in[d_1],j\in S_i} $ denote the set of random variables, where 
$X^{(\ell)}_{ij}$ is drawn from the exponential distribution with mean $e^{-\Theta^{(\ell)}_{ij}}$. 
The MNL ranking follows by ordering the alternatives in each $S_i$ according to this $\{X^{(\ell)}_{ij}\}_{j\in S_i}$ 
by ranking the smaller ones on the top. This forms a 
Markov chain $L\text{--}X^{(L)}\text{--}\Sigma$, and the 
standard data processing inequality gives 
\begin{eqnarray}
  D_{\rm KL}(\Theta^{(\ell_1)}\|\Theta^{(\ell_2)} ) &\leq& D_{\rm KL}( X^{(\ell_1)} \|X^{(\ell_2)} ) \\ 
    &=& \sum_{i\in[d_1]} \sum_{j\in S_i}   \Big\{ e^{\Theta^{(\ell_1)}_{ij}-\Theta^{(\ell_2)}_{ij}} - (\Theta^{(\ell_1)}_{ij}-\Theta^{(\ell_2)}_{ij}) -1 \Big\} \\
    &\leq& \frac{e^{2\bb}}{4\bb^2} \sum_{i\in[d_1]} \sum_{j\in S_i}    (\Theta^{(\ell_1)}_{ij}-\Theta^{(\ell_2)}_{ij})^2 \;,
\end{eqnarray}
where the last inequality follows from the fact that $e^x-x-1 \leq (e^{2\bb}/(4\bb^2))x^2$ for any $x\in[-2\bb,2\bb]$.
Taking expectation over the randomly chosen set of alternatives,  
\begin{eqnarray}
  \E_{S(1),\ldots,S(d_1)} [D_{\rm KL}(\Theta^{(\ell_1)}\|\Theta^{(\ell_2)})] &\leq&  
  \frac{e^{2\bb}\,k }{4\,\bb^2\, d_2}\fnorm{\Theta^{(\ell_1)}-\Theta^{(\ell_2)}}^2\;.
\end{eqnarray}
Combined with \eqref{eq:kwise_fano}, we get that 
\begin{eqnarray}
  \prob{\hL\neq L} &=& \E_{S(1),\ldots,S(d_1)}[\prob{\hL\neq L|S(1),\ldots,S(d_1)} ]\\
    &\geq & 1- \frac{ {M \choose 2}^{-1}  \sum_{\ell_1,\ell_2\in[M]}  (e^{2\bb} k /(4\bb^2 d_2))\fnorm{\Theta^{(\ell_1)} - \Theta^{(\ell_2)}}^2 + \log 2}{\log M}\;,
\end{eqnarray}
The remainder of the proof relies on the following probabilistic packing. 
\begin{lemma}
  \label{lem:packing} 
  Let $d_2 \geq d_1 \geq 607$ be positive integers. 
  Then for each $r\in\{1,\ldots,d_1 \}$, and for any positive 
  $\delta>0$ there exists a family of $d_1\times d_2$ dimensional matrices 
  $\{\Theta^{(1)},\ldots,\Theta^{(M(\delta))}\}$ with cardinality 
  $M(\delta) = \lfloor (1/4)\exp(r d_2 /576)\rfloor$ such that each matrix is rank $r$ and the following bounds hold:
  \begin{eqnarray}
    \fnorm{\Theta^{(\ell)}}&\leq& \delta \;, \text{ for all }\ell\in[M]\\
    \fnorm{\Theta^{(\ell_1)} - \Theta^{(\ell_2)}} &\geq& \delta \;, \text{ for all } 
    \ell_1, \ell_2\in[M] \\
    \Theta^{(\ell)} &\in& \Omega_{\tilde{\bb}} \;, \text{ for all } \ell\in [M]\;, 
  \end{eqnarray}
  with $\tilde{\bb}=   (8\delta/d_2)\sqrt{2\log d} $ for $d=(d_1+d_2)/2$.
\end{lemma}
We omit the proof of the above lemma since it is similar to that of Lemma \ref{lem:graph_packing2}. Suppose $\delta\leq \alpha  d_2/(8\sqrt{2\log d}) $ such that 
the matrices in the packing set are entry-wise bounded by $\bb$,  then the above lemma 
implies that $\fnorm{\Theta^{(\ell_1)} - \Theta^{(\ell_2)}}^2 \leq4\delta^2$, which gives 
\begin{eqnarray*}
  \prob{ \hL\neq L } &\geq& 1- \frac{\frac{e^{2\bb} k \delta^2}{\bb^2 d_2} + \log 2}{\frac{rd}{576} - 2\log 2}  \;\geq\; \frac{1}{2}\;,
\end{eqnarray*}
where the last inequality holds for $\delta^2 \leq (\bb^2 d_2/(e^{2\bb} k))((rd/1152) -2\log 2)$.  
If we assume $rd \geq 3195$ for simplicity, 
this bound on $\delta$ can be simplified to $\delta\leq \bb e^{-\bb} \sqrt{r\,d_2\, d/(2304\, k)}$.
Together with \eqref{eq:lb_fano1} and \eqref{eq:lb_fano2}, this proves that for all $\delta\leq\min\{ \bb d_2  /(8\sqrt{2\log d}), \\\bb e^{-\bb}  \sqrt{\,r\,d_2\, d/(2304\, k)} \}$, 
\begin{eqnarray*}
  \inf_{\hTheta} \sup_{\Theta^*\in\Omega_{\bb}} \E\Big[\, \fnorm{\hTheta-\Theta^*} \,\Big] &\geq& \frac{\delta}{4}\;.
\end{eqnarray*}
Choosing $\delta$ appropriately to maximize the right-hand side finishes the proof of the desired claim.

\section{Proof of Pairwise Rank Breaking in Theorem \ref{thm:pairwise_ub}}
\label{sec:pairwise-theorem}

Analogous to Section \ref{sec:kwise_ub_proof}, we define the gradient  $\nabla \calL (\Theta)$ as $\nabla_{ij} \calL = \frac{\partial\calL(\Theta)}{\partial \Theta_{ij}}$ and $\Delta \equiv \hat{\Theta} - \Theta^*$, and  
provide two main technical lemmas. 

\begin{lemma}
  \label{lem:pairwise_deltabound} 
  If $\lambda\geq 2 \lnorm{\nabla\calL (\Theta^*)}{2}$, then we have,
  \begin{align}
  \nucnorm{\Delta} &\leq 4\sqrt{2 r} \fnorm{\Delta} + 4 \sum^{\min\{d_1,d_2\}}_{j = \rho+1}\sigma_j(\Theta^*) \;, \label{eq:pairwise_deltabound}
  \end{align}
for all $\rho \in[\min\{d_1,d_2\}]$.\
\end{lemma}
  \begin{proof}
    This follows from the proof of Lemma \ref{lmm:kwise_deltabound}, which  only depends on the convexity of $\calL(\Theta)$. 
  \end{proof}

\begin{lemma}
  \label{lem:pairwise_deriv} 
  For any positive constant $c \geq 1$,  if $k \leq \max\{d_1, d_2^2/d_1\} \log d \text{ and } d_1 \geq 4$  then  with probability at least $1 - 2d^{-c}$,
  \begin{align}
     \lnorm{\nabla \calL(\Theta^\ast)}{2} 
      \;\leq \;
      \sqrt{\frac{ 16 (c+4)\log d } { k\,d_1^2  }}   
      \max \left\{ \sqrt{ \max \left\lbrace \frac{1}{4}, \frac{d_1}{d_2}\right\rbrace}  \,,\, \frac{2 }{3} \sqrt{\frac{2(c+4)\log d  }{k}} \right\} \;. \label{eq:pairwise-deriv-ineq}
  \end{align}
\end{lemma}
The proof of this lemma is provided in Section \ref{sec:pairwise_deriv_proof}. 
We will simplify the above lemma by assuming, $2(c+4) \log d \leq k$ which implies the last term in RHS is less than equal to first term,

\begin{align}
 \frac{2 }{3} \sqrt{\frac{2(4+c)\log d  }{k}} \leq \sqrt{\frac{1}{4}}\;. \label{eq: pairwise-max-eliminate}
\end{align}

\eqref{eq: pairwise-max-eliminate} simplifies \eqref{eq:pairwise-deriv-ineq} as,
\begin{align}
\lnorm{\nabla \calL(\Theta^\ast)}{2} 
      &\leq \sqrt{\frac{ 16(c+4)\log d} { k\,d_1^2  } \max \left\lbrace \frac{1}{4}, \frac{d_1}{d_2}\right\rbrace} \nonumber \\
      &\leq \sqrt{\frac{ 32 d\;(c+4)\log d} { k\,d_1^2\;d_2}} \nonumber \\      &\overset{(a)}{\leq} \sqrt{32(c+4)} \lambda\;,
\end{align}
where $(a)$ is due to \eqref{eq:pairwise-lambda} .

For Lemma \ref{lem:pairwise_deltabound} and further proof of Theorem\ref{thm:pairwise_ub} we want $\lambda \geq 2\lnorm{\nabla \calL(\Theta)}{2}$, therefore we assume that,
\begin{align}
\lambda \in [2\sqrt{32 (c+4)} \lambda, c_p \lambda], \text{ for some } c_p \geq 2\sqrt{32(c+4)} \label{eq: pairwise_cp}
\end{align}

Similar to the k-wise ranking,we will divide the proof into two cases and each part we will prove that $\fnorm{\Delta}^2 \leq 36 e^{2 \alpha}\;c\;\lambda\;d_1d_2\;\nucnorm{\Delta}$ with probability at least $1 - 2/d^c - 2/d^{2^{13}}$. We define a new constant $\mu$ as,
\begin{align}
\mu = 16 \alpha \sqrt{\frac{48\;d_1 d_2^2 \log d}{k\; \min\{d_1,d_2\}}} \label{eq:pairwise-mu}\;.
\end{align}

\textbf{Case 1: Assume $\mu \nucnorm{\Delta} \leq \fnorm{\Delta}^2 $}.\\
Since $\calL$ is a sum of a linear function of $\Theta$ and log-sum-exponential functions, which are convex, we know that $\calL$ is a convex function of $\Theta$. Therefore, by convexity and Taylor expansion we get,
\begin{align}
\label{eq: taylor-pairwise}
\calL(\hat{\Theta}) &= \calL(\Theta^*) - {\llangle\nabla \calL(\Theta^*), \Delta \rrangle}\;\; + \\
& \frac{1}{2!\;d_1 {k \choose 2}}\sum_{i=1}^{d_1} \sum_{(m_1, m_2) \in \calP_0} \frac{e^{\Theta_{i, u_{i,m_1}}}e^{\Theta_{i, u_{i,m_2}}}}{\Big(e^{\Theta_{i,u_{i,m_1}}} + e^{\Theta_{i,u_{i,m_2}}}\Big)^2} \Big(\Delta_{i, u_{i,m_1}} - \Delta_{i, u_{i,m_2}}\Big)^2 \nonumber\;,
\end{align}
where $\Theta = a\Theta^* + (1-a)\hat{\Theta}$ for some $a \in [0,1]$ and $\calP_0 = \{(i,j) |\;\; 1 \leq i < j \leq k\}$. We lower bound the final term in \eqref{eq: taylor-pairwise} as,
\begin{align}
\frac{1}{2!\;d_1 {k \choose 2}}\sum_{i=1}^{d_1} \sum_{(m_1, m_2) \in \calP_0} &\frac{e^{\Theta_{i, u_{i,m_1}}}e^{\Theta_{i, u_{i,m_2}}}}{\Big(e^{\Theta_{i, u_{i,m_1}}} + e^{\Theta_{i, u_{i,m_2}}}\Big)^2} \Big(\Delta_{i, u_{i,m_1}} - \Delta_{i, u_{i,m_2}}\Big)^2 \nonumber \\ 
&\overset{(a)}{\geq} \frac{1}{2\;d_1 {k \choose 2}}\sum_{i=1}^{d_1} \sum_{(m_1, m_2) \in \calP_0} \frac{e^{-\alpha}e^{\alpha}}{(e^{-\alpha}+e^{\alpha})^2} \Big(\Delta_{i, u_{i,m_1}} - \Delta_{i, u_{i,m_2}}\Big)^2 \nonumber \\
&\geq \frac{1}{2\;d_1 {k \choose 2}}\sum_{i=1}^{d_1} \sum_{(m_1, m_2) \in \calP_0} \frac{e^{-2\alpha}}{4}\Big(\Delta_{i, u_{i,m_1}} - \Delta_{i, u_{i,m_2}}\Big)^2\;, \label{eq:pairwise-hess-lower}
\end{align}
where $(a)$ is due to the fact that $\Delta_{ij}$'s are upper and lower bounded by $\alpha$ and $-\alpha$ respectively. We can bound this term further according to the following Lemma.

\begin{lemma}
  \label{lem:pairwise_hessian} 
  For $(4 \log d)/9 \leq k \leq \max \{d_1,d_2^2/d_1 \} \log d$, with probability at least $1 - 2d^{-2^{13}}$,
  \begin{align}
  \frac{1}{d_1 {k \choose 2}}\sum_{i=1}^{d_1} \sum_{(m_1, m_2) \in \calP_0} \Big(\Delta_{i, u_{i,m_1}} - \Delta_{i, u_{i,m_2}}\Big)^2 \geq \frac{1}{3d_1d_2} \fnorm{\Delta}^2\;,
  \end{align}
for all $\Delta \in \calA_p$ where,
  \begin{align}
  \label{eq:pairwise-A-set}
  \calA =\Bigg\lbrace \Delta \in \mathbb{R}^{d_1 \times d_2}\;\Big|\; \lnorm{\Delta}{\infty} \leq 2\alpha,\; \sum_{j \in [d_2]} \Delta_{ij} = 0 \text{ , for all } i \in [d_2],\text{ and } \mu \nucnorm{\Delta} \leq \fnorm{\Delta}^2 \Bigg\rbrace\;.
  \end{align}
\end{lemma}
The proof is given in  Section \ref{sec:pairwise_hessian_proof}. Now using Lemma \ref{lem:pairwise_hessian} and \eqref{eq:pairwise-hess-lower} with high probability we get,
\begin{align}
\frac{1}{2!\;d_1 {k \choose 2}}\sum_{i=1}^{d_1} \sum_{(m_1, m_2) \in \calP_0} &\frac{e^{\Theta_{i, u_{i,m_1}}}e^{\Theta_{i, u_{i,m_2}}}}{\Big(e^{\Theta_{i, u_{i,m_1}}} + e^{\Theta_{i, u_{i,m_2}}}\Big)^2} \Big(\Delta_{i, u_{i,m_1}} - \Delta_{i, u_{i,m_1}}\Big)^2 \geq \frac{e^{-2\alpha}}{24\;d_1\;d_2} \fnorm{\Delta}^2\;.
\end{align}

Incorporating the above inequality in \eqref{eq: taylor-pairwise} we obtain,
\begin{align}
\frac{e^{-2\alpha}}{24\;d_1\;d_2} \fnorm{\Delta}^2 \leq \calL(\hat{\Theta}) - \calL(\Theta^*) + {\llangle\nabla \calL(\Theta^*), \Delta \rrangle}\;.
\end{align}

From the definition of $\hat{\Theta}$ we have $\calL(\hat{\Theta}) - \calL(\Theta^*) \leq \lambda\left(\nucnorm{\Theta^*} - \nucnorm{\hat{\Theta}}\right) \leq \lambda \nucnorm{\Delta}$, and we assume that $\lambda \geq 2\sqrt{32(c+1)} \; \lambda$, so that $\lambda \geq 2 \lnorm{\nabla \calL \left( \Theta^* \right)}{2}$ is true with a probability of at least $1 - 2d^{-c}$ from Lemma \ref{lem:pairwise_deriv} . These give us the following with at least probability $1 - 2d^{-c} - 2d^{-2^{13}}$.
\begin{align}
\frac{e^{-2\alpha}}{24\;d_1\;d_2} \fnorm{\Delta}^2 &\leq \lambda \nucnorm{\Delta} + \lnorm{\nabla \calL(\Theta^*)}{2}\nucnorm{\Delta}\, \nonumber \\
&\leq \frac{3 \lambda}{2} \nucnorm{\Delta}
\end{align}
which gives us,
\begin{align}
\fnorm{\Delta}^2 &\leq 36 e^{2 \alpha}\;\lambda\;d_1d_2\;\nucnorm{\Delta}  \nonumber \\
&\overset{(a)}{\leq} 36 e^{2 \alpha}\;c_p\;\lambda\;d_1d_2\;\nucnorm{\Delta}\;,
\end{align}
where $(a)$ is due to the fact that $\lambda \leq c_p \lambda$.

\textbf{Case 2: Assume $\fnorm{\Delta}^2 \leq \mu \nucnorm{\Delta}$}.\\
Here we prove that $\mu \leq 36\;e^{2\alpha}\;c_p \lambda\;d_1d_2$.
\begin{align}
\frac{\mu}{36\;e^{2\alpha}\;c_p \lambda\;d_1d_2} 
&\overset{(a)}{\leq} \frac{\alpha}{e^{2\alpha}} \times \frac{16 \sqrt{48}}{72 \sqrt{32(c+4)}} \times \sqrt{\frac{d_1 d_2}{min\{d_1,d_2\}d}} \nonumber \\
&\overset{(b)}{\leq} 1 \times \frac{16 \sqrt{48}}{72 \sqrt{32 \times 4}} \times \sqrt{\frac{\max\{d_1,d_2\}}{d}} \nonumber \\
&\overset{(c)}{\leq} \sqrt{\frac{\max\{d_1,d_2\}}{2d}} \nonumber \\
&\overset{(d)}{\leq}1\;,
\end{align}
where $(a)$ is by substituting $\mu$, $\lambda$ and $c_p$ from \eqref{eq:pairwise-mu}, \eqref{eq:pairwise-lambda} and \eqref{eq: pairwise_cp} respectively, $(b)$ is because $x \leq e^x$ $(c)$ is because $d = (\max\{d_1,d_2\} + \min\{d_1,d_2\})/2$.

Now combining the above result with \eqref{eq:pairwise_deltabound} we get with probability at least $1 - 2d^{-c} - 2d^{-2^{13}}$,
\begin{align}
\frac{1}{d_1 d_2}\fnorm{\Delta}^2 &\leq 144\sqrt{2}e^{2\alpha}c_p \lambda \sqrt{r} \fnorm{\Delta} + 144e^{2\alpha}c_p \lambda \sum^{\min\{d_1,d_2\}}_{j = \rho+1}\sigma_j(\Theta^*) \;.
\end{align}

\subsection{Proof of Lemma \ref{lem:pairwise_deriv}} 
\label{sec:pairwise_deriv_proof}
From definition of $\calL(\Theta)$ in \eqref{eq:pairwise-LL} we get, 
\begin{align}
\nabla \calL_{p}(\Theta^\ast) = \frac{1}{d_1 {k \choose 2}} \sum_{i = 1}^{d_1} \sum_{(m_1, m_2) \in \calP_{0}} \frac{e_i\left(e_{l_i\left(m_1, m_2\right)} - e_{h_i\left(m_1, m_2\right)}\right)^\top}{\exp\left(\Theta_{i, h_i\left(m_1,m_2\right)}^\ast - \Theta_{i, l_i\left(m_1,m_2\right)}^\ast\right) + 1} \;,
\end{align}
where $\calP_0 = \{(i,j) |\;\; 1 \leq i < j \leq k\}$. 
We use the matrix Bernstein inequality \cite{Jo11} for the sum of independent matrices. 
Similar to Lemma \ref{lmm:bundle_partition}, 
we can partition the set of all pairs $\calP_0$ into 
$(k-1)$ sets $\{\calP_a\}_{a\in[k-1]}$ of $k/2$ disjoint pairs each. 
Define $Y_a \equiv \sum_{i = 1}^{d_1} \sum_{(m_1, m_2) \in \calP_a} \tX_{i, m_1, m_2 }$, and 
\begin{align*}
\tX_{i, m_1, m_2}  \;\;\equiv\;\;  \frac{\exp\left(\Theta_{i, l_i\left(m_1,m_2\right)}^\ast\right)}{\exp\left(\Theta_{i,  h_i\left(m_1,m_2\right)}^\ast\right) + \exp\left(\Theta_{i,  l_i\left(m_1,m_2\right)}^\ast\right)} e_i\left(e_{l_i\left(m_1, m_2\right)} - e_{h_i\left(m_1, m_2\right)}\right)^\top\;, 
\end{align*}
such that 
\begin{align}
\nabla \calL_{p}(\Theta^\ast) = \frac{1}{d_1 {k \choose 2}} \sum_{a = 1}^{k-1} \tY_a \;.  \label{eq:pairwise-partition} 
\end{align}

For a fixed value of $a$, it is easy to see that $\tX_{i,m_1,m_2}$'s are independent. Further, we can easily show that $\expect{\tX_{i, m_1, m_2}} = 0$, and $\|\tX_{i, m_1, m_2}\|_2 \leq \sqrt{2}$. We also have,
\begin{align}
&\expect{\tX_{i, m_1, m_2} \tX_{i, m_1, m_2}^\top} \nonumber \\
  & \preceq 2 \,e_ie_i^\top\, \expect{ \expect{\frac{\exp\left(\Theta_{i, l_i\left(m_1,m_2\right)}^\ast\right)^2}{\left(\exp\left(\Theta_{i, u_{i,m_1}}^\ast\right) + \exp\left(\Theta_{i, u_{i,m_2}}^\ast\right)\right)^2}\, \Bigg| \, u_{i,m_1},u_{i,m_1}  }} \nonumber \\
  &\overset{(a)}{=} 2\;e_ie_i^\top\, \expect{\frac{\exp\left(\Theta_{i u_{i,m_1}}^\ast\right)\exp\left(\Theta_{i u_{i,m_2}}^\ast\right)}{\left(\exp\left(\Theta_{i, u_{i,m_1}}^\ast\right) + \exp\left(\Theta_{i, u_{i,m_2}}^\ast\right)\right)^2}} \nonumber \\
  &\overset{(b)}{\preceq} \frac{1}{2} e_i e_i^\top  \label{eq:pairwise-xtx} \;,
\end{align}
where we get $(a)$ from the MNL model for the random choice of $l_i (m_1,m_2) $, $(b)$ is due to the fact that $xy/(x+y)^2 \leq 1/4$ for all $x,y>0$. 

Let $p_{i,m_1, m_2} \equiv \frac{\left(\exp(\Theta_{i, u_{i,m_1}}^\ast)e_{u_{i,m_1}} + \exp(\Theta_{i, u_{i,m_2}}^\ast)e_{u_{i,m_2}}\right)}{\left(\exp(\Theta_{i, u_{i,m_1}}^\ast) + \exp(\Theta_{i, u_{i,m_2}}^\ast)\right)}$, then we have,
\begin{align}
\expect{\tX_{i, m_1, m_2}^\top \tX_{i, m_1, m_2}}
  &= \expect{(e_{h_i\left(m_1,m_2\right)} - p_{i,m_1, m_2})(e_{h_i\left(m_1,m_2\right)} - p_{i,m_1, m_2})^\top} \nonumber \\
  &= \expect{e_{h_i\left(m_1,m_2\right)}e_{h_i\left(m_1,m_2\right)}^\top} - \expect{p_{i,m_1, m_2}p_{i,m_1, m_2}^\top} \nonumber \\
  &\overset{(a)}{\preceq} \expect{e_{u_{i,m_1}}e_{u_{i,m_1}}^\top + e_{u_{i,m_2}}e_{u_{i,m_2}}^\top} \nonumber \\ 
  & =\frac{2}{d_2}{\mathbf I}_{d_2 \times d_2} \label{eq:pairwise-xxt}  \;,
\end{align}
where $(a)$ comes from the fact that $p_{i,m_1, m_2}p_{i,m_1, m_2}^\top$ is a positive semi-definite matrix.
Therefore using \eqref{eq:pairwise-xtx} and \eqref{eq:pairwise-xxt}, we get 

\begin{align}
\sigma^2
  &\equiv 
  \left \lbrace \lnorm{\sum_{\substack{i \in [d_1]\\ (m_1,m_2)\in \calP_a }} \expect{\tX_{i,m_1,m_2}\tX_{i,m_1,m_2}^\top}}{2},\lnorm{\sum_{\substack{i \in [d_1]\\ (m_1,m_2)\in \calP_a }} \expect{\tX_{i,m_1,m_2}^\top \tX_{i,m_1,m_2}}}{2}\right\rbrace \nonumber \\
  &\leq k \max \left\lbrace \frac{1}{4} , \frac{d_1}{d_2}\right\rbrace \;.
\end{align}
Define $\rho \equiv \max \left\lbrace 1/4 , d_1/d_2\right\rbrace$, then  
by the matrix Bernstein inequality \cite{Jo11}, $\forall
\;\;a \in [k-1]$,

\begin{align*}
  {\mathbb P} \Big( \lnorm{\tY_a}{2} > t \Big) \leq (d_1+d_2)
  \exp \Bigg( \frac{- t^2/2}{ k \rho + \sqrt{2} t/3} \Bigg)\;,
\end{align*}
which gives a tail probability of $2d^{-c}/(k-1)$ for the choice of
\begin{align}
  t =  \max \left\{  \sqrt{4 k \rho\, ((1+c)\log d + \log(k-1))} \,,\,  \frac{4\sqrt{2} ((1+c)\log d + \log(k-1))}{3}\right\}\;.
\end{align}
For this choice of $t$, 
 using union bound we can get the probabilistic bound on the derivative of log likelihood as, 
\begin{align}
  {\mathbb P} \Bigg(  \|\nabla \calL_{p}(\Theta^\ast)\|_2 \geq \frac{k-1}{d_1 {k \choose 2}}t \Bigg)
  &\leq {\mathbb P} \Bigg(  \sum_{a=1}^{k-1}\lnorm{\tY_a}{2} \geq (k-1)t \Bigg) \nonumber \\
  &\overset{(a)}{\leq} {\mathbb P} \Bigg(  \max_{a \in [k-1]}\lnorm{\tY_a}{2} \geq t \Bigg) \nonumber \\
  &\overset{(b)}{\leq}  \sum_{a=1}^{k-1} {\mathbb P} \Big(\lnorm{\tY_a}{2} \geq t \Big) \nonumber \\
  & = 2\,d^{-c}\;,
\end{align}
where we obtain $(a)$ by pigeon-hole principle which implies that among a set of numbers, there should be, at the very least one number greater or equal to the average of the set of numbers and $(b)$ by union-bound. Assuming $k \leq  \max\{d_1,d_2^2/d_1\} \log d $ and $d_1 \geq 4$, we have,
\begin{align}
(c+1) \log d + \log(k-1) \leq (c+4) \log d\;, 
\end{align}
from $\log (k-1) \leq \log \left(\max\{d_1,d_2^2/d_1\} \log d\right) \leq \log (\left((d_1^2+d_2^2) \log d )/d_1\right) \leq \log \left((4\;d^2 \log d)/d_1\right) \\ \leq 3 \log d$. 
This proves the desired lemma. 

\subsection{Proof of Lemma \ref{lem:pairwise_hessian}} 
\label{sec:pairwise_hessian_proof}

With a slight abuse of notation, we define $\tH$ as 
\begin{eqnarray}
\tH(\Delta) &\equiv& \frac{1}{d_1 {k \choose 2}}\sum_{i=1}^{d_1} \sum_{(m_1, m_2) \in \calP_0} \Big(\Delta_{i, u_{i,m_1}} - \Delta_{i, u_{i,m_2}}\Big)^2\;  \;,
\end{eqnarray}
and provide a lower bound. The mean  is easily computed as 
\begin{align}
\expect{\tH(\Delta)} 
&= \frac{1}{d_1 {k \choose 2}}\sum_{i=1}^{d_1} \sum_{(m_1, m_2) \in \calP_0}\left[ \frac{2}{d_2}\sum_{j \in [d_2]}\Delta_{ij}^2 - \frac{2}{d_2^2}\sum_{j \in [d_2]}\Delta_{ij} \sum_{j' \in [d_2]}\Delta_{ij'} \right] \nonumber \\
&= \frac{2}{d_1d_2} \fnorm{\Delta}^2\;,
\end{align}
where we used the fact that $\sum_j \Delta_{ij}=0$.
We want to upper bound the probability that $\tH(\Delta) \leq \frac{1}{3d_1d_2}\fnorm{\Delta}^2$ for some $\Delta\;\in\;\calA$. As in the case of k-wise ranking we using the following peeling argument used in \cite[Lemma 3]{NW11}, \cite{Van00}. The strategy is to split this above event as union of many event events as follows. We construct 
the following family of subsets $\{\cal\tilde{S}_{\ell}\}$ such that $\calA \subseteq \cup_{\ell = 1}^{\infty} \cal\tilde{S}_{\ell}$ and,
\begin{align}
{\cal\tilde{S}}_{\ell} = \Bigg\lbrace \Delta \in \mathbb{R}^{d_1 \times d_2}\;\Big|\; &\lnorm{\Delta}{\infty} \leq 2\alpha,\; \beta^{\ell - 1}\mu \leq \fnorm{\Delta} \leq \beta^{\ell}\mu,\; \nonumber \\&\sum_{j \in [d_2]} \Delta_{ij} = 0 \text{ for all } i \in [d_2],\text{ and } \nucnorm{\Delta} \leq \beta^{2 \ell}\mu \Bigg\rbrace\;,
\end{align}
where $\beta = \sqrt{10/9}$ and $\ell \in \{1,2,3,\ldots\}$. This is true since, for any $\Delta \in \calA$, $\fnorm{\Delta}^2 \geq \mu \nucnorm{\Delta}$ and this implies $\fnorm{\Delta}^2 \geq \mu \fnorm{\Delta}$ (or, $\fnorm{\Delta} \geq \mu$). Also note that,
\begin{align}
\tH(\Delta) \leq \frac{1}{3d_1d_2} \fnorm{\Delta}^2
&\implies \frac{2}{d_1d_2}\fnorm{\Delta}^2  - \tH(\Delta) \geq \frac{5}{3d_1d_2} \fnorm{\Delta}^2 \nonumber \\
&\implies \left(\expect{ \tH(\Delta)} - \tH(\Delta)\right) \geq \frac{5}{3d_1d_2} \fnorm{\Delta}^2\;. 
\end{align} 
Therefore using union bound we get,
\begin{align}
&{\mathbb P} \Big(\exists \; \Delta \in \calA \textit{ s.t. } \tH(\Delta) \leq  \frac{1}{3d_1d_2}\fnorm{\Delta}^2 \Big) \nonumber \\
&\leq \sum_{\ell = 1}^\infty {\mathbb P}\left(\sup_{\Delta \in {\cal\tilde{S}}_\ell} (\expect{\tH(\Delta)} - \tH(\Delta)) \geq \frac{5}{3d_1d_2}\fnorm{\Delta}^2 \right) \nonumber \\
&\overset{(a)}{\leq} \sum_{\ell = 1}^\infty {\mathbb P}\left(\sup_{\Delta \in {\cal\tilde{S}}_\ell} (\expect{\tH(\Delta)} - \tH(\Delta)) \geq \frac{3}{2d_1d_2}(\beta^\ell \mu)^2 \right) \nonumber \\
&\overset{(b)}{\leq} \sum_{\ell = 1}^\infty {\mathbb P}\left(\sup_{\Delta \in {\cal\tilde{B}}(\beta^\ell \mu)} (\expect{\tH(\Delta)} - \tH(\Delta)) \geq \frac{3}{2d_1d_2}(\beta^\ell \mu)^2 \right)\;, \label{eq:pairwise-prob-upper}
\end{align}
where $\cal\tilde{B}(D)$ is defined as,
\begin{align}
\label{eq:pairwise-B-set}
{\cal\tilde{B}}(D) =\Bigg\lbrace \Delta \in \mathbb{R}^{d_1 \times d_2}\;\Big|\; &\lnorm{\Delta}{\infty} \leq 2\alpha,\; \fnorm{\Delta} \leq D,\; \nonumber \\&\sum_{j \in [d_2]} \Delta_{ij} = 0 \text{ for all } i \in [d_2],\text{ and } \mu \nucnorm{\Delta} \leq D^2 \Bigg\rbrace\;,
\end{align}
and $(a)$is true because for $\Delta \in {\cal{\tilde{S}}}_l$,
\begin{align}
\frac{5}{3d_1d_2}\fnorm{\Delta}^2 \geq \frac{5}{3d_1d_2}(\beta^{\ell - 1} \mu)^2 = \frac{3}{2d_1d_2}(\beta^\ell\mu)^2\;,
\end{align}
and $(b)$ is true because ${\cal\tilde{S}}_\ell \subset {\cal\tilde{B}}(\beta^\ell {\mu})$. Now we use following lemma to upper bound \eqref{eq:pairwise-prob-upper}.
\begin{lemma}
\label{lem:pairwise_mcdiarmid} 
For $4 (\log d)/3 \leq k \leq d^2 \log d$,
\begin{align}
{\mathbb P}\left(\sup_{\Delta \in {\cal\tilde{B}}(D)} (\expect{\tH(\Delta)} - \tH(\Delta)) \geq \frac{3}{2d_1d_2}D^2 \right) \leq \exp\left(\frac{- k D^4}{2048\;\alpha^4\;d_1 d_2^2}\right)
\end{align}
\end{lemma}
Proof has been relegated to Section \ref{sec:pairwise_mcdiarmid_proof}. Now by \eqref{eq:pairwise-prob-upper} and Lemma \ref{lem:pairwise_mcdiarmid} we get,
\begin{align}
{\mathbb P} \Big(\exists \; \Delta \in \calA \textit{ s.t. } \tH(\Delta) \leq  \frac{1}{3d_1d_2}\fnorm{\Delta}^2 \Big) 
&\leq \sum_{\ell = 1}^\infty \exp\left(\frac{- k \left(\beta^\ell\;\mu \right)^4}{2048\;\alpha^4\;d_1 d_2^2}\right) \nonumber \\
&\overset{(a)}{\leq} \sum_{\ell = 1}^\infty \exp\left(\frac{- 2^{13} \; 9 \;\beta^{4\ell} \; d_1 d_2^2 \log^2 d}{k\; \min^2\{d_1, d_2\}} \right) \nonumber \\
&\overset{(b)}{\leq} \sum_{\ell = 1}^\infty \exp\left(\frac{- 2^{13}\;9\;4\ell \times \frac{1}{36} \; d_1 d_2^2 \log^2 d}{k\; \min^2\{d_1, d_2\}} \right) \nonumber \\
&\overset{(c)}{\leq} \sum_{\ell = 1}^\infty \exp\left(- 2^{13} \; \ell \; \log d\right) \nonumber \\
&= \sum_{\ell = 1}^\infty \left(\frac{1}{d^{2^{13}}}\right)^\ell \nonumber \\
&\overset{(d)}{=} \frac{1/d^{2^{13}}}{1 - 1/d^{2^{13}}} \nonumber \\
&\overset{(e)}{\leq} \frac{2}{d^{2^{13}}} \;,
\end{align}
where we get $(a)$ by substituting $\mu$ from \eqref{eq:pairwise-mu}, $(b)$ by the fact that for $\beta = \sqrt{10/9}$ and $x \geq 1$, $\beta^x \geq x \log \beta \geq x (\beta - 1) \geq x/32$, $(c)$ is obtained by assuming $k \leq \max\{d_1,d_2^2/d_1\} \log d$, we get $(d)$ because we are summing an infinite geometric sequence with common ratio of $1/d^{2^{13}}$ and $(e)$ is because for $d \geq 2$, $1/d^{2^{13}}$ is less than $1/2$.

\subsection{Proof of Lemma \ref{lem:pairwise_mcdiarmid}} 
\label{sec:pairwise_mcdiarmid_proof}

With a slight abuse of notations, let $\tilde{Z} \equiv  \sup_{\Delta \in {\cal\tilde{B}}(D)} \left(\expect{\tH(\Delta)} - \tH(\Delta) \right)$. Notice that $\tZ$ is a function of $d_1 k$ random variables, $\{u_{i,\ell}\}_{i \in [d_1], \ell \in [k]}$. 
We apply the McDiarmid's bounded differences inequality. 
Let $\tZ_1$ and $\tZ_2$ be two realizations of $\tZ$ where value of only one random variable $ u_{i',\ell'} $ is changed to $ u'_{i',\ell'} $. Also with a little more abuse of notation the two realizations of $\tH(\Delta)$ are written as $\tH(\Delta', u_{1,1}, \ldots, u_{i',\ell'}, \ldots, u_{d_1,k})$ and $\tH(\Delta', u_{1,1}, \ldots, u'_{i',\ell'}, \ldots, u_{d_1,k})$. We let $\Delta^*$ be the maximizer of $\max\{\tZ_1,\tZ_2\}$. Maximum absolute difference between them is upper bounded as follows,
\begin{align}
&|\tZ_1 - \tZ_2| \nonumber \\
= &\Bigg| \max_{\Delta \in {\cal\tilde{B}}(D)} \left(\expect{\tH(\Delta)} - \tH(\Delta, u_{1,1}, \ldots, u_{i',\ell'}, \ldots, u_{d_1,k}) \right)  - \nonumber \\
&\sup_{\Delta' \in {\cal\tilde{B}}(D)} \left(\expect{\tH(\Delta')} - \tH(\Delta', u_{1,1}, \ldots, u'_{i',\ell'}, \ldots, u_{d_1,k}) \right)\Bigg| \nonumber \\
\overset{(a)}{\le} &\Bigg| \left(\expect{\tH(\Delta^*)} - \tH(\Delta^*, u_{1,1}, \ldots, u_{i',\ell'}, \ldots, u_{d_1,k}) \right) - \nonumber\\ &\left(\expect{\tH(\Delta^*)} - \tH(\Delta^*, u_{1,1}, \ldots, u'_{i',\ell'}, \ldots, u_{d_1,k}) \right)\Bigg| \nonumber \\
\overset{}{\le} & \sup_{\Delta \in {\cal\tilde{B}}(D)} \Bigg| \tH(\Delta, u_{1,1}, \ldots, u_{i',\ell'}, \ldots, u_{d_1,k})- \tH(\Delta, u_{1,1}, \ldots, u'_{i',\ell'}, \ldots, u_{d_1,k}) \Bigg| \nonumber \\
\overset{(b)}{\le} & \sup_{\Delta \in {\cal\tilde{B}}(D)} \Bigg| \frac{1}{d_1{k \choose 2}} \sum_{\ell \neq \ell'} \left(\Delta_{i',u_{i',\ell}} - \Delta_{i',u_{i',\ell'}}\right)^2 - \left( \Delta_{i',u_{i',\ell}} - \Delta_{i',u'_{i',\ell'}} \right)^2 \Bigg| \nonumber \\
\overset{(c)}{\le} &\frac{1}{d_1{k \choose 2}} (k-1)  \left( 4 \alpha \right)^2  = \frac{32 \alpha^2}{d_1k} \label{eq:pairwise-bdd-diff}.
\end{align}
where $(a)$ follows from the fact that $\Delta^*$ is maximizer of $\max\{\tZ_1,\tZ_2\}$, $(b)$ is due to the fact that the terms which change because of $u'_{i',\ell'}$ are the $k-1$ difference square terms between $\Delta_{i u_{i',\ell \neq \ell'}}$ and $\Delta_{i,\; u_{i',\ell'}}$ and $(c)$ is because maximum and minimum value of difference square terms are $(4\alpha)^2$ and $0$ respectively. Using McDiarmid's bounded differences inequality we get,
\begin{align}
{\mathbb P}\{\tZ - \expect{\tZ} \geq \epsilon\} \leq \exp\left( - \frac{2 \epsilon^2}{d_1k \left(\frac{32 \alpha^2}{d_1 k}\right)^2 } \right)\;, \label{eq: pairwise-mcdiarmid-bound}
\end{align}
because of \eqref{eq:pairwise-bdd-diff} and the fact that there are $d_1 k$ random variables. We upper bound $\expect{\tZ}$ as follows.
\begin{align}
\expect{\tZ} &\overset{}{=} \mathbb{E}\; \sup_{\Delta \in {\cal\tilde{B}}(D)} \frac{1}{d_1 {k \choose 2}}\sum_{i=1}^{d_1} \sum_{(m_1, m_2) \in \calP_0} \expect{\Big(\Delta_{i,\; u_{i,m_1}} - \Delta_{i,\; u_{i,m_2}}\Big)^2} - \Big(\Delta_{i,\; u_{i,m_1}} - \Delta_{i,\; u_{i,m_2}}\Big)^2 \nonumber \\
&\overset{(a)}{\leq} \mathbb{E}\; \sup_{\Delta \in {\cal\tilde{B}}(D)} \frac{1}{d_1 {k \choose 2}}\sum_{i=1}^{d_1} \sum_{(m_1, m_2) \in \calP_0} 2\txi_{i,m_1,m_2} \Big(\Delta_{i,\; u_{i,m_1}} - \Delta_{i,\; u_{i,m_2}}\Big)^2 \nonumber \\
&\overset{(b)}{\leq} \mathbb{E}\; \sup_{\Delta \in {\cal\tilde{B}}(D)} \frac{1}{d_1 {k \choose 2}}\sum_{i=1}^{d_1} \sum_{a=1}^{k-1} \sum_{(m_1, m_2) \in \calP_a} 2\txi_{i,m_1,m_2} \Big(\Delta_{i,\; u_{i,m_1}} - \Delta_{i,\; u_{i,m_2}}\Big)^2 \nonumber \\
&\overset{(c)}{\leq}  \sum_{a=1}^{k-1}  \mathbb{E}\; \sup_{\Delta \in {\cal\tilde{B}}(D)} \frac{1}{d_1 {k \choose 2}}\sum_{i=1}^{d_1} \sum_{(m_1, m_2) \in \calP_a} 2\txi_{i,m_1,m_2} \Big(\Delta_{i,\; u_{i,m_1}} - \Delta_{i,\; u_{i,m_2}}\Big)^2 \label{eq:pairwise-mcd-expect}\;,
\end{align}
where $(a)$ is by standard symmetrization technique as used in k-wise ranking and  \\$\{\xi_{i,m_1,m_2}\}_{i \in [d_1],\; m_1,m_2 \in [k]}$ are i.i.d. Rademacher variables, $(b)$ is due to the fact that we can partition set of all pairs into $k-1$ independent sets as in \eqref{eq:pairwise-partition} and $(c)$ is because of fact that supremum of sum is less than or equal to sum of supremum and the linearity of expectation. Since $|\Delta_{i,\; u_{i,m_1}} - \Delta_{i,\; u_{i,m_2}}| \leq 4\alpha$, we can use Ledoux-Talagrand contraction inequality \cite{ledoux2013probability} on \eqref{eq:pairwise-mcd-expect} to get,
\begin{align}
E[\tZ] &\leq \sum_{a=1}^{k-1}  \mathbb{E}\; \sup_{\Delta \in {\cal\tilde{B}}(D)} \frac{1}{d_1 {k \choose 2}}\sum_{i=1}^{d_1} \sum_{(m_1, m_2) \in \calP_a} 2\txi_{i,m_1,m_2} \Big(\Delta_{i,\; u_{i,m_1}} - \Delta_{i,\; u_{i,m_2}}\Big)^2 \nonumber \\
&\leq \sum_{a=1}^{k-1}  \mathbb{E}\; \sup_{\Delta \in {\cal\tilde{B}}(D)} \frac{1}{d_1 {k \choose 2}}\sum_{i=1}^{d_1} \sum_{(m_1, m_2) \in \calP_a} 4\alpha \; 2\txi_{i,m_1,m_2} \Big(\Delta_{i,\; u_{i,m_1}} - \Delta_{i,\; u_{i,m_2}}\Big) \nonumber \\
&\overset{(a)}{\leq} \sum_{a=1}^{k-1} \frac{8 \alpha}{d_1 {k \choose 2}} \mathbb{E}\; \sup_{\Delta \in {\cal\tilde{B}}(D)} {\llangle \sum_{i=1}^{d_1} \sum_{(m_1, m_2) \in \calP_a} \tW_{i,m_1,m_2}, \Delta \rrangle} \nonumber \\
&\overset{(b)}{\leq} \sum_{a=1}^{k-1} \frac{8 \alpha}{d_1 {k \choose 2}} \expect{\lnorm{\sum_{i=1}^{d_1} \sum_{(m_1, m_2) \in \calP_a} \tW_{i,m_1,m_2}}{2}} \sup_{\Delta \in {\cal\tilde{B}}(D)} \nucnorm{\Delta}\;, \label{eq:pairwise-group-expect}
\end{align}
where we get $(a)$ by putting $\tW_{i,m_1,m_2} = \txi_{i,m_1,m_2} e_i (e_{u_{i,m_1}} - e_{u_{i,m_2}})^\top$ and $(b)$ is due to H\"older's inequality $\left(\llangle x, y\rrangle \leq \lnorm{x}{2} \nucnorm{y} \right)$. Now we use Bernstein's inequality \cite{Jo11} to upperbound the above expectation terms. First fix $a$ to value in $[k-1]$. We can easily show that $\tW_{i,m_1,m_2}$ is zero mean and, 
\begin{align}
\label{eq: pairwise-deterministic-bdd}
\lnorm{\tW_{i,m_1,m_2}}{2} \leq \sqrt{2}\;.
\end{align}

We also get,
\begin{align}
\expect{\tW_{i, m_1, m_2} \tW_{i, m_1, m_2}^\top} &= 2e_ie_i^\top \expect{1 - e_{u_{i,m_1}}^\top e_{u_{i,m_2}}} \nonumber \\
&\preceq e_ie_i^\top \left(2 - \frac{2}{d_2}\right) \nonumber \\
&\preceq 2e_ie_i^\top \label{eq:pairwise-ztz} \;,
\end{align}
and,
\begin{align}
\expect{\tW_{i, m_1, m_2}^\top \tW_{i, m_1, m_2}} &= \expect{2 e_{u_{i,m_1}}e_{u_{i,m_1}}^\top - 2 e_{u_{i,m_1}}e_{u_{i,m_2}}^\top} \nonumber \\
&\preceq \frac{2}{d_2}{\mathbf I}_{d_2 \times d_2}  - \frac{2}{d_2^2}\mathbf{11}_{d_2 \times d_2} \nonumber \\
&\preceq \frac{2}{d_2}{\mathbf I}_{d_2 \times d_2} \label{eq:pairwise-zzt}  \;.
\end{align}

Therefore, using \eqref{eq:pairwise-ztz} and \eqref{eq:pairwise-zzt}, the standard deviation of $\sum_{(i,m_1,m_2)} Z_{i,m_2,m_2}$ is,

\begin{align}
\sigma^2
  &= \max \left\lbrace \lnorm{\sum_{\substack{i \in [d_1] \\(m_1,m_2) \in \calP_a}} \expect{\tW_{i,m_2,m_2}\tW_{i,m_2,m_2}^\top}}{2}, \lnorm{\sum_{\substack{i \in [d_1] \\(m_1,m_2) \in \calP_a}} \expect{\tW_{i,m_2,m_2}^\top \tW_{i,m_2,m_2}}}{2}\right\rbrace \nonumber \\
  &\leq \max \left\lbrace \frac{d_1 k}{2} \frac{2}{d_1} \lnorm{{\mathbf I}}{2}, \frac{d_1 k}{2} \frac{2}{d_2}\lnorm{{\mathbf I}}{2}\right\rbrace \nonumber \\
  &= \frac{kd_1}{\min\{d_1,d_2\}}\;. 
\end{align}

By matrix Bernstein inequality \cite{Jo11}, $\forall
\;\;a \in [k-1]$,

\begin{align*}
  {\mathbb P} \left(\lnorm{\sum_{i \in [d_1]} \sum_{(m_1,m_2) \in \calP_a} \tW_{i,m_2,m_2}}{2} > t \right) \leq (d_1+d_2)
  \exp \Big( \frac{- t^2/2}{ 2k d_1 / \min\{d_1,d_2\}+ \sqrt{2} t/3} \Big)\;,
\end{align*}
which gives a tail probability of $2d^{-c_1}$ for the choice of
\begin{align}
  t &= \max \left\{ \sqrt{\frac{8 k d_1 \left((1+c_1)\log d \right)}{\min\{d_1,d_2\}}} \,,\,  \frac{4\sqrt{2} \left((1+c_1)\log d \right) }{3}\right\} \nonumber \\
  & = \sqrt{\frac{8 k d_1 \left((1+c_1)\log d \right) }{\min\{d_1,d_2\}}}, \text{when $k \geq 4(c_1 + 1) \log d / 9$}\;.
\end{align}

Therefore $\forall$ $a \in [k-1]$,
\begin{align}
\expect{\lnorm{\sum_{i=1}^{d_1} \sum_{(m_1, m_2) \in \calP_a} \tW_{i,m_2,m_2}}{2}} \leq \sqrt{\frac{8 k d_1 \left((1+c_1)\log d\right) }{\min\{d_1,d_2\}}} \;\; + \;\; \frac{2}{d^{c_1}} \frac{\sqrt{2}d_1 k}{2}\;, \label{eq:pairwise-single-expect}
\end{align}
because from \eqref{eq: pairwise-deterministic-bdd} we get $\lnorm{\sum_{\substack{i\in[d_1]\\ (m_1, m_2) \in \calP_a}} \tW_{i,m_2,m_2}}{2} \leq \sum_{\substack{i\in[d_1]\\ (m_1, m_2) \in \calP_a}} \lnorm{\tW_{i,m_2,m_2}}{2} \leq \frac{d_1 k}{2(\sqrt{2})}$. From \eqref{eq:pairwise-group-expect} and \eqref{eq:pairwise-single-expect}, putting $c_1 = 2$, we get,
\begin{align}
\expect{\tZ} &\overset{}{\leq}\sum_{a=1}^{k-1} \frac{8 \alpha}{d_1 {k \choose 2}} \left( \sqrt{\frac{24\; k d_1 \log d}{\min\{d_1,d_2\}}} \;\; + \;\; \frac{\sqrt{2}d_1 k}{{d^2}} \right) \sup_{\Delta \in {\cal\tilde{B}}(D)} \nucnorm{\Delta} \nonumber \\
&\overset{(a)}{\leq} 8\alpha \left(2 \sqrt{\frac{24\; \log d}{k\;d_1 \min\{d_1,d_2\}}} \;\; + \;\; \frac{2\sqrt{2}}{d^2} \right) \frac{D^2}{\mu} \nonumber \\
&\overset{(b)}{\leq} 16\alpha \sqrt{\frac{48 \log d}{k\;d_1 \min\{d_1,d_2\}}} D^2 \frac{1}{16\alpha} \sqrt{\frac{k\; \min\{d_1,d_2\}}{48 d_1 d_2^2 \log d}} \nonumber \\
&= \frac{D^2}{d_1 d_2}\;,
\end{align}
where $(a)$ is obtained because of \eqref{eq:pairwise-B-set} which gives $\sup_{D \in \calB(D)} \nucnorm{\Delta} \leq D^2/\mu$ and $(b)$ can be got by assuming $k \leq d^2 \log d$. Using the above bound in \eqref{eq: pairwise-mcdiarmid-bound} we get,
\begin{align}
{\mathbb P}\{\tZ - D^2/(d_1 d_2) \geq \epsilon\} \leq {\mathbb P}\{\tZ - \expect{\tZ} \geq \epsilon\} \leq  \exp\left( - \frac{2 \epsilon^2}{d_1k \left(\frac{32 \alpha^2}{d_1 k}\right)^2 } \right)\;,
\end{align}
and using $\epsilon = D^2/(2d_1 d_2)$ will get us the required bound.

\section{Proof of Bundled Choices Theorem \ref{thm:bundle_ub}} 
\label{sec:bundle_ub_proof}

We use similar notations and techniques as the proof of Theorem \ref{thm:kwise_ub} in Appendix \ref{sec:kwise_ub_proof}. 
From the definition of $\cL(\Theta)$ in Eq. \eqref{eq:defbundleL}, we have for the true parameter $\Theta^*$, the gradient evaluated at the true parameter is 
\begin{eqnarray}
  \nabla \cL(\Theta^*) &=& -\frac{1}{n}\sum_{i=1}^n (e_{u_i} e_{v_i}^T - p_i) \;, 
\end{eqnarray}
where $p_i$ denotes the conditional probability of the MNL choice for the $i$-th sample. Precisely, 
$p_i=\sum_{j_1\in S_i}\sum_{j_2\in T_i} p_{j_1,j_2|S_i,T_i} e_{j_1}e_{j_2}^T$ where 
$p_{j_1,j_2|S_i,T_i}$ is the probability that the pair of items $(j_1,j_2)$ is chosen 
 at the $i$-th sample such that $p_{j_1,j_2|S_i,T_i} \equiv \prob{(u_i,v_i)=(j_1,j_2)| S_i,T_i} = e^{\Theta^*_{j_1,j_2}}/(\sum_{j'_1\in S_i,j'_2\in T_i} e^{\Theta^*_{j_1',j_2'}})$, 
 where $(u_i,v_i)$ is the pair of items selected by the $i$-th user among the set of pairs of  alternatives $S_i\times T_i$. 
 The Hessian can be computed as 
 \begin{align}
 &  \frac{\partial^2 \cL(\Theta)}{\partial\Theta_{j_1,j_2}\,\partial\Theta_{j_1',j_2'}} \;=\; \frac{1}{n} \sum_{i=1}^n \ind\big( (j_1,j_2) \in S_i\times T_i \big) \frac{\partial p_{j_1,j_2|S_i,T_i}}{\partial \Theta_{j_1',j_2'}}\\
    & \; = \frac{1}{n} \sum_{i=1}^n  \ind\big( (j_1,j_2),(j_1',j_2') \in S_i\times T_i \big) \,\Big( p_{j_1,j_2|S_i,T_i} \ind((j_1,j_2)=(j_1',j_2')) -
      p_{j_1,j_2|S_i,T_i}p_{j_1',j_2'|S_i,T_i} \Big)\;, \label{eq:bundle_hess} 
 \end{align}
We use $\nabla^2\cL(\Theta) \in \reals^{d_1d_2\times d_1d_2}$ to denote this Hessian. 
Let $\Delta = \Theta^*-\hTheta$ where $\hTheta$ is an optimal solution to the convex optimization in \eqref{eq:bundleopt}. 
We introduce the following  key technical lemmas. The following lemma provides a bound on the gradient
using the concentration of measure for sum of independent random matrices \cite{Jo11}. 
\begin{lemma}\label{lmm:bundle_gradient2}
  For any positive constant $c \ge 1$ and $n \geq  (4(1+c)e^{2\bb} d_1d_2 \log d)/\max\{d_1,d_2\}$, with probability at least $1-2 d^{-c}$,
  \begin{eqnarray}
    \lnorm{\nabla \calL(\Theta^\ast)}{2} &\leq&
    \sqrt{\frac{ 4(1+c) e^{2\bb} \max\{d_1,d_2\} \, \log d} { d_1\,d_2\,n  }} 
    \;.
  \end{eqnarray}
\end{lemma}
Since we are typically interested in the regime where 
the number of samples is much smaller than the dimension $d_1\times d_2$ of the problem, 
the Hessian is typically not positive definite. However, when we restrict our attention to the vectorized $\Delta$ with 
relatively small nuclear norm, then we can prove restricted strong convexity, which gives the following bound. 
\begin{lemma}[{\bf Restricted Strong Convexity for bundled choice modeling}]
\label{lmm:bundle_hessian2}
Fix any $\Theta \\\in \Omega_\bb$ and assume $ (\min \{d_1,d_2\}/\min\{k_1,k_2\})\log d \leq n \leq \min\{d^5 \log d, k_1 k_2 \max\{d_1^2,d_2^2\}   \log d\} $.
  Under the random sampling model of the alternatives $\{j_{i a}\}_{i\in[n],a\in[k_1]}$ from the first set of items $[d_1]$, 
  $\{j_{i b}\}_{i\in[n],b\in[k_1]}$ from the second set of items $[d_2]$ and
  the random outcome of the comparisons described in section \ref{sec:intro},
  with probability larger than $1-2d^{-2^{25}}$,
  \begin{eqnarray}
    {\rm Vec}(\Delta)^\top \,\nabla^2\calL(\Theta)\, {\rm Vec}(\Delta)  &\geq& \frac{e^{-2\bb}}{8\, d_1\,d_2} \fnorm{\Delta}^2\;,
    \label{eq:bundle_hessian2}
  \end{eqnarray}
  for all $\Delta$ in $\calA$ where
  \begin{eqnarray}
    \calA = \Big\{ \Delta \in \reals^{d_1\times d_2} \,\big|\, \lnorm{\Delta}{\infty} \leq 2\bb\,, \, \sum_{j_1\in[d_1],j_2\in[d_2]}\Delta_{j_1j_2} = 0\,\text{ and } \fnorm{\Delta}^2 \geq \mu' \nucnorm{\Delta}
    \Big\} \;. \label{eq:bundle_defA}
  \end{eqnarray}
  with 
  \begin{eqnarray}
    \mu' &\equiv& 2^{10}\, \bb\,d_1 d_2 \sqrt{\frac{ \log d}{n\,\min\{d_1,d_2\}\,\min\{k_1,k_2\} }} \;.
  \label{eq:bundle_defmu}
  \end{eqnarray}
\end{lemma}

Building on these lemmas, the proof of Theorem \ref{thm:bundle_ub} is divided into the following two cases.
In both cases, we will show that 
\begin{align}
  \fnorm{\Delta}^2  \;\leq\;  12 \, e^{2\bb}  c_1 \lambda \,d_1d_2\,  \nucnorm{\Delta} \;, 
  \label{eq:bundle_errorfrobeniusbound}
\end{align}
 with high probability. Finally, applying an omitted result similar to Lemma \ref{lmm:kwise_deltabound} proves the desired theorem. We are left to show Eq. \eqref{eq:bundle_errorfrobeniusbound} holds.

\bigskip
\noindent{\bf Case 1: Suppose $\fnorm{\Delta}^2 \geq  \mu' \,\nucnorm{\Delta} $.}
With $\Delta=\Theta^* - \hTheta$, the Taylor expansion yields
\begin{align}
\calL(\widehat{\Theta})=\calL(\Theta^\ast) -  {\llangle \nabla \calL(\Theta^\ast), \Delta \rrangle} + \frac{1}{2} {\rm Vec}(\Delta) \nabla^2\calL(\Theta) {\rm Vec}^\top (\Delta),
\end{align}
where $\Theta=a \widehat{\Theta} + (1-a) \Theta^\ast$ for some $a \in [0,1]$.
It follows from Lemma \ref{lmm:bundle_hessian2} that with probability at least $1-2d^{-2^{25}}$,
\begin{align*}
   \calL(\widehat{\Theta}) -\calL(\Theta^\ast) 
  &\; \ge\; - \lnorm{\nabla\calL(\Theta^\ast)}{2} \nucnorm{\Delta}+ \frac{  e^{-2\bb}}{8 \,d_1\,d_2} \fnorm{\Delta}^2\;.
\end{align*}
From the definition of $\widehat{\Theta}$ as an optimal solution of the minimization, we have
\begin{align*}
\calL(\widehat{\Theta})-  \calL(\Theta^\ast)  \;\le \; \lambda \left(  \nucnorm{\Theta^\ast} - \nucnorm{\widehat{\Theta}} \right) \;\le\; \lambda \nucnorm{\Delta}\;.
\end{align*}
By the assumption, we choose $\lambda\geq 8 \lambda_0$. 
 In view of Lemma \ref{lmm:bundle_gradient2}, this implies that 
  $\lambda \geq 2\lnorm{\nabla\calL (\Theta^*)}{2} $ 
  with probability at least $1-2d^{-3}$. 
   It follows that with probability at least $1-2d^{-3}- 2d^{-2^{25}}$,
\begin{align*}
\frac{ e^{-2\bb}}{8 d_1 d_2 } \fnorm{\Delta}^2 \;  \leq \;  \big(\lambda + \lnorm{\nabla\calL(\Theta^*)}{2}\big)\,  \nucnorm{\Delta} \;\leq\; \frac{3 \lambda}{2}  \nucnorm{\Delta} \;.
\end{align*}
By our assumption on $\lambda \leq c_1 \lambda_0$, this proves the desired bound in Eq. \eqref{eq:bundle_errorfrobeniusbound}

\noindent{\bf Case 2:  Suppose $\fnorm{\Delta}^2 \leq  \mu' \,\nucnorm{\Delta} $.}  
By the definition of $\mu$ and the fact that $c_1 \ge 128/\sqrt{\min\{k_1,k_2\}}$, it follows that 
$\mu' \le 12  \, e^{2\bb}  c_1 \lambda   \,d_1d_2$, and we get the same bound as in Eq. \eqref{eq:bundle_errorfrobeniusbound}.

\subsection{Proof of Lemma \ref{lmm:bundle_gradient2}}

Define $X_i=-(e_{u_i}e_{v_i}^T-p_i)$ such that $\nabla\cL(\Theta^*) = (1/n)\sum_{i=1}^n X_i$, which is a sum of 
$n$ independent random matrices. Note that since $p_i$ is entry-wise bounded by $e^{2\bb}/(k_1k_2)$, 
\begin{eqnarray*}
  \lnorm{X_i}{2} &\leq& 1 + \frac{e^{2\bb}}{\sqrt{k_1 k_2}} \;,
\end{eqnarray*}
and 
\begin{eqnarray}
  \sum_{i=1}^n \E[X_iX_i^T]  &= &  \sum_{i=1}^n (\E[e_{u_i}e_{u_i}^T] - p_ip_i^T)\\
    &\preceq& \sum_{i=1}^n \E[e_{u_i}e_{u_i}^T]\\
    & \preceq&\frac{e^{2\bb}\,n}{d_1} \id_{d_1\times d_1}\;,
\end{eqnarray}
where the last inequality follows from the fact that for any given $S_i$, 
$u_i$ will be chosen with probability at most $e^{2\bb}/k_1$, 
if it is in the set $S_i$ which happens with probability $k_1/d_1$.
Therefore, 
\begin{eqnarray}
  \lnorm{\sum_{i=1}^n \E[X_iX_i^T]}{2}  &\leq &  \frac{e^{2\bb} \,n}{d_1} \;.
\end{eqnarray}
Similarly, 
\begin{eqnarray}
  \lnorm{\sum_{i=1}^n \E[X_i^TX_i]}{2}  &\leq &  \frac{e^{2\bb} \,n}{d_2} \;.
\end{eqnarray}
Applying matrix Bernstein inequality \cite{Jo11}, we get 
\begin{align}
  &\prob{ \lnorm{\nabla \cL(\Theta^*)}{2} > t} \nonumber \\&\leq (d_1+d_2) \exp\Big\{ \frac{- n^2 t^2/2}{(e^{2\bb}n\max\{d_1,d_2\}/(d_1d_2)) \,+\, ( (1 + (e^{2\bb}/\sqrt{k_1k_2}))n t /3)} \Big\}\;, 
\end{align}
which gives the desired tail probability of $2d^{-c}$ for the choice of 
\begin{eqnarray*}
  t &=& \max \Big\{ \sqrt{\frac{4(1+c) e^{2\bb} \max\{d_1,d_2\} \log d}{d_1d_2n}} \,,\, \frac{4(1+c) (1+\frac{e^{2\bb}}{\sqrt{k_1k_2}})\log d}{3n} \Big\} \\
  &=&\sqrt{\frac{4(1+c) e^{2\bb} \max\{d_1,d_2\} \log d}{d_1d_2n}} \;,
\end{eqnarray*}
where the last equality follows from the assumption, 
$n\geq (4(1+c) e^{2\bb} d_1d_2 \log d )/\max\{d_1,d_2\}$.

\subsection{Proof of  Lemma \ref{lmm:bundle_hessian2}}
Thee quadratic form of the Hessian defined in \eqref{eq:bundle_hess} can be lower bounded by 
\begin{eqnarray}
  {\rm Vec}(\Delta)^T\, \nabla^2\cL(\Theta)\, {\rm Vec}(\Delta)  &\geq & \underbrace{\frac{e^{-2\bb}}{2 \,k_1^2\,k_2^2\,n}\sum_{i=1}^n  \sum_{j_1,j_1'\in S_i}\sum_{j_2,j_2'\in T_i} \big( \Delta_{j_1,j_2}-\Delta_{j_1',j_2'} \big)^2  }_{\equiv H(\Delta)}\;,
  \label{eq:bundle_defH} 
\end{eqnarray}
which follows from Remark \ref{rem:hess}. 
To lower bound $H(\Delta)$, we first compute the mean: 
\begin{eqnarray}
  \E[H(\Delta)] &=& \frac{e^{-2\bb}}{2 \,k_1^2\,k_2^2\,n}\sum_{i=1}^n \E\big[ \sum_{j_1,j_1'\in S_i}\sum_{j_2,j_2'\in T_i} \big( \Delta_{j_1,j_2}-\Delta_{j_1',j_2'} \big)^2 \big]\\
  &=& \frac{e^{-2\bb}}{ \,d_1\,d_2}\fnorm{\Delta}^2 \;,
\end{eqnarray}
where we used the fact that $\E[\sum_{j_1\in S_i,j_2\in T_i}\Delta_{j_1,j_2}]=(k_1k_2/(d_1d_2))\sum_{j_1'\in[d_1],j_2'\in[d_2]}\Delta_{j_1',j_2'}=0$ for $\Delta \in \Omega_{2\bb}$ in \eqref{eq:defbundleomega}. 

We now prove that $H(\Delta)$ does not deviate from its mean too much. 
Suppose there exists a $\Delta \in \calA$ defined in \eqref{eq:bundle_defA} such that 
Eq. \eqref{eq:bundle_hessian2} is violated, i.e. $ H(\Delta) < (e^{-2\bb}/(8 d_1d_2))\fnorm{\Delta}^2$. 
In this case, 
\begin{eqnarray}
  \E[H(\Delta)] - H(\Delta)  &\geq& \frac{7\,e^{-2\bb}}{8 d_1d_2} \fnorm{\Delta}^2 \;.
\end{eqnarray}
We will show that this happens with a small probability. 
We use the same peeling argument as in Appendix \ref{sec:kwise_ub_proof} with 
\begin{align*}
  \calS_\ell = \Big\{ \Delta\in\reals^{d_1\times d_2} \,|\,& \lnorm{\Delta}{\infty} \leq 2\bb, \beta^{\ell-1}\mu' \leq \fnorm{\Delta}\leq\beta^\ell \mu', \nonumber\\ &\sum_{j_1\in[d_1],j_2\in[d_2]}\Delta_{j_1,j_2}=0, \text{ and }\nucnorm{\Delta}\leq \beta^{2\ell} \mu' \Big\}\;,
\end{align*}
where $\beta=\sqrt{10/9}$ and for $\ell\in\{1,2,3,\ldots\}$, and $\mu'$ is defined in \eqref{eq:bundle_defmu}.
By the peeling argument, there exists an $\ell\in \Z_+$ such that $\Delta\in\calS_\ell$ and 
\begin{eqnarray}
  \E[H(\Delta)] - H(\Delta)  &\geq& \frac{7\,e^{-2\bb}}{8 d_1d_2} \beta^{2\ell-2}(\mu')^2  \;\geq\;  \frac{7\,\,e^{-2\bb}}{9\, d_1d_2} \beta^{2\ell}(\mu')^2 \;.
\end{eqnarray}
Applying the union bound over $\ell\in\Z_+$, 
\begin{align}
  &\prob{\exists \Delta \in \calA \;,\; H(\Delta) < \frac{ e^{-2\bb} } { 8\, d_1\, d_2 } \fnorm{\Delta}^2 }\; \nonumber \\&\leq\;
    \sum_{\ell=1}^\infty \prob{ \sup_{\Delta\in\calS_\ell} \big(\; \E[ H(\Delta)] - H(\Delta) \;\big) > \frac{7\,e^{-2 \bb} }{ 9 d_1 d_2}(\beta^\ell \mu' )^2 } \nonumber\\
    &\hspace{3.6cm}\leq\;\; \sum_{\ell=1}^\infty \prob{ \sup_{\Delta\in\calB'(\beta^\ell\mu')} \big(\; \E[ H(\Delta)] - H(\Delta) \;\big) > \frac{7e^{-2 \bb}  }{ 9   d_1 d_2}(\beta^\ell \mu')^2 }
    \;, \label{eq:bundle_peeling3}
\end{align}
where we define the set $\calB'(D)$ such that $\calS_\ell \subseteq \calB'(\beta^\ell \mu')$: 
\begin{align}
  \calB'(D) = \big\{\, \Delta \in \reals^{d_1\times d_2} \,\big|\, \|\Delta\|_\infty \leq 2\bb, \fnorm{\Delta}\leq D, \sum_{j_1\in[d_1],j_2\in[d_2]} \Delta_{j_1j_2}=0,  \mu' \nucnorm{\Delta} \leq D^2 \,
      \Big\} \;. \label{eq:bundle_defcB}
\end{align}
The following key lemma provides the upper bound on this probability.  

\begin{lemma}
For $ (\min\{d_1,d_2\}/\min\{k_1,k_2\}) \log d \leq n \leq d^5\log d$, 
\begin{align}
  \prob{\sup_{\Delta\in\calB'(D)} \Big(\; \E[ H(\Delta)] - H(\Delta) \;\Big) \geq \frac{e^{-2\bb} D^2}{2 d_1d_2} } &\leq& \exp\Big\{ - \frac{ n\,\min\{k_1^2,k_2^2\}\,k_1k_2\,D^4}{2^{10} \bb^4 d_1^2 d_2^2 } \Big\}\;.
  \label{eq:bundle_hessian3}
\end{align}
\label{lmm:bundle_hessian3}
\end{lemma}

Let $\eta= \exp\left(-\frac{nk_1k_2 \min\{k_1^2,k_2^2\} (\beta-1.002) (\mu')^4}{2^{10} \bb^4d_1^2 d_2^2} \right)$. 
Applying the tail bound to \eqref{eq:bundle_peeling3}, we get 
\begin{eqnarray*}
  \prob{\exists \Delta \in \calA \;,\; H(\Delta) < \frac{ e^{-2\bb} } { 8  \,d_1 d_2 } \fnorm{\Delta}^2 } 
  &\leq & \sum_{\ell=1}^\infty \exp \Big\{
  - \frac{ n\,k_1k_2\,\min\{k_1^2,k_2^2\}\,(\beta^\ell \mu')^4}{2^{10} \bb^4 d_1^2 d_2^2
  }\Big\} \\
  & \overset{(a)}{\leq} &\sum_{\ell=1}^\infty \exp\Big\{-\frac{n k_1k_2\min\{k_1^2,k_2^2\}\ell (\beta-1.002) (\mu')^4 }{2^{10} \bb^4 d_1^2 d_2^2}\Big\} \\
  &\leq & \frac{\eta}{1-\eta},
\end{eqnarray*}
where $(a)$ holds because $\beta^{x} \geq x \log\beta \ge x(\beta-1.002)$ for the choice of $\beta=\sqrt{10/9}$.
By the definition of $\mu'$,
\begin{align*}
\eta \; = \; \exp\Big\{ - \frac{  2^{30}\,k_1 k_2  \max\{ d_2^2, d_1^2\} (\log d)^2 (\beta-1.002) }{ n}  \Big\} \;  \le\;   \exp \{ -\,2^{25}\,\log d\} \;,
\end{align*}
where the last inequality follows from the assumption that 
$n \leq  k_1 k_2  \max\{d_1^2,d_2^2\} \log d$, 
and $\beta - 1.002\geq 2^{-5}$. 
Since for $d \ge 2$, $\exp\{-2^{25}\log d\} \leq 1/2$ and thus $\eta \le 1/2$, the lemma follows by assembling the last two
displayed inequalities.

\subsection{Proof of Lemma \ref{lmm:bundle_hessian3}}
Let $Z \equiv \sup_{\Delta\in\calB'(D)} \E[H(\Delta)]-H(\Delta)$ and consider the tail bound using McDiarmid's inequality. 
Note that $Z$ has a bounded difference of 
$(8\bb^2e^{-2\bb}\max\{k_1,k_2\} )/(k_1^2k_2^2n)$ when one of the $k_1k_2n$ independent random variables are 
changed, which gives 
\begin{eqnarray}
  \prob{Z-\E[Z] \geq t} &\leq& \exp \Big(- \frac{k_1^4 k_2^4 n^2 t^2 }{64\bb^4 e^{-4\bb} \max\{k_1^2,k_2^2\}k_1k_2n} \Big)\;. 
\end{eqnarray}
With the choice of $t=D^2/(4e^{2\bb} \, d_1d_2)$, this gives 
\begin{eqnarray}
  \prob{Z-\E[Z] \geq \frac{e^{-2\bb}}{4 d_1d_2}D^2} &\leq& \exp \Big(- \frac{k_1^3 k_2^3 n D^4 }{2^{10} \bb^4  d_1^2d_2^2\max\{k_1^2,k_2^2\}} \Big)\;. 
\end{eqnarray}
We first construct a partition of the space similar to Lemma \ref{lmm:kwise_partition}. 
Let 
\begin{eqnarray}
  \tk &\equiv&  \min\{k_1,k_2\}   \;. \label{eq:bundle_deftk}
\end{eqnarray}
\begin{lemma}
  There exists a partition $(\cT_1,\ldots,\cT_N)$ of $\{[k_1]\times[k_2]\}\times\{[k_1]\times[k_2]\}$ for some $N\leq 2k_1^2k_2^2/\tk$  such that
  $\cT_\ell$'s are disjoint subsets,
  $\bigcup_{\ell\in[N]} \cT_\ell =  \{[k_1]\times[k_2]\}\times\{[k_1]\times[k_2]\}$,  $ |\cT_\ell| \leq \tk $ and
  for any $\ell \in[N]$ the set of random variables in $\cT_\ell$ satisfy
  \begin{eqnarray*}
    \{ (\Delta_{j_{i,a},j_{i,b}} - \Delta_{j_{i,a'},j_{i,b'}} )^2\}_{i\in[n],((a,b),(a',b'))\in \cT_\ell} \text{ are mutually independent }\;.
  \end{eqnarray*}
  where $j_{i,a}$ for $i\in[n]$ and $a\in[k_1]$ denote the $a$-th chosen item to be included in the set $S_i$. 
  \label{lmm:bundle_partition}
\end{lemma}

Now we prove an upper bound on $\E[Z]$ using the  symmetrization technique. 
Recall that $j_{i,a}$ is independently and uniformly chosen from $[d_1]$ for $i\in[n]$ and $a\in[k_1]$. 
Similarly, $j_{i,b}$ is independently and uniformly chosen from $[d_1]$ for $i\in[n]$ and $b\in[k_2]$. 
\begin{align}
  &\E[Z] \nonumber\\ &=
   \frac{e^{-2\bb}}{2 \,k_1^2\,k_2^2\,n} \E\left[ \sup_{\Delta\in\calB'(D)} \sum_{i=1}^n  \sum_{\substack{a,a'\in[k_1]  \\ b,b'\in[k_2]}} \E\big[\big( \Delta_{j_{i,a},j_{i,b}}-\Delta_{j_{i,a'},j_{i,b'}} \big)^2 \big] - \big( \Delta_{j_{i,a},j_{i,b}}-\Delta_{j_{i,a'},j_{i,b'}} \big)^2  \right] \nonumber\\
   &\leq \frac{e^{-2\bb}}{2 \,k_1^2\,k_2^2\,n} \sum_{\ell\in[N]} \E\left[ \sup_{\Delta\in\calB'(D)} \sum_{i=1}^n  \sum_{(j_1,j_2,j_1', j_2')\in \cT_\ell} \E\big[\big( \Delta_{j_1,j_2}-\Delta_{j_1',j_2'} \big)^2 \big] - \big( \Delta_{j_1,j_2}-\Delta_{j_1',j_2'} \big)^2  \right]\nonumber\\
   &\leq \frac{e^{-2\bb}}{ \,k_1^2\,k_2^2\,n} \sum_{\ell\in[N]} \E\left[ \sup_{\Delta\in\calB'(D)} \sum_{i=1}^n  \sum_{(j_1,j_2,j_1', j_2')\in \cT_\ell} \xi_{i,j_1,j_2,j_1',j_2'} \big( \Delta_{j_1,j_2}-\Delta_{j_1',j_2'} \big)^2  \right]\;, \label{eq:bundle_conc1}
\end{align}
where the first inequality follows for the fact that the supremum of the sum is smaller than the sum of supremum, and the second inequality follows from standard symmetrization with i.i.d. Rademacher random variables $\xi_{i,j_1,j_2,j_1',j_2'}$'s. 
It follows from Ledoux-Talagrand contraction inequality \cite{ledoux2013probability} that 
\begin{align}
  & \E\left[ \sup_{\Delta\in\calB'(D)} \sum_{i=1}^n  \sum_{(j_1,j_2,j_1', j_2')\in \cT_\ell} \xi_{i,j_1,j_2,j_1',j_2'} \big( \Delta_{j_1,j_2}-\Delta_{j_1',j_2'} \big)^2  \right] \\
  & \;\;\; \leq\; 8\bb\,
    \E\left[ \sup_{\Delta\in\calB'(D)} \sum_{i=1}^n  \sum_{(j_1,j_2,j_1', j_2')\in \cT_\ell} \xi_{i,j_1,j_2,j_1',j_2'} \big( \Delta_{j_1,j_2}-\Delta_{j_1',j_2'} \big)  \right] \\
  & \;\;\; \leq \; 8\bb\,   \E\left[ \sup_{\Delta\in\calB'(D)} \nucnorm{\Delta} \lnorm{\sum_{i=1}^n  \sum_{(j_1,j_2,j_1', j_2')\in \cT_\ell} \xi_{i,j_1,j_2,j_1',j_2'} \big( e_{j_1,j_2}-e_{j_1',j_2'} \big)}{2}  \right] \\
  &\;\;\;\leq \; \frac{8 \bb D^2}{\mu'}\E\left[   \lnorm{\sum_{i=1}^n  \sum_{(j_1,j_2,j_1', j_2')\in \cT_\ell} \xi_{i,j_1,j_2,j_1',j_2'} \big( e_{j_1,j_2}-e_{j_1',j_2'} \big)}{2}  \right] \;, \label{eq:bundle_conc2}
\end{align}
where the second inequality follows for the H\"older's inequality and 
 the last inequality follows from $\mu'\nucnorm{\Delta}\leq D^2$ for all $\Delta \in \calB'(D)$.
To bound the expected spectral norm of the random matrix, 
we use matrix Bernstein's inequality. 
Note that 
$\lnorm{ \xi_{i,j_1,j_2,j_1',j_2'} c }{2}\leq\sqrt{2}$ almost surely,
$\E[( e_{j_1,j_2}-e_{j_1',j_2'})( e_{j_1,j_2}-e_{j_1',j_2'})^T] \preceq (2/d_1)\id_{d_1 \times d_1}$, and 
$\E[( e_{j_1,j_2}-e_{j_1',j_2'})^T( e_{j_1,j_2}-e_{j_1',j_2'})] \preceq (2/d_2)\id_{d_2 \times d_2}$. 
It follows that $\sigma^2=2n|\cT_\ell|/\min\{d_1,d_2\} $, where $|\cT_\ell|\leq \min\{k_1,k_2\}$. 
It follows that 
\begin{align*}
  &\prob{\lnorm{\sum_{i=1}^n  \sum_{(j_1,j_2,j_1', j_2')\in \cT_\ell} \xi_{i,j_1,j_2,j_1',j_2'} \big( e_{j_1,j_2}-e_{j_1',j_2'} \big)}{2}> t} \\ &\leq (d_1+d_2)\exp\Big\{\frac{-t^2/2}{{2n\min\{k_1,k_2\}}/{\min\{d_1,d_2\}} +{\sqrt{2} t}/{3}}\Big\}\;,
\end{align*}
Choosing $t=\max\{ \sqrt{64 n (\min\{k_1,k_2\}/\min\{d_1,d_2\}) \log d} ,(16\sqrt{2}/3)\log d\}$, we obtain a bound on the spectral norm of $t$ with probability 
at least $1-2d^{-7}$. 
From the fact that \\$\lnorm{\sum_{i=1}^n  \sum_{(j_1,j_2,j_1', j_2')\in \cT_\ell} \xi_{i,j_1,j_2,j_1',j_2'} \big( e_{j_1,j_2}-e_{j_1',j_2'} \big)}{2} \leq (n /\sqrt{2})\min\{k_1,k_2\}$, it follows that 
\begin{align}
& \E\left[   \lnorm{\sum_{i=1}^n  \sum_{(j_1,j_2,j_1', j_2')\in \cT_\ell} \xi_{i,j_1,j_2,j_1',j_2'} \big( e_{j_1,j_2}-e_{j_1',j_2'} \big)}{2}  \right] \\
  &\;\;\; \leq\;  \max\Big\{ \sqrt{ \frac{64\,n\,\min\{k_1,k_2\}\log d}{\min\{d_1,d_2\}} } ,(16\sqrt{2}/3)\log d\Big\} + \frac{2n \min\{k_1,k_2\}}{\sqrt{2} d^7}\\
  & \;\;\; \leq \; \sqrt{ \frac{66\,n\,\min\{k_1,k_2\}\log d}{\min\{d_1,d_2\}}} 
\end{align}
which follows form the assumption that 
$n \min\{k_1,k_2\} \geq \min\{d_1,d_2\} \log d$ and 
$n\leq d^5\log d$.
Substituting this bound in \eqref{eq:bundle_conc1}, and \eqref{eq:bundle_conc2}, we get that 
\begin{eqnarray}
  \E[Z] &\leq& \frac{16 e^{-2\bb} \bb D^2}{\mu' }\sqrt{\frac{66 \log d}{n \min\{k_1,k_2\} \min\{d_1,d_2\}}}\\
    &\leq& \frac{e^{-2\bb} D^2 }{4\, d_1d_2 }\;. 
\end{eqnarray}

\section{Proof of the Information-theoretic Lower Bound in Theorem \ref{thm:bundle_lb}}
\label{sec:bundle_lb_proof}
This proof follow closely the proof of Theorem \ref{thm:kwise_lb} in Appendix \ref{sec:kwise_lb_proof}. 
We apply the generalized Fano's inequality in the same way to get Eq. \eqref{eq:kwise_fano} 
\begin{eqnarray}
  \prob{\hL\neq L} 
  &\geq& 1-\frac{ {M \choose 2}^{-1} \sum_{\ell_1,\ell_2\in[M]} D_{\rm KL}(\Theta^{(\ell_1)} \| \Theta^{(\ell_2)}) +\log 2}{\log M} \label{eq:bundle_fano}\;,
\end{eqnarray}

The main challenge in this  case is that we can no longer directly apply the RUM interpretation to compete 
$D_{\rm KL}(\Theta^{(\ell_1)} \| \Theta^{(\ell_2)})$. 
This will result in over estimating the KL-divergence,  
because this approach does not take into account that we  only take the top winner, out of those $k_1k_2$ alternatives. 
Instead, we compute the divergence directly, and provide an appropriate bound. Let the set of $k_1$ rows and $k_2$ columns chosen in one of the $n$ samples be $S \subset [d_1]$ and $T \subset [d_2]$ respectively. Then,
\begin{eqnarray*}
  &&D_{\rm KL}(\Theta^{(\ell_1)} \| \Theta^{(\ell_2)}) \\ &\overset{(a)}{=}& \frac{n}{{d_1 \choose k_1}{d_2 \choose k_2}} \sum_{S, T} \sum_{\substack{i \in S \\ j \in T}} \frac{e^{\Theta_{ij}^{(\ell_1)}}}{\sum_{\substack{i' \in S \\ j' \in T}} e^{\Theta_{i'j'}^{(\ell_1)}}} \log\left(\frac{e^{\Theta_{ij}^{(\ell_1)}}\sum_{\substack{i' \in S \\ j' \in T}} e^{\Theta_{i'j'}^{(\ell_2)}}}{e^{\Theta_{ij}^{(\ell_2)}}\sum_{\substack{i' \in S \\ j' \in T}} e^{\Theta_{i'j'}^{(\ell_1)}}}\right)\\ 
    &\overset{(b)}{\leq}& \frac{n}{{d_1 \choose k_1}{d_2 \choose k_2}} \sum_{S, T} \left(\sum_{\substack{i , j }}\frac{{e^{2\Theta_{ij}^{(\ell_1)}}}\sum_{\substack{i', j' }}e^{\Theta_{i'j'}^{(\ell_2)}} - e^{\Theta_{ij}^{(\ell_1)} + \Theta_{ij}^{(\ell_2)}}\sum_{\substack{i' , j' }}e^{\Theta_{i'j'}^{(\ell_1)}}}{e^{\Theta_{ij}^{(\ell_2)}} \left(\sum_{\substack{i',j' }} e^{\Theta_{i'j'}^{(\ell_1)}} \right)^2  }\right) \\
    &\overset{(c)}{\leq}& \frac{ne^{2\alpha}}{k_1^2 k_2^2 {d_1 \choose k_1}{d_2 \choose k_2}} \sum_{S,T} \sum_{i,j} \left (e^{2\Theta_{ij}^{(\ell_1)}-\Theta_{ij}^{(\ell_2)} }  \sum_{i', j' }e^{\Theta_{i'j'}^{(\ell_2)}} - e^{\Theta_{ij}^{(\ell_1)}} \sum_{i', j' }e^{\Theta_{i'j'}^{(\ell_1)}} \right)
    \\
    & = & \frac{ne^{2\alpha}}{k_1^2 k_2^2 {d_1 \choose k_1}{d_2 \choose k_2}} \sum_{S, T} \left( \sum_{\substack{i' , j' }} e^{\Theta_{i'j'}^{(\ell_2)}} \sum_{\substack{i,j }} \frac{\left(e^{\Theta_{ij}^{(\ell_1)}} - e^{\Theta_{ij}^{(\ell_2)}}\right)^2}{e^{\Theta_{ij}^{(\ell_2)}}}  - \Big(\sum_{i,j}  ( e^{\Theta_{ij}^{(\ell_1)}} - e^{\Theta_{ij}^{(\ell_2)}} )\Big)^2 \right) \\
    &\overset{(d)}{\leq}& \frac{ne^{4\alpha}}{k_1 k_2 {d_1 \choose k_1}{d_2 \choose k_2}} \sum_{S, T} \sum_{\substack{i , j }}\left(e^{\Theta_{ij}^{(\ell_1)}} - e^{\Theta_{ij}^{(\ell_2)}}\right)^2\\
    &\overset{(e)}{\leq}& \frac{ne^{5\alpha}}{k_1 k_2 {d_1 \choose k_1}{d_2 \choose k_2}} \sum_{S, T} \sum_{\substack{i , j }}\left(\Theta_{ij}^{(\ell_1)} - \Theta_{ij}^{(\ell_2)}\right)^2\\
    &\overset{(f)}{=}& \frac{ne^{5\alpha}}{d_1 d_2} \fnorm{\Theta_{ij}^{(\ell_1)} - \Theta_{ij}^{(\ell_2)}}^2
\end{eqnarray*}

Here $(a)$ is by definition of KL-distance and the fact that $S$, $T$ are chosen uniformly from all possible such sets and $(b)$ is due to the fact that $\log(x) \leq x-1$ with $x=({e^{\Theta_{ij}^{(\ell_1)}}\sum_{\substack{i' \in S,j' \in T}} e^{\Theta_{i'j'}^{(\ell_2)}}})/({e^{\Theta_{ij}^{(\ell_2)}}\sum_{\substack{i' \in S, j' \in T}} e^{\Theta_{i'j'}^{(\ell_1)}}})$. The constants at $(c)$   is due to the fact that each element of $\Theta^{(\ell_1)}$ is upper bounded by $\alpha$ and lower bounded by $-\alpha$. We can get $(d)$ by 
removing the second term which is always negative, and using the bond of $\alpha$.  $(e)$ is obtained because $e^x$ where $-\alpha \leq x \leq \alpha$ is Lipschitz continuous with Lipschitz constant $e^\alpha$. At last $(f)$ is obtained by simple counting of the occurrences of each $ij$. 
Thus we have,
\begin{eqnarray}
  \prob{\hL\neq L} 
  &\geq& 1-\frac{ {M \choose 2}^{-1} \sum_{\ell_1,\ell_2\in[M]} \frac{ne^{5\alpha}}{d_1 d_2} \fnorm{\Theta_{ij}^{(\ell_2)} - \Theta_{ij}^{(\ell_2)}}^2 +\log 2}{\log M},
\end{eqnarray}

The remainder of the proof relies on the following probabilistic packing.

\begin{lemma}
  \label{lem:bundle_packing} 
  Let $d_2 \geq d_1$ be sufficiently large positive integers. 
  Then for each $r\in\{1,\ldots,d_1 \}$, and for any positive 
  $\delta>0$ there exists a family of $d_1\times d_2$ dimensional matrices 
  $\{\Theta^{(1)},\ldots,\\\Theta^{(M(\delta))}\}$ with cardinality 
  $M(\delta) = \lfloor (1/4)\exp(r d_2 /576)\rfloor$ such that each matrix is rank $r$ and the following bounds hold:
  \begin{eqnarray}
    \fnorm{\Theta^{(\ell)}}&\leq& \delta \;, \text{ for all }\ell\in[M]\\
    \fnorm{\Theta^{(\ell_1)} - \Theta^{(\ell_2)}} &\geq& \frac12\delta \;, \text{ for all } 
    \ell_1, \ell_2\in[M] \label{eq:packing_lower}\\
    \Theta^{(\ell)} &\in& \Omega_{\tilde{\bb}} \;, \text{ for all } \ell\in [M]\;, 
  \end{eqnarray}
  with $\tilde{\bb}=   (8\delta/d_2)\sqrt{2\log d} $ for $d=(d_1+d_2)/2$.
\end{lemma}

Suppose $\delta\leq \alpha d_2/(8 \sqrt{2 \log d)}$ such that 
the matrices in the packing set are entry-wise bounded by $\bb$,  then the above lemma \ref{lem:bundle_packing} implies that $\fnorm{\Theta^{(\ell_1)} - \Theta^{(\ell_2)}}^2 \leq4\delta^2$, which gives 
\begin{eqnarray}
  \label{eq:prob-2item}
  \prob{ \hL\neq L } &\geq& 1- \frac{\frac{e^{5\bb}n4\delta^2}{d_1d_2} + \log 2}{\frac{rd_2}{576} - 2\log2}  \;\geq\; \frac{1}{2}\;,
\end{eqnarray}
where the last inequality holds for $\delta^2 \leq (r d_1 d_2^2/(1152e^{5\bb}n))$ and assuming $rd_2 \geq 1600$. Together with \eqref{eq:prob-2item} and \eqref{eq:packing_lower}, this proves that for all $\delta\leq\min\{ \alpha d_2/(8 \sqrt{2 \log d)},\\ \sqrt{r d_1 d_2^2/(9216\,e^{5\bb}n)}\}$, 
\begin{eqnarray*}
  \inf_{\hTheta} \sup_{\Theta^*\in\Omega_{\bb}} \E\Big[\, \fnorm{\hTheta-\Theta^*} \,\Big] &\geq& \delta/4\;.
\end{eqnarray*}
Choosing $\delta$ appropriately to maximize the right-hand side finishes the proof of the desired claim. 
Also by symmetry, we can apply the same argument to get similar bound with $d_1$ and $d_2$ interchanged.

\subsection{Proof of Lemma \ref{lem:bundle_packing}} 
\label{sec:bundle_packing_proof}

We show that the following procedure succeeds in producing the desired family with probability at least half, which proves its  existence. 
Let $d=(d_1+d_2)/2$, and suppose $d_2\geq d_1$ without loss of generality. 
For the choice of  $M'=e^{r d_2 /576}$, and for each $\ell \in [M']$, generate a rank-$r$ matrix $\Theta^{(\ell)}\in\reals^{d_1\times d_2}$ as follows: 
\begin{eqnarray}
  \Theta^{(\ell)} &=& \frac{\delta}{\sqrt{r d_2}} U(V^{(\ell)})^T - \frac{\delta}{\sqrt{r d_2}} \frac{\ones^TU(V^{(\ell)})^T\ones}{d_1d_2}\ones\ones^T \;, 
  \label{eq:bundle_defpacking}
\end{eqnarray}
where 
$U\in\reals^{d_1\times r}$ is a random orthogonal basis such that $U^TU=\id_{r\times r}$ and $V^{(\ell)} \in\reals^{d_2\times r}$ 
is a random matrix with each entry $V^{(\ell)}_{ij} \in\{-1,+1\}$ chosen independently and uniformly at random. 
By construction, notice that $\fnorm{\Theta^{(\ell)}} \leq (\delta/\sqrt{rd_2})\fnorm{U(V^{(\ell)})^T}  = \delta$. 

Now, by triangular inequality, we have 
\begin{eqnarray*}
    &&\fnorm{\Theta^{(\ell_1)}-\Theta^{(\ell_2)}} \\&\geq& 
    \frac{\delta}{\sqrt{r d_2}} \fnorm{U(V^{(\ell_1)} - V^{(\ell_2)})^T } -  \frac{\delta\,|\ones^T U(V^{(\ell_1)} - V^{(\ell_2)})^T\ones|}{d_1d_2\sqrt{r d_2}}  \fnorm{ \ones\ones^T } \\
    &\geq& \frac{\delta}{\sqrt{r d_2}} \underbrace{\fnorm{V^{(\ell_1)} - V^{(\ell_2)} }}_{A} \,-\,  \frac{\delta}{\sqrt{r\, d_1\,d_2^2}}\big(\underbrace{|\ones^T U(V^{(\ell_1)})^T\ones|}_{B} +|\ones^T U (V^{(\ell_2)})^T \ones|\big)   \;.
\end{eqnarray*}
We will prove that the first term is bounded by $A\geq \sqrt{rd_2}$ with probability at least $7/8$ for all $M'$ matrices, 
and we will show that we can find $M$ matrices such that the second term is bounded by 
$B\leq 8 \sqrt{2 r d_2 \log (32 r) \log (32 d)}$ with probability at least $7/8$. 
Together, this proves that with probability at least $3/4$, there exists $M$ matrices such that 
\begin{eqnarray*}
    \fnorm{\Theta^{(\ell_1)}-\Theta^{(\ell_2)}} &\geq& \delta \Big(1-\sqrt\frac{2^7 \log (32r) \,\log (32d)}{ d_1 d_2}\Big)  \; \geq \; \frac{1}{2}\delta\;,
\end{eqnarray*}
for all $\ell_1$, $\ell_2\in[M]$ and for sufficiently large $d_1$ and $d_2$. 

Applying similar McDiarmid's inequality as Eq. \eqref{eq:mcdiarmid1} in Appendix \ref{sec:kwise_lb_proof}, it follows that 
$A^2 \geq r d_2 $ with probability at least $7/8$ for $M'= e^{r d_2/576}$ and a sufficiently large $d_2$.  

To prove a bound on $B$, we will show that for a given $\ell$, 
\begin{eqnarray}
    \prob{|\ones^TU(V^{(\ell)})^T\ones| \leq 8 \sqrt{2 r d_2 \log (32 r) \log (32 d)}} \;\geq\; \frac78 \;. \label{eq:bundle_diffbound1}
\end{eqnarray}
Then using the similar technique as in \eqref{eq:unioncardinality}, it follows that we can find $M=(1/4)M'$ matrices 
all satisfying this bound and also the bound on the max-entry in \eqref{eq:maxentrybound2}. 
We are left to prove \eqref{eq:bundle_diffbound1}. We apply a series of concentration inequalities. Let $H_1$ be the event that $\{|\llangle V^{(\ell)}_i,\ones \rrangle| \leq \sqrt{2 d_2 \log(32 r)} \text{ for all } i\in[r]\}$.
Then, applying the standard Hoeffding's inequality, we get that 
$\prob{ H_1 } \geq 15/16$, where $V^{(\ell)}_i$ is the $i$-th column of $V^{(\ell)}$. 
We next change the variables and represent $\ones^T U$ as 
$\sqrt{d_1} u^T \tilde{U}$, where $u$ is drawn uniformly at random from the unit sphere and $\tilde{U}$ is a $r$ dimensional subspace drawn uniformly  at random. 
By symmetry, $\sqrt{d_1} u^T \tilde{U}$ have the same distribution as $\ones^T U$. 
Let $H_2$ be the event that $\{ | \llangle \tilde{U}_i,(V^{(\ell)})^T \ones \rrangle | \leq \sqrt{16 r (d_2/d_1) \log(32r)\log(32 d) } \text{ for all } i \in [d_1]\}$, where $\tilde{U}_i$ is the $i$-th row of $\tilde{U}$. 
Then, applying Levy's theorem for concentration on the sphere \cite{Led01}, we have 
$\prob{H_2 | H_1} \geq 15/16$.
Finally, let $H_3$ be the event that 
$\{| \sqrt{d_1} \llangle u , \tilde{U} (V^{(\ell)})^T \rrangle \ones | \leq 8\sqrt{2 r d_2 \log(32 r)\log(32 d)} \}$. 
Then, again applying Levy's concentration, we get 
$\prob{H_3|H_1,H_2} \geq 15/16$. 
Collecting all three concentration inequalities, we get that with probability at least $13/16$, 
$|\ones^TU(V^{(\ell)})^T\ones| \leq 8\sqrt{2rd_2 \log(32r) \log (32d)}$, which proves Eq. \eqref{eq:bundle_diffbound1}.

We are left to prove that $\Theta^{(\ell)}$'s are in 
$\Omega_{(8\delta/d_2)\sqrt{2\log d_2}}$ as defined in \eqref{eq:defbundleomega}. 
 Similar to Eq. \eqref{eq:lb_maxbound}, applying Levy's concentration gives 
\begin{eqnarray}
  \prob{\, \max_{i,j} |\Theta^{(\ell)}_{ij}|  \leq \frac{2\delta\sqrt{32 \log d_2} }{d_2}\,} &\geq& 1-2 \exp \Big\{ -2 \log d_2 \Big\} \;\geq \;\frac12 \;,
  \label{eq:maxentrybound2}
\end{eqnarray}
for a fixed $\ell\in[M']$. 
 Then using the similar technique as in \eqref{eq:unioncardinality},  
it follows that  there exists $M=(1/4)M'$ matrices 
 all satisfying this bound and also the bound on $B$ in Eq. \eqref{eq:bundle_diffbound1}. 

\vskip 0.2in
\end{document}